%% file: arxiv-v4.tex
\documentclass{article}
\input{math_commands.tex}
\usepackage{amsmath,amssymb,amsthm,bbm,color}
\usepackage[hyperfootnotes=false]{hyperref}
\usepackage{geometry}  % set the margins to 1in on all sides
\usepackage{tikz}
\usetikzlibrary{decorations.pathreplacing}
\usepackage{caption}
\usepackage{graphicx}
\usepackage{subcaption}
\usepackage{afterpage}
\usepackage[hang,flushmargin]{footmisc}
\usepackage{multirow}
\usepackage{microtype}
\usepackage{ragged2e}
\usepackage{soul}
\usepackage{natbib}
\usepackage{booktabs}       % professional-quality tables
\usepackage{algorithm}      % pseudocode
\usepackage{algorithmic}    % pseudocode
\setcitestyle{authoryear}
\usepackage{dsfont}
\usepackage{caption}

\usepackage[utf8]{inputenc} % allow utf-8 input
\usepackage{amsthm}
\newtheorem{theorem}{Theorem}
\newtheorem{lemma}{Lemma}[section]

\newtheorem{assumption}{Assumption}
\newtheorem{proposition}{Proposition}
\newtheorem{remark}{Remark}[section]

\newtheorem{corollary}{Corollary}[section]
\newtheorem{definition}{Definition}

\def\FDR{\operatorname{FDR}}
\def\FCR{\operatorname{FCR}}

\def\FCP{\operatorname{FCP}}
\def\err{\operatorname{err}}
\newcommand{\independent}{\mathop{\perp \! \! \! \perp}}

\allowdisplaybreaks

\title{CAP: A General Algorithm for Online Selective Conformal Prediction with FCR Control}
\author{
Yajie Bao$^a$, Yuyang Huo$^a$, Haojie Ren$^b$\thanks{Corresponding Author: haojieren@sjtu.edu.cn; The authors are listed in alphabetical order.} and Changliang Zou$^a$\\
$^a$ School of Statistics and Data Science, Nankai University\\
 Tianjin, P.R. China\\
$^b$ School of Mathematical Sciences,  Shanghai Jiao Tong University \\ Shanghai, P.R. China
}

\begin{document}
\maketitle
\begin{abstract}
We study the problem of post-selection predictive inference in an online fashion.
To avoid devoting resources to unimportant units, a preliminary selection of the current individual before reporting its prediction interval is common and meaningful in online predictive tasks. Since the online selection causes a temporal multiplicity in the selected prediction intervals, it is important to control the real-time false coverage-statement rate (FCR) which measures the overall miscoverage level.  We develop a general framework named CAP (\textbf{C}alibration after \textbf{A}daptive \textbf{P}ick) that
performs an \emph{adaptive pick} rule on historical data to construct a calibration set if the current individual is selected and then outputs a conformal prediction interval for the unobserved label. We provide tractable procedures for constructing the calibration set for popular online selection rules. We proved that CAP can achieve an exact selection-conditional coverage guarantee in the finite-sample and distribution-free regimes. To account for the distribution shift in online data, we also embed CAP into some recent dynamic conformal prediction algorithms and show that the proposed method can deliver long-run FCR control.  
Numerical results on both synthetic and real data corroborate that CAP can effectively control FCR around the target level and yield more narrowed prediction intervals over existing baselines across various settings.

\end{abstract}

\section{Introduction}

{Conformal inference provides a powerful and flexible tool to quantify the uncertainty of ``black-box'' prediction models by issuing prediction intervals (PI) for {\emph{unlabeled data}} \citep{vovk1999machine,vovk2005algorithmic}. In many applications, it is unnecessary or inefficient to perform predictive inference on all unlabeled data due to collection and cost constraints. For example, in drug discovery, scientists aim to select promising drug-target pairs based on prediction values of binding affinity for further clinical trials \citep{dara2021machine}. Hence, a more feasible option is to perform predictive inference on only the selected individuals of interest, which is referred to as {\emph{Selective Predictive Inference}} \citep{bao2023selective}. 

Recently, several works \citep{bao2023selective,jin2024confidence,gazin2024selecting} have formally explored this area in offline settings. In applications of scientific discovery or industrial production, it is desirable to perform real-time selection or screening prior to predictive inference.
{As in the example of drug discovery, drug-target pairs often appear sequentially, requiring scientists to determine whether to retain the current pair for further investigation based on the predicted affinity values.}
In contrast to offline scenarios where individuals of interest can be selected simultaneously, online selection rules may change in real-time or be influenced by incoming data, leading to complicated impacts for downstream predictive inference. As a result, it becomes more challenging to guarantee the validity of online selected PIs.}

{This paper studies reliable selective conformal predictions in the online case.} Formally, suppose the feature-label pairs $\{(X_t, Y_t)\}_{t\geq 0} \subseteq \sR^d \times \sR$ are collected in a sequential and delayed fashion. At time $t$, one can observe the previous label $Y_{t-1}$ and the new feature $X_{t}$. Let $\Pi_t(\cdot): \sR^d \to \{0,1\}$ be a generic
online selection rule that may depend on previously observed data. To be specific, let $S_{t} = \Pi_{t}(X_{t})$ be the selection indicator or decision, and the task is to report the PI, $\gI_t(X_t)$, for the unobserved label $Y_{t}$ when $S_t = 1$.
% Since training large-scale machine learning models is time-consuming, we focus on the setting that a pre-trained model $\widehat{\mu}$ is given, and we discuss the case with the online-updating learning model in Section \ref{sec:dist_shift}. 
% {\color{red} Throughout the paper, we denote $\{(X_i,Y_i)\}_{i\in \gH_t}$ the holdout set at time $t$, where $\gH_t$ is an index subset of previously observed data.}
% The split conformal method \citep{papadopoulos2002inductive,vovk2005algorithmic} can naturally yield a marginal PI $\gI_t^{\rm{marg}}(X_t;\alpha)$ by computing the empirical quantile of residuals on the holdout set $\{R_i = |Y_i-\widehat{\mu}(X_i)|\}_{i\in \gH_t}$, {\color{red}where $\widehat{\mu}$ is a given pre-tained machine learning model}. If the data points $\{(X_i,Y_i)\}_{i\leq t}$ are i.i.d., the interval $\gI_t^{\rm{marg}}(X_t;\alpha)$ enjoys a distribution-free coverage guarantee $\sP\LRl{Y_t \in \gI_t^{\rm{marg}}(X_t;\alpha)} \geq 1- \alpha$ as discussed in \citet{lei2018distribution}.

%Hereafter, denote $\gI_t(X_t)$ as $Y_t$'s prediction interval constructed from the historical data $\{(X_i,Y_i)\}_{i\in \gH_t}$ and the new feature $X_t$.

{As \citet{benjamini2005false} highlighted, the selection process introduces multiplicity, and neglecting this multiplicity in the construction of selected parameters’ confidence intervals results in undesirable consequences. Similar issues also appear in online selective predictive inference.}
\citet{weinstein2020online} considered temporal multiplicity and extended the definition of false coverage-statement rate (FCR) proposed by \citet{benjamini2005false} to the online regime. For any online predictive procedure that returns PIs $\{\gI_t(X_t): S_t = 1\}_{t\geq 0}$, the corresponding FCR value and false coverage proportion (FCP) up to time $T$ are defined as
\begin{align*}
    \FCP(T) = \frac{\sum_{t=0}^T S_t\cdot \Indicator{Y_t \not\in \gI_t(X_t)}}{1\vee \sum_{j=0}^T S_j},\quad \FCR(T)=\E\{\FCP(T)\},
\end{align*}
where $a\vee b = \max\{a, b\}$ for any $a,b \in \sR$. To achieve real-time FCR control when constructing post-selection confidence intervals of parameters, \citet{weinstein2020online} proposed a novel approach named LORD-CI based on the building of marginal confidence intervals at a sequence of adjusted confidence levels $\{\alpha_t\}_{t\geq 0}$ such that $\sum_{t=0}^T \alpha_t/(1 \vee \sum_{j=0}^T S_j) \leq \alpha$ for any $T\geq 0$.  The LORD-CI is a general algorithm that can be readily applied to construct post-selection PIs. However, the resulting PI with level $(1-\alpha_t)$ tends to be overly wide {since it does not incorporate the selection event into calculating miscoverage probabilities when estimating FCR}.
%it his issue is essentially caused by ignoring the selection condition in the construction of prediction intervals
%t does not take the unique feature of online prediction setting into account, say we are usually able to obtain past labels at current time (or at most certain time lag). 
%one drawback of LORD-CI is that the prediction interval tends to be overly wide because $\alpha_t$ may be very small in
%resulting in unsatisfactory or even powerless performance in practice.
%This issue is essentially caused by ignoring the selection condition in the construction of prediction intervals. 
In fact, it would be desirable to achieve the so-called \emph{selection-conditional coverage} (SCC) guarantee,
\begin{align}
    \sP\LRl{Y_t \in \gI_t(X_t) \mid S_t=1} \geq 1- \alpha,\quad \forall t\geq 0,\nonumber
\end{align}
which characterizes the coverage property of PI conditioning on the selection event and has been studied in \citet{bao2023selective} and \citet{jin2024confidence}. %{Under some regularities on selection rules (e.g., decision-driven selection in Section \ref{sec:decision_selection}), the FCR can be controlled if the SCC is satisfied at the same level $\alpha$.}

\subsection{Our approach: calibration after adaptive pick (CAP) on historical data}

This paper aims to develop a distribution-free framework to construct post-selection prediction intervals with selection-conditional coverage while successfully controlling real-time FCR around the target level. Our strategy is motivated by the idea of post-selection calibration in \citet{bao2023selective}, which proposed a selective conditional conformal prediction procedure (SCOP) in the offline scheme. They first apply a pick rule on independent labeled data with the \emph{identical} threshold used in the test set to obtain a selected calibration set, and then construct split conformal PIs by leveraging the empirical distribution of residuals in the selected calibration data. {If the threshold is invariant to the permutation of all data points in the labeled data set and test set, the selected test data is exchangeable with the selected calibration data, then SCOP can achieve both SCC guarantee and FCR control.} However, this assumption about the threshold may not be realistic in the online setting, where the selection rule $\Pi_t$ usually depends only on previously observed data. %Moreover, the nonadaptive pick of SCOP on the holdout set may fail in the online setting. %For example, if we aim to select the individual whose predicted value is no smaller than the minimum of predictions in the holdout set, that is $S_t = \Indicator{\widehat{\mu}(X_t) \geq \min_{s \leq t-1} \widehat{\mu}(X_s)}$. Then for any $t\geq 0$, SCOP always returns a \emph{marginal} conformal prediction interval because the entire holdout set is selected, and cannot provide exact post-selection calibration anymore. This motivates us to develop a more principled algorithm for online selective conformal prediction.

%Intuitively, the selected calibration set of SCOP is used to approximate the distribution of the test data given the selection condition. As shown in Figure \ref{fig:alg_illustration}, 

For online selective conformal prediction, we develop a more principled algorithm, named \textbf{C}alibration after \textbf{A}daptive \textbf{P}ick (CAP) on all available historical data. { Let $\gH_t$ be indices of historical labeled data at time $t$, and we call the data $\{(X_s,Y_s)\}_{s\in \gH_t}$ as the \emph{holdout set}.}
When $S_t = 1$, we firstly use a sequence of \emph{adaptive pick} rules $\{\Pi_{t,s}^{\rm{Ada}}(\cdot)\}_{s\in \gH_t}$ on historical data to select a \emph{calibration set} $\{(X_s,Y_s)\}_{s\in \widehat{\gC}_t}$ where $\widehat{\gC}_t =\{s\in \gH_t: \Pi_{t,s}^{\rm{Ada}}(X_s)=1\}$. The rule $\Pi_{t,s}^{\rm{Ada}}(\cdot)$ is constructed by integrating the information from the historical selection rules and $X_t$. {Those selected calibration points $(X_s,Y_s)$ in $\widehat{\gC}_t$ satisfy that each of them and the selected test point $(X_t,Y_t)$ are exchangeable conditioning on other data $\{(X_i,Y_i)\}_{i\neq s,t}$}. Then for a target FCR level $\alpha$, we report the following PI:
\begin{align}
    \gI_t^{\rm{CAP}}(X_t;\alpha) = \widehat{\mu}(X_t) \pm q_{\alpha}(\{R_i\}_{i\in \widehat{\gC}_t}),\nonumber
\end{align}
where $q_{\alpha}(\{R_i\}_{i\in \widehat{\gC}_t})$ denotes the $\lceil (1-\alpha)(|\widehat{\gC}_t|+1)\rceil$-st smallest value in $\{R_i\}_{i\in \widehat{\gC}_t}$.

\begin{figure}[tb]
    \centering
    \includegraphics[width=1\linewidth]{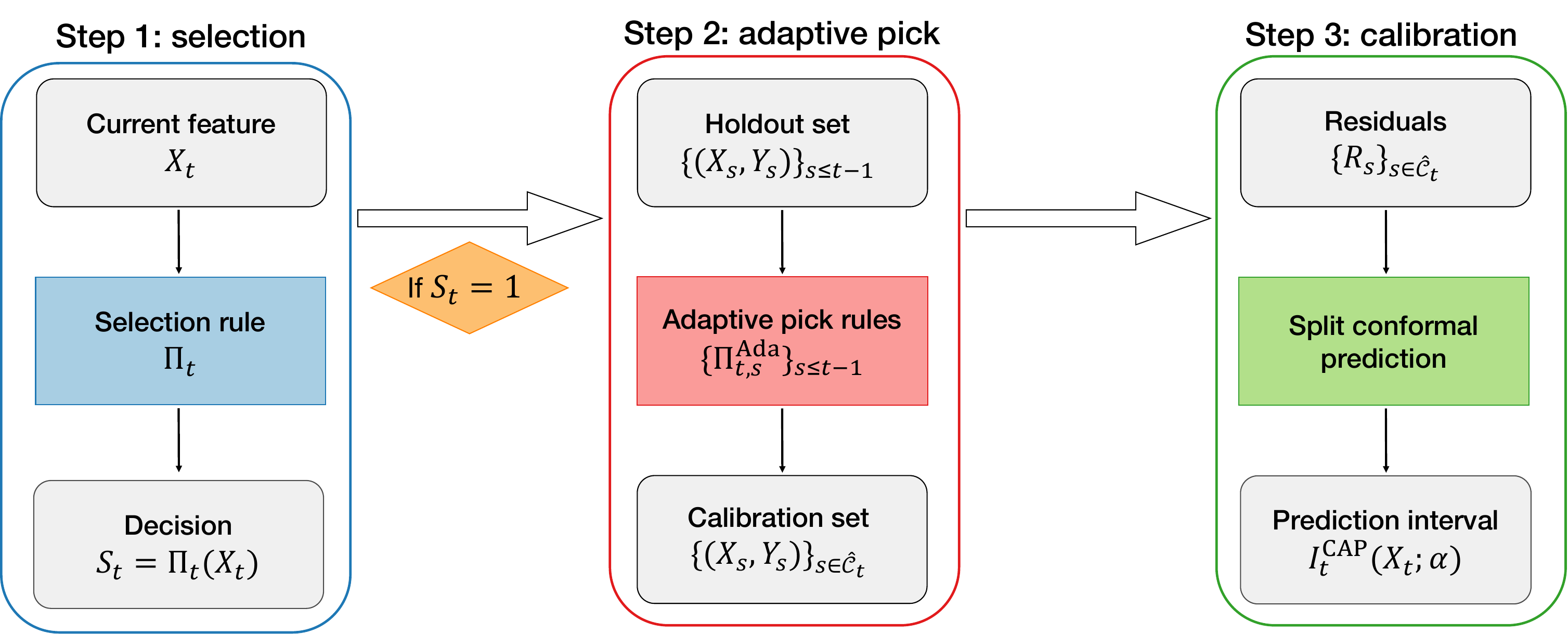}
    \caption{\small\it The workflow of CAP at time $t$. The picked calibration set is $\{(X_s,Y_s)\}_{s\in \widehat{\gC}_t}$, where $\widehat{\gC}_t = \LRl{s\in \gH_t: \Pi_{t,s}^{\rm{Ada}}(X_s) = 1}$. The residuals are computed by $R_s = |\widehat{\mu}(X_s) - Y_s|$.}
    \label{fig:alg_illustration}
\end{figure}

To ensure exact exchangeability after selection, we design adaptive pick rules for two popular classes of selection procedures. The first class is the {\it decision-driven selection} considered in \citet{weinstein2020online}, where the adaptive pick rule takes advantage of the intrinsic property of decision-driven selection to obtain an ``intersecting'' subset of the holdout set. The second class pertains to {\it selection with symmetric thresholds}, which involves screening individuals according to the empirical distributions of historical samples. Here, we propose an adaptive pick rule by ``swapping'' $X_t$ and $X_s$ for $s\in \gH_t$ in the explicit form of the indicator $S_t$ to obtain a new indicator determining whether $(X_s,Y_s)$ is picked as a calibration point.

The workflow of the proposed method CAP at time $t$ is described in Figure \ref{fig:alg_illustration}. Our contributions are:
\begin{itemize} 
    \item[(1)] Compared to the offline regime, controlling the real-time FCR is more challenging due to the temporal dependence of decisions $\{S_t\}_{t\geq 0}$. For decision-driven selection, we prove that CAP exactly controls the real-time FCR below the target level without any distributional assumption. For selection with symmetric thresholds, we provide an upper bound on the real-time FCR under certain mild stability conditions on the selection threshold. %The additional error in the FCR bound will vanish to zero in the asymptotic regime in many common settings. %when the threshold is the historical mean orquantile.

    \item[(2)] Credited to the adaptive pick on historical data, CAP could achieve the finite-sample SCC guarantee in both decision-driven selection and selection with symmetric thresholds. More importantly, our results are distribution-free and can be applied to many practical tasks without prior knowledge of data distribution.

    \item[(3)] To cope with the distribution shift in online data, we adjust the level of PIs whenever the selection happens through the adaptive conformal inference framework in \citet{gibbs2021adaptive}. The new algorithm achieves long-run FCR control with properly chosen parameters under arbitrary distribution shifts.
    %The idea comes from recent works on online conformal prediction under distribution shift \citep{gibbs2021adaptive}.

    \item[(4)] %We thoroughly evaluate the empirical applicability of our method on synthetic datasets by considering various combinations of selection rules, prediction algorithms, and calibration strategies. 
    Through extensive experiments on both synthetic and real-world data, we demonstrate the consistent superiority of our method over other benchmarks in terms of accurate FCR control and narrow PIs.
\end{itemize}

%(In fact, $S_t$ is a function of $X_t$ and $\{X_i\}_{i\in \gH_t}$.)
%We also provide tractable constructions for $\widehat{\gC}_t$ under two classes of selection rules, including the one considered in \citet{weinstein2020online}.

\begin{figure}[htb!]
    \centering
    \includegraphics[width=\textwidth]{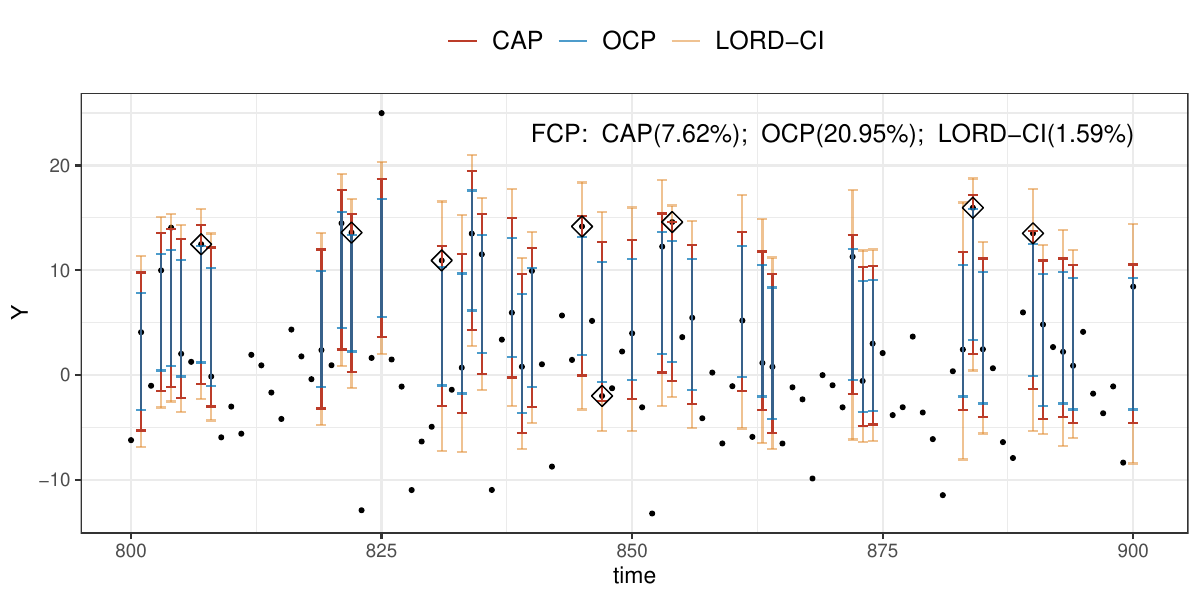}
    \caption{\small\it Plot for the real-time PIs for selected points from time 800 to 900. The selected points are marked by the cross. The experimental setup is the same as Scenario B with a decision-driven selection rule in Section \ref{sec:experiments}. The PIs are constructed by three methods with a target FCR level $10\%$. Red interval: CAP (FCP at index 900 is $7.62\%$); Blue interval: ordinary online conformal prediction which provides marginal interval (FCP is $20.95\%$); Orange interval: LORD-CI with defaulted parameters (FCP is $1.59\%$). Points circled by hollow diamond symbols indicate cases where CAP successfully covers the true response, while OCP fails. }
    \label{fig:illustra_plot}
\end{figure}

{Before closing this section, we display an example to illustrate the selective effects on predictive inference. We compared CAP with the other two benchmarks in a simulated scenario. The first one is the ordinary conformal prediction (OCP), which constructs the $(1-\alpha)$ marginal conformal PIs whenever $S_t=1$ without consideration of the selective bias. Another benchmark is LORD-CI in \citet{weinstein2020online}. Figure \ref{fig:illustra_plot} visualizes the real-time PIs with a target FCR level $10\%$ constructed by different methods. The simulation details are given in Section \ref{sec:experiments}. The proposed method CAP (red ones) produces the shortest intervals with FCP at $7.62\%$. 
The OCP (blue ones) fails to cover the responses with FCP at $20.95\%$. The points circled by diamonds indicate cases where our method, CAP, covers the true response while OCP fails. And LORD-CI (orange ones) produces excessively wide intervals and yields a conservative FCP level $1.59\%$. Therefore, CAP emerges as a valid approach to accurately quantifying uncertainty while simultaneously achieving effective interval sizes.

}

\subsection{Outline}
The remainder of this paper is organized as follows. The CAP methodology and its related works are described in Section 2.
Sections \ref{sec:decision_selection} and \ref{sec:symmetric_selection} present the construction of adaptive pick rules and the theoretical properties of CAP for decision-driven selection and online selection with symmetric thresholds, respectively. Section \ref{sec:dist_shift} investigates the CAP under distribution shift.
Numerical results and real-data examples are presented in Sections \ref{sec:experiments} and \ref{sec:real_data}. Section \ref{sec:conclusion} concludes the paper, and the technical proofs are relegated to the Appendix.

\section{Online selective conformal prediction}

\subsection{Algorithmic structure of CAP}

Suppose a prediction model {$\widehat{\mu}(\cdot):\sR^d \to \sR$} is pre-trained by an independent training set. To make sure that the PIs can be constructed when $t$ is small, we assume there exists an independent labeled set denoted by $\{(X_{i}, Y_{i})\}_{i = -n}^{-1}$. %Denote $R_i = |Y_i - \widehat{\mu}(X_i)|$ for $i\geq -n$ as the absolute residuals. 
%At time $t$, we have access to the historical residuals $\{R_i\}_{i= -n}^{t-1}$ after observing $Y_{t-1}$. 
Let {$\gH_t = \{-n,\ldots,t-1\}$} be indices of the \emph{holdout} set at time $t$. The selection rule $\Pi_t$ is generated from the previously observed data $\{(X_{i}, Y_{i})\}_{i\in\gH_t}$. We summarize the general procedure of the proposed method CAP for online selective conformal prediction in Algorithm \ref{alg:main}.%, which is abbreviated as CAP in this paper. 

%Let $\gH_t = \{-n,\ldots,t-1\}$ be deterministic indices of the \emph{holdout} set at time $t$. If $S_t = 1$, we apply the adaptive pick rules $\{\Pi_{t,s}^{\rm{Ada}}\}_{s\in \gH_t}$ on the features $\{X_s\}_{s\in \gH_t}$ respectively to obtain the indices of \emph{calibration} set $\widehat{\gC}_t \subseteq \gH_t$. After that, for any target level $\alpha$ of FCR control, we report the following prediction interval
% \begin{align}
%     \gI_t^{\rm{CAP}}(X_t;\alpha) = \widehat{\mu}(X_t) \pm q_{\alpha}(\{R_i\}_{i\in \widehat{\gC}_t}).\nonumber
% \end{align}
%where $q_{\alpha}(\{R_i\}_{i\in \widehat{\gC}_t})$ is the $\lceil (1-\alpha)(|\widehat{\gC}_t|+1)\rceil$-st smallest value in $\{R_i\}_{i\in \widehat{\gC}_t}$. 
%To achieve the selection-conditional coverage, the adaptive pick rules $\{\Pi_{t,s}^{\rm{Ada}}(\cdot)\}_{s\in \gH_t}$ needs to be carefully constructed by leveraging the triple $(\Pi_t,\Pi_s, X_t)$. 
%We summarize the general procedure of the online selective conformal prediction in Algorithm \ref{alg:main}, which is abbreviated as CAP in this paper. 
%The next proposition shows that conditional coverage property can be guaranteed under conditional exchangeability. 

\begin{algorithm}[htb!]
	\renewcommand{\algorithmicrequire}{\textbf{Input:}}
	\renewcommand{\algorithmicensure}{\textbf{Output:}}
	\caption{\textbf{C}alibration after \textbf{A}daptive \textbf{P}ick (CAP)}
	\label{alg:main}
	\begin{algorithmic}[1]
	\REQUIRE Pre-trained model $\widehat{\mu}$, initial labeled data $\{(X_{i},Y_i)\}_{i=-n}^{-1}$, FCR level $\alpha \in (0, 1)$.
        \STATE Compute the residuals in the initial labeled data $\{R_{i} = |Y_{i} - \widehat{\mu}(X_{i})|\}_{i=-n}^{-1}$.
	\FOR{$t=0,1,\ldots$}
        \STATE Observe $Y_{t-1}$ and compute $R_{t-1} = |Y_{t-1} - \widehat{\mu}(X_{t-1})|$.
        \STATE Specify the selection rule $\Pi_{t}(\cdot)$ and obtain $S_t = \Pi_{t}(X_t)$.
        \IF{$S_t=1$}
            \STATE Specify the adaptive pick rules $\{\Pi_{t,s}^{\rm{Ada}}(\cdot)\}_{s\in \gH_t}$.
            \STATE Obtain the indices of the picked calibration set $\widehat{\gC}_t = \{s\in \gH_t: \Pi_{t,s}^{\rm{Ada}}(X_s) = 1\}$.
            %\STATE Perform selection on the full calibration set to obtain the subset $\widehat{\gC}_t \subseteq \gH_t$.
            \STATE Report the prediction interval: $\gI_t^{\rm{CAP}}(X_t;\alpha) = \widehat{\mu}(X_t) \pm q_{\alpha}(\LRl{R_i}_{i \in \widehat{\gC}_t})$.
        \ENDIF
        \ENDFOR
        \ENSURE Selected PIs: $\{\gI_t^{\rm{CAP}}(X_t;\alpha): S_t = 1, 0\leq t \leq T\}$.
	\end{algorithmic}
\end{algorithm}

% \begin{proposition}\label{pro:cond_coverage}
%     If residuals $\{R_s\cdot \Pi_{t,s}^{\rm{Ada}}(X_s)\}_{s\in \gH_t}\cup \{R_t\cdot S_t\}$ are exchangeable, Algorithm \ref{alg:main} satisfies the following selection-conditional coverage property
%     \begin{align}
%         \sP\LRl{Y_t \in \gI_t^{\rm{CAP}}(X_t;\alpha) \mid S_t=1} \geq 1-\alpha.\nonumber
%     \end{align}
% \end{proposition}
 
%In Sections \ref{sec:decision_selection} and \ref{sec:symmetric_selection}, we will explore the tractable construction methods of $\widehat{\gC}_t$ for two practical classes of online selection rules. 

The adaptive pick rules $\{\Pi_{t,s}^{\rm{Ada}}(\cdot)\}_{s\in \gH_t}$ are designed to ensure the following symmetric properties:
\begin{equation}\label{eq:indicator_prod_symmetry}
    \text{$\Pi_{t,s}^{\rm{Ada}}(X_s)\cdot \Pi_t(X_t)$ is symmetric to $(X_s,X_t)$},\tag{P-1}
\end{equation}
and
\begin{equation}\label{eq:cal_set_symmetry}
    \widehat{\gC}_t\setminus \{s\} \text{ is symmetric to $(X_s,X_t)$ if $\Pi_{t,s}^{\rm{Ada}}(X_s)\cdot \Pi_t(X_t)=1$}.\tag{P-2}
\end{equation}
It is worthwhile noticing that $\Pi_{t,s}^{\rm Ada}(\cdot)$ and $\Pi_t(\cdot)$ depend on data $\{X_i\}_{0\leq i\leq t-1}$.
The symmetric property \eqref{eq:indicator_prod_symmetry} ensures the pairwise exchangeability \citep{barber2015controlling,zhao2024false} of the calibration data $(X_s,Y_s)$ and the test data $(X_t,Y_t)$ given the joint selection event $\{\Pi_{t,s}^{\rm{Ada}}(X_s)\cdot \Pi_t(X_t) = 1\}$. 
The symmetric property \eqref{eq:cal_set_symmetry} says that the leave-one-out picked calibration set is invariant with swapping $X_s$ and $X_t$ under the joint selection event.
In traditional split conformal prediction, the marginal coverage guarantee relies on the joint exchangeability between test data and calibration data. However, post-selection conformal prediction requires a stronger pairwise exchangeability under the selection events to control SCC. The next proposition
shows that \eqref{eq:indicator_prod_symmetry} and \eqref{eq:cal_set_symmetry} are sufficient conditions for finite-sample SCC control.

\begin{proposition}\label{pro:selection_exchangeable}
 If $\{(X_i,Y_i)\}_{i\geq -n}$ are i.i.d. and the conditions \eqref{eq:indicator_prod_symmetry} and \eqref{eq:cal_set_symmetry} hold, for any $t \geq 0$ with $\sP(S_t = 1) > 0$, we have
    \begin{align}
        \sP\LRl{Y_t \in \gI_t^{\rm CAP}(X_t) \mid S_t=1} \geq 1-\alpha.\nonumber
    \end{align}
    %where $\gD_{-(s,t)} = \{(X_i,Y_i)\}_{i\neq s,t}$.
\end{proposition}

This proposition implies that the key challenge lies in constructing adaptive pick rules for historical data, which depends largely on the selection rules implemented. In this context, we explore two broad classes of selection rules, which are detailed in Sections \ref{sec:decision_selection} and \ref{sec:symmetric_selection}. In addition, we also analyze the real-time FCR control results.

\begin{remark}

Throughout the paper, we use the absolute residual $R(X, Y) = |Y- \widehat{\mu}(X)|$ as the nonconformity score. It is straightforward to extend Algorithm \ref{alg:main} to general nonconformity scores, such as quantile regression \citep{romano2019conformalized} or distributional regression \citep{chernozhukov2021exact}. Let $R(\cdot,\cdot):\sR^n\times \sR \to \sR$ be a general nonconformity score function. We can replace the PI in Algorithm \ref{alg:main} with the following form
\begin{align*}
    \gI_t^{\rm{CAP}}(X_t;\alpha) = \LRl{y\in \sR: R(X_t,y) \leq q_{\alpha}\LRs{\{R(X_i,Y_i)\}_{i \in \widehat{\gC}_t}}}.
\end{align*}
All theoretical results in our paper will remain intact with the PIs defined above.
\end{remark}

\subsection{Related work}

This work is closely related to the post-selection inference on parameters or labels. 
\citet{benjamini2005false} proposed the first method that controls FCR in finite samples by adjusting the confidence level of the marginal confidence interval. Along this path, \citet{weinstein2013selection},  \citet{zhao2022general} and \citet{xu2022post} further investigated how to narrow the adjusted confidence intervals by using more useful selection information. Another line of work is the conditional approach. \citet{fithian2014optimal}, \citet{lee2016exact} and \citet{taylor2018post} proposed to construct confidence intervals for each selected variable conditional on the selection event and showed that the FCR can be further controlled if the conditional coverage property holds for an arbitrary selection subset. %Unfortunately, 
Those methods usually require a tractable conditional distribution given the selection condition. %, which is usually unknown in many real applications. %{Despite that our method is built upon the splitting conformal framework, we emphasize that the selective inference problem in this paper is highly different from conventional inference based on the sample splitting \citep{cox1975note,wasserman2009high,meinshausen2009p}. In these works, two subsets of samples are used to perform inference and selection separately, while the selection and calibration on the test point can depend on the same holdout set in our setting.}
In particular, for the problem of online selective inference, {\citet{weinstein2020online} proposed a solution based on the LORD \citep{ramdas2017online} procedure to achieve real-time FCR control for decision-driven rules.} %However, there were no viable algorithms provided in \citet{weinstein2020online} to construct such confidence intervals using conformal prediction. 
Recently, \citet{xu2023online} introduced a new approach called e-LOND-CI, which utilizes e-values \citep{vovk2021values} with LOND \citep{javanmard2015online} procedure for real-time FCR control. This method alleviates the constraints on selection rules in \citet{weinstein2020online} and provides a valid FCR control under arbitrary dependence, but its setting is much different from the present one in Section \ref{sec:dist_shift}, where we consider integrating feedback information over time. 
% However, one drawback of the two existing methods is that the constructed interval tends to be extremely wide. This issue is essentially caused by ignoring incorporating the selection condition into the construction of intervals.
%Our proposed method is extended from conformal prediction \citep{vovk2005algorithmic}. 

Conformal prediction %\citep{vovk2005algorithmic} 
is the fundamental brick of our proposed method.
As a powerful tool for predictive inference, it provides a distribution-free coverage guarantee in both the regression \citep{lei2018distribution} and the classification \citep{sadinle2019least}. Beyond predictive intervals, conformal inference is also broadly applied to the testing problem by constructing conformal $p$-values \citep{bates2023aos,jin2022selection}.
We refer to \cite{angelopoulos2021gentle} and \cite{shafer2008tutorial} for more comprehensive applications and reviews. The conventional conformal inference requires that the data points are exchangeable, which may be violated in practice. There are several works devoted to conformal inference beyond exchangeability. When the feature shift exists between the calibration set and the test set, \citet{tibshirani2019conformal} and \citet{jin2023model} introduced weighted conformal PIs and weighted conformal $p$-values, respectively, by injecting likelihood ratio weights. For general non-exchangeable data, \citet{barber2022conformal} used a robust weighted quantile to construct conformal PIs. For the online data under distribution shift, \citet{gibbs2021adaptive,gibbs2022conformal} developed adaptive conformal prediction algorithms based on the online learning approach. Besides, a relevant direction is to study the \emph{test-conditional} coverage $\sP\{Y_t \in \gI_t(X_t) \mid X_t = x\}$, which has been proved impossible for a finite-length PI without imposing distributional assumptions \citep{lei2014distribution,Barber20limits}. In contrast, our concerned SCC $\sP\{Y_t \in \gI_t(X_t) \mid S_t = 1\}$ could achieve valid finite-sample guarantee without distributional assumptions. 

{%Recently, we noticed that \citet{jin2024confidence} proposed a new framework named Joint Mondrian Conformal Inference (JOMI) to control the SCC of prediction intervals after selection in test data. JOMI and CAP independently employ the swapping technique to ensure the post-selection exchangeability for symmetric selection rules (see Section \ref{sec:symmetric_selection}), and the two methods achieve SCC guarantee in finite samples, which are both distribution free. However, \citet{jin2024confidence} focused on the SCC in the offline setting, whereas the ultimate goal of our paper is the real-time FCR control, which is more technically involved in handling the temporal dependence of online selection rules. Specifically, \citet{jin2024confidence} investigated the computational aspect of JOMI and proposed efficient implements for the selection rules depending on both covariates and labels. Our paper considers only the symmetric selection rules that depend solely on covariates to avoid extra computational costs. Another recent related work is \citet{gazin2024selecting} studying a different selective predictive inference problem in the offline setting, where the selection process is required to select informative prediction sets. The informative sets are usually defined as specific subsets of the label space. The authors developed a novel framework by performing the BH procedure \citep{benjamini1995controlling} to the carefully constructed p-values, which can control FCR on the selected prediction sets in finite samples. Despite the control targets are same, our paper focuses on selecting samples with some data-driven selection rules instead of selecting prediction intervals satisfying some special informative condition. 

Recently, we noticed that \citet{jin2024confidence} proposed Joint Mondrian Conformal Inference (JOMI) to guarantee the SCC after selection in test data. JOMI and CAP independently employ the swapping technique to ensure post-selection exchangeability for symmetric selection rules (see Section \ref{sec:symmetric_selection}) and achieve finite-sample distribution-free SCC guarantees. In contrast to \citet{jin2024confidence} that focused on the offline setting and label-involved selection rules with practical computation algorithms, we aim to achieve real-time FCR control which requires addressing the temporal dependence issue of online selection rules. In an another related study, \citet{gazin2024selecting} proposed to select informative prediction sets with FCR control by applying the BH procedure \citep{benjamini1995controlling}. Besides, \citet{sarkar2023post} proposed a post-selection framework to ensure simultaneous inference \citep{berk2013valid} across all coverage levels. This approach differs from our focus, which is on inference conditional on the selection event. Table \ref{table:comparison_RW} displays a summary of the comparison with related works in selective conformal prediction. % is given in .
}

\begin{table}[ht]
\begin{minipage}{\textwidth}
\centering
\caption{Comparison with related works in selective conformal prediction}\label{table:comparison_RW}
\resizebox{\textwidth}{!}{\begin{tabular}{llll}
\toprule
Methods & References & Selection rules \footnote{\emph{Decision-driven} selection is defined in Definition \ref{def:decision_driven}. \emph{Symmetric} selection refers to selection rules whose output is invariant to any permutation of the holdout set, and \emph{Joint-symmetric} selection requires this invariance holds for any permutation of the holdout set and test set. Top-K selection refers to the rules where the number of selected test data is fixed as a deterministic integer $K$.}     & Control    \\
\midrule
\underline{\textbf{Offline}}  \\
%cfBH  & \citet{jin2022selection} & BH        & FDR        \\
SCOP & \citet{bao2023selective} & Joint-symmetric \& Top-K          & FCR \& SCC \\
JOMI & \citet{jin2024confidence} & Symmetric        & SCC        \\
InfoSP \& InfoSCOP & \citet{gazin2024selecting} & BH    & FCR        \\
\underline{\textbf{Online}} &\\
LORD-CI & \citet{weinstein2020online} & Decision-driven                 & FCR        \\
e-LOND-CI\footnote{We extend e-LOND-CI to the conformal prediction setting in Appendix \ref{appen:e-lord-ci}.}%, whereas the output intervals are considerably wide in numerical experiments.} 
& \citet{xu2023online} & Arbitrary & FCR\\
{CAP} & {This paper} & Decision-driven \& Symmetric        & FCR \& SCC\\
\bottomrule
\end{tabular}}
\end{minipage}
\end{table}

%Here we stress the difference between \citet{jin2024confidence} and our work: (1) we consider both the selection-conditional coverage and real-time FCR control in the online setting, but they only considered the former in the offline setting; (2) even though the swapping technique in JOMI is similar to the way of constructing calibration set in CAP for the selection with symmetric thresholds, but CAP's construction for decision-driven selection is novel and irrelevant with swapping.

%\section{Online selective conformal prediction}\label{sec:CAP}
\section{CAP for decision-driven selection }\label{sec:decision_selection}

%The motivation for this paper is to generalize the spirit of post-selection calibration to the online setting and design new efficient algorithms to control $\mFCR$ or FCR.

{In this section, we investigate the online selective conformal prediction under the decision-driven selection rules.} % defined below.}
%We will investigate the online selective conformal prediction under a class of selection rules defined below.

\begin{definition}\label{def:decision_driven}
    Let $\sigma(\{S_i\}_{i=0}^{t-1})$ be the $\sigma$-field generated by decisions $\{S_i\}_{i=0}^{t-1}$. The online selection rule is called decision-driven selection if $\Pi_t(\cdot)$ is $\sigma(\{S_i\}_{i=0}^{t-1})$-measurable.
    %\rhj{Check $t$ starting from $0$ or $t>0$.}
\end{definition}

%\rhj{Some examples on decision-driven selection rules should be listed.}
%The decision-driven selection depends on the history data only through the history decisions, which incorporates many online error rate control algorithms (see Section \ref{sec:online_fdr}).
The decision-driven selection depends on historical data only through previous decisions. For example, one can choose $S_t=\Indicator{\widehat{\mu}(X_t)\leq c_t}$, where $c_t=C_1+C_2(\sum_{i=0}^{t-1}S_i)$ for constants $C_1,C_2$. It is more flexible than choosing a constant $c_t\equiv C_1$ as the threshold since we incorporate the cumulative selection number to dynamically adjust the selection rule. Besides, {many online error rate control algorithms \citep{foster2008alpha,aharoni2014generalized}, used for online multiple testing in sequential clinical trials \citep{lee2021statistical} and computational biology \citep{aharoni2010quality}, also fall in the category of decision-driven. We will discuss this selection with online multiple testing in detail in Section \ref{sec:online_fdr}. %{\color{red}These selection rules can be used for online multiple testing in sequential clinical trials \citep{lee2021statistical} and computational biology \citep{aharoni2010quality} {\color{blue}[Haojie: to be specific.]}. 
Before exploring the implementation of CAP under decision-driven selection rules, we introduce the following assumption for FCR control.}
%Before exploring the application of CAP under decision-driven selection rules, we introduce the following assumption for FCR control.

\begin{assumption}\label{assum:independent_initial_holdout}
    The decision-driven selection rules $\{\Pi_{t}(\cdot)\}_{t\geq 0}$ are independent of the initial labeled data $\{(X_i,Y_i)\}_{i=-n}^{-1}$.
\end{assumption}

Since the $\{(X_i,Y_i)\}_{i=-n}^{-1}$ are used only for calibration and the selection rule $\Pi_t$ depends on the previous decisions by Definition \ref{def:decision_driven}, Assumption \ref{assum:independent_initial_holdout} is reasonable for most scenarios. We notice that \citet{weinstein2020online} require the confidence interval $\gI_t(\cdot)$ to be $\sigma(\{S_i\}_{i=0}^{t-1})$-measurable, which means the previously observed data $\{(X_i,Y_i)\}_{i=0}^{t-1}$ \emph{cannot} be used for calibration at time $t$ and the holdout set needs to be fixed as $\{(X_i,Y_i)\}_{i=-n}^{-1}$. We first regard this case as a warm-up and demonstrate that the CAP with a nonadaptive pick on the holdout set is enough to control FCR. Then in the case of the full holdout set $\gH_t$, we show that the nonadaptive pick may fail and present a novel construction for the adaptive pick rules to select calibration data points.

\subsection{Warm-up: fixed holdout set}\label{sec:decision_warm_up}

%{\color{red}In this subsection, we assume the holdout set is fixed as $\{(X_i,Y_i)\}_{i=-n}^{-1}$ in the entire online process, namely $\gH_t = \{-n,\ldots,-1\}$ for any $t\geq 0$.} 
{Here, we use the initial labeled data $\{(X_i,Y_i)\}_{i=-n}^{-1}$ as a fixed holdout set in the entire online process, namely, $\gH_0 = \{-n,\ldots,-1\}$ instead of $\gH_t$ in Lines 6 and 7 of Algorithm \ref{alg:main}.} 
Under Assumption \ref{assum:independent_initial_holdout}, the product of selection indicators $\Pi_t(X_t)\Pi_t(X_s)$ is symmetric to $(X_s,X_t)$ for $s\in \gH_0$ because $\Pi_t$ is independent of both $X_t$ and $X_s$. Therefore, the nonadaptive pick on the fixed holdout set is enough to guarantee the SCC. When $S_t = 1$, we perform $\Pi_t$ on $\{X_i\}_{i=-n}^{-1}$ and obtain the calibration set $\{(X_s,Y_s): \Pi_{t}(X_{s}) = 1\}_{s\in \gH_0}$. With this selected calibration set, the next theorem shows that the real-time FCR can be controlled below $\alpha$ and owns an anti-conservative lower bound.%, and the control is not conservative.

\begin{theorem}\label{thm:FCR_decision_select}
    Under Assumption \ref{assum:independent_initial_holdout}, if we use the fixed holdout set $\{(X_s,Y_s)\}_{s\in \gH_0}$ at time $t$ in Algorithm \ref{alg:main} and set $\Pi_{t,s}^{\rm Ada}(\cdot) = \Pi_t(\cdot)$ for $s\in \gH_0$, it satisfies: (1) For any $T \geq 0$, $\FCR(T) \leq \alpha$; (2) Let $p_t = \sP\LRl{S_t = 1 \mid \sigma(\{S_i\}_{i=0}^{t-1})}$. If the residuals $\{R_i\}_{i = -n}^T$ are distinct and $\sum_{t=0}^T S_t > 0$ almost surely, we also have the following lower bound,
        \begin{align}\label{eq:anti_conservative_bound}
            \FCR(T) \geq \alpha - \E\LRm{\frac{\sum_{t=0}^T S_t\LRl{\frac{1 - (1 - p_t)^{n+1}}{(n+1)p_t}}}{\sum_{j=0}^T S_j} }.
        \end{align}
\end{theorem}

Theorem \ref{thm:FCR_decision_select} reveals that the CAP achieves finite-sample and distribution-free FCR control. Similar to the marginal coverage of split conformal~\citep{lei2018distribution}, we also have the anti-conservative guarantee in \eqref{eq:anti_conservative_bound} when the residuals are continuous. The quantity $(n+1)p_t$ characterizes the size of the picked calibration set. If the selection probability $p_t$ is bounded above zero, %or $\min_{0\leq t \leq T}p_t = O_p(1)$, 
then the lower bound \eqref{eq:anti_conservative_bound} becomes $\FCR(T) \geq \alpha - O\LRs{n^{-1}}$. Consequently, we have exact FCR control in the asymptotic regime, i.e., $\lim_{(n,T)\to \infty} \FCR(T) = \alpha$.

For completeness and comparison, we also provide the construction and validity of the online adjusted method named LORD-CI proposed by \citet{weinstein2020online} in the conformal setting. Given any $\sigma(\{S_i\}_{i=0}^{t-1})$-measurable coverage level $\alpha_t \in (0,1)$, a \emph{marginal} split conformal PI is constructed as
\begin{align}\label{eq:def_marginal_PI}
    \gI_t^{\rm{marg}}(X_t;\alpha_t) = \widehat{\mu}(X_t) \pm q_{\alpha_t}\LRs{\{R_i\}_{i\in \gH_0}},
\end{align}
where $q_{\alpha_t}\LRs{\{R_i\}_{i\in \gH_0}}$ is the $\lceil (1-\alpha_t)(n+1)\rceil$-st smallest value in $\{R_i\}_{i\in \gH_0}$. 
The PI \eqref{eq:def_marginal_PI} can serve as a recipe for LORD-CI by dynamically updating the marginal level $\alpha_t$ to maintain the following invariant
\begin{align}\label{eq:FCP_estimate}
    \frac{\sum_{t=0}^T \alpha_t}{1 \vee \sum_{j=0}^{T} S_j} \leq \alpha,\quad \forall~ T\geq 0.
\end{align}
We refer to \citet{weinstein2020online} and literature therein for explicit procedures in constructing the sequence $\{\alpha_t\}_{t\geq 0}$ satisfying \eqref{eq:FCP_estimate}.  {The left hand side of \eqref{eq:FCP_estimate} is an over-conservative upper bound of $\FCP(T)$ by discarding $S_t$ in the numerator, which yields the following result
\begin{align}
    \FCR(T) \leq \E\left[\frac{\sum_{t=0}^T \Indicator{Y_t\not\in\gI_t^{\rm marg}(X_t;\alpha_t)}}{1 \vee \sum_{j=0}^{T} S_j}\right]\leq \E\left[\frac{\sum_{t=0}^T \alpha_t}{1 \vee \sum_{j=0}^{T} S_j}\right] \leq \alpha,\nonumber
\end{align}
where the second inequality holds since $\Pi_t$ is decision-driven, and the last inequality holds due to \eqref{eq:FCP_estimate}. Hence LORD-CI has information loss about the selection event.} Under the same conditions in Theorem 2 of \citet{weinstein2020online}, we can obtain the FCR control results for LORD-CI in the conformal prediction setting.

\begin{proposition}\label{pro:LORD_CI}
    Let $\{S_j\}_{j=0}^T$ and $\{\widetilde{S}_j\}_{j=0}^T$ be two decision sequences, suppose $S_t \geq \widetilde{S}_t$ holds whenever $S_j\geq \widetilde{S}_j$ for any $j\leq t-1$. Under Assumption \ref{assum:independent_initial_holdout}, if $\alpha_t \in \sigma(\{S_i\}_{i=0}^{t-1})$ for any $t\geq 0$ and \eqref{eq:FCP_estimate} holds, the LORD-CI algorithm satisfies that $\FCR(T) \leq \alpha$ for any $T \geq 0$.
\end{proposition}

Despite that LORD-CI controls the real-time FCR, the marginal PI $\gI_t^{\rm{marg}}(X_t;\alpha_t)$ tends to be wider as $t$ grows because $\alpha_t$ may shrink to zero when few selections are made. The PIs output by CAP will be relatively narrower due to the \emph{constant} miscoverage level $\alpha$, which is also confirmed by Figure \ref{fig:alg_illustration} and numerical results in Section \ref{sec:experiments}.

\subsection{Full holdout set}

{Since we can observe new labels at each time step, it is more efficient to include all previously observed labeled data in the holdout set.} However, using the full holdout set results in additional dependence between the current decision $S_t$ and historical data $\{(X_s,Y_s)\}_{s\in \gH_t}$. 

\subsubsection{Non-adaptive pick rule}
Typically, if we still conduct nonadaptive pick on $\{(X_s,Y_s)\}_{s\in \gH_t}$ to obtain the picked calibration set indexed by
\begin{align}\label{eq:naive_cal_set}
    \gN_t = \{s\in \gH_t: \Pi_{t}(X_s) = 1\}.
\end{align}
The next theorem characterizes the FCR and SCC control error for Algorithm \ref{alg:main} with the nonadaptive pick rule.

\begin{theorem}\label{thm:FCR_decision_dyn_SCOP}
    Under Assumption \ref{assum:independent_initial_holdout}, we use the full holdout set $\{(X_s,Y_s)\}_{s\in \gH_t}$ at time $t$ in Algorithm \ref{alg:main} and set $\Pi_{t,s}^{\rm Ada}(\cdot) = \Pi_t(\cdot)$ for $s\in \gH_t$. Define the error term
    \begin{align}
        \Delta_t = \sum_{s=0}^{t-1} \frac{\Pi_t(X_s)\mathbbm{1}\{\Pi_s(X_t) \neq \Pi_s(X_s)\}}{|\widehat{\gC}_t|+1} \LRs{\mathbbm{1}\{R_t > Q_{\alpha}(\{R_i\}_{i\in \widehat{\gC}_t \cup \{t\})}\} - \mathbbm{1}\{R_s > Q_{\alpha}(\{R_i\}_{i\in \widehat{\gC}_t \cup \{t\})}\}},\nonumber
    \end{align}
    where $Q_{\alpha}(\{R_i\}_{i\in \widehat{\gC}_t \cup \{t\})}$ denotes the $\lceil(1-\alpha)(|\widehat{\gC}_t|+1)\rceil$-th smallest value in $\{R_i\}_{i\in \widehat{\gC}_t \cup \{t\}}$.
    Then for any $T \geq 0$, we have
        \begin{align}
            \FCR(T) \leq \alpha + \sum_{t = 0}^T \E\LRm{ \frac{S_t \Delta_t}{1 \vee \sum_{j=0}^T S_j}}.\nonumber
        \end{align}
    In addition, for any $t \geq 0$ and $\sP(S_t = 1) >0$, we also have
    \begin{align}
        \sP\LRl{Y_t \in \gI_t^{\rm CAP}(X_t;\alpha) \mid S_t = 1} \geq 1-\alpha - \frac{\E\LRm{S_t \Delta_t}}{\sP(S_t=1)}.\nonumber
    \end{align}
\end{theorem}

Notice that the product of selection indicator $\Pi_t(X_t) \Pi_t(X_s)$, exhibits a non-symmetric dependence on the features $X_t$ and $X_s$. In fact, the selection rule $\Pi_t$ is independent of $X_t$ but relies on $\{X_s\}_{s\in \gH_t}$ through historical decisions $\{S_s\}_{s\in \gH_t}$. Hence, the symmetric property \eqref{eq:indicator_prod_symmetry} does not hold. Next, we discuss two scenarios where the error vanishes.

The following corollary shows that for the nonincreasing selection rule, the additional error $S_t\Delta_t = 0$ since $\Pi_t(X_t)\Pi_t(X_s) = 1$ implies $\Pi_s(X_t) = \Pi_s(X_s) = 1$. For example, $\Pi_t(x) = \mathbbm{1}\{\widehat{\mu}(x) \geq \tau_0 \sum_{j=0}^{t-1}S_j\}$ for some $\tau_0 > 0$. In this case, we can show symmetric properties \eqref{eq:indicator_prod_symmetry} and \eqref{eq:cal_set_symmetry} hold for any time $t$, and both FCR and SCC can be controlled.

\begin{corollary}\label{cor:scop-1}
    Under the same setting of Theorem \ref{thm:FCR_decision_dyn_SCOP}, if the selection rule is nonincreasing over time, that is $\Pi_t(x) \leq \Pi_s(x)$ holds for any $s\leq t$ and $x\in \sR^d$, we have $\FCR(T) \leq \alpha$ and $\sP\LRl{Y_t \in \gI_t^{\rm CAP}(X_t;\alpha) \mid S_t = 1} \geq 1-\alpha$ when $\sP(S_t = 1) >0$.
\end{corollary}

In addition, the next corollary shows that if the selection rule tends to be stable, that is, $\Pi_t(\cdot)$ returns the same value if we replace one historical data point, then $\E[S_t\Delta_t] = 0$. For example, the selection rule $\Pi_t(x) = \mathbbm{1}\{\widehat{\mu}(x) \leq \min\{\tau_0, \sum_{j=0}^{t-1} S_j\}\}$ with $\tau_0 > 0$ and a bounded predictor, becomes $\mathbbm{1}\{\widehat{\mu}(x) \leq \tau_0\}$ when $\sum_{j=0}^{t-1} S_j > \tau_0 + 1$. 

\begin{corollary}\label{cor:scop-2}
    Under the same setting of Theorem \ref{thm:FCR_decision_dyn_SCOP}. Let $\{\Pi_j^{(s\gets t)}(\cdot)\}_{j \geq s+1}$ be the selection rules generated by replacing $X_s$ with $X_t$ for $0\leq s \leq t-1$.  If there exists some finite time $t_0$, $\Pi_t^{(s\gets t)}(\cdot) = \Pi_t(\cdot)$ holds for any $t \geq t_0 +1$, we have $\sP\LRl{Y_t \in \gI_t^{\rm CAP}(X_t;\alpha) \mid S_t = 1} \geq 1-\alpha$ for any $t \geq t_0 + 1$. Further, if $\lim_{T\to \infty}\sum_{j=0}^T S_j \to \infty$, we also have $\limsup_{T\to\infty} \FCR(T) \leq \alpha$.
\end{corollary}

\subsubsection{Adaptive pick rule}
%For the simplicity of definition, we let $\Pi_s(\cdot) \equiv 1$ for offline point $-n \leq s \leq -1$.
To make the symmetric properties \eqref{eq:indicator_prod_symmetry} and \eqref{eq:cal_set_symmetry} be satisfied for arbitrary decision-driven selection, we set the adaptive pick rules as
\begin{align}\label{eq:Ada_rule_inter}
\Pi_{t,s}^{\rm Ada}(\cdot) = \Pi_t(\cdot) \prod_{i\in \gN_t^{\rm on}}\Indicator{\Pi_i(\cdot) = \Pi_i(X_t)},
\end{align}
where $\gN_t^{\rm on} = \{0\leq i \leq t-1: \Pi_t(X_i)=1\} \subseteq \gN_t$. By definition, we know $\widehat{\gC}_t = \{s\in \gH_t: \Pi_{t,s}^{\rm Ada}(X_s) = 1\}$ is a subset of the calibration points $\gN_t$ picked by the nonadaptive rule, see \eqref{eq:naive_cal_set}.

% , we can equivalently write the product of selection indicator in \eqref{eq:indicator_prod_symmetry} as
% \begin{align}\label{eq:joint_selection_sum}
%     \Pi_t(X_t)\cdot \Pi_{t,s}^{\rm Ada}(X_s) &= \Pi_t(X_t)\Pi_t(X_s)\mathbbm{1}\{\Pi_s(X_s) = \Pi_s(X_t)\} \prod_{i\in \gN_t^{\rm on}\setminus \{s\}}\Indicator{\Pi_i(X_s) = \Pi_i(X_t)}\nonumber\\
%     &= \Pi_t(X_t)\Pi_t(X_s) \Pi_s(X_t)\Pi_s(X_s)\prod_{i\in \gN_t^{\rm on}\setminus \{s\}}\Indicator{\Pi_i(X_s) = \Pi_i(X_t)}\nonumber\\
%     &+ \Pi_t(X_t)\Pi_t(X_s) [1-\Pi_s(X_t)][1-\Pi_s(X_s)]\prod_{i\in \gN_t^{\rm on}\setminus \{s\}}\Indicator{\Pi_i(X_s) = \Pi_i(X_t)}.
% \end{align}
\begin{remark}
    For offline point $-n \leq s \leq -1$, we can directly check that both \eqref{eq:indicator_prod_symmetry} and \eqref{eq:cal_set_symmetry} hold since $\{\Pi_i(\cdot)\}_{i= 0}^t$ are independent of $(X_s,X_t)$.
    For online point $0\leq s \leq t-1$, we can check two properties according to the decomposition $\mathbbm{1}\{\Pi_s(X_s) = \Pi_s(X_t)\} = \Pi_s(X_t)\Pi_s(X_s) + [1-\Pi_s(X_t)][1-\Pi_s(X_s)]$.
    Notice that if $\Pi_s(X_s) = 1$, we can replace $X_s$ with some $x_s^* \in \sigma(\{S_i\}_{i\leq s-1})$ such that $\Pi_s(x_s^*) = 1$. It will generate a sequence of \emph{virtual} selection rules, denoted by $\{\widetilde{\Pi}_{j}^{(s)}(\cdot)\}_{j\geq s+1}$. By Definition \ref{def:decision_driven}, we know $\widetilde{\Pi}_{t}^{(s)}$ is identical to the \emph{real} selection rule $\Pi_t$ under the event $\{\Pi_s(X_s) = 1\}$. Then we can verify \eqref{eq:indicator_prod_symmetry} and \eqref{eq:cal_set_symmetry} using the fact $\{\widetilde{\Pi}_i^{(s)}\}_{i\geq s+1}$ and $\{\Pi_i\}_{i\leq s}$ are independent of $(X_s,X_t)$. The verification under the counterpart $\Pi_s(X_s) = 0$ follows a similar decoupling analysis. This leave-one-out technique is used in \citet{weinstein2020online} to prove FCR control of LORD-CI; here, we leveraged it differently to verify the post-selection exchangeability. The detailed verification of two symmetric properties is deferred to Appendix \ref{proof:thm:FCR_decision_dyn_mSCOP}.
\end{remark}

\begin{remark}
    Recently, \citet{sale2025online} proposed a new procedure named EXPRESS to pick calibration points from historical data, which coincides with the main idea of CAP and also guarantees finite-sample FCR and SCC control. However, the derivation of CAP in \eqref{eq:Ada_rule_inter} is significantly different from EXPRESS. While, EXPRESS is designed to satisfy the \emph{global symmetry}: the index set $\widehat{\gC}_t \cup \{t\}$ is invariant to the permutation of all historical data $\{(X_i,Y_i)\}_{i=-n}^t$ if $S_t = 1$. Our approach incorporates two specific symmetric properties \eqref{eq:indicator_prod_symmetry} and \eqref{eq:cal_set_symmetry} restricted within the picked calibration points, as we discussed earlier. Notably, the global symmetry condition automatically implies \eqref{eq:indicator_prod_symmetry} and \eqref{eq:cal_set_symmetry}, meaning that the calibration set picked by EXPRESS is always a subset of that picked by CAP. In Appendix \ref{appen:compare_express}, we provide a comprehensive comparison of the two methods.
\end{remark}

% Then the first summand in \eqref{eq:joint_selection_sum} equals to
% \begin{align}
%     \widetilde{\Pi}_{t}^{(s)}(X_t)\widetilde{\Pi}_t^{(s)}(X_s)\Pi_s(X_t)\Pi_s(X_s)\prod_{i\in \gN_t\setminus \{s\}}\Indicator{\widetilde{\Pi}_i^{(s)}(\cdot) = \widetilde{\Pi}_i^{(s)}(X_t)},\nonumber
% \end{align}
% which is symmetric on $(X_s,X_t)$ since $\{\widetilde{\Pi}_i^{(s)}\}_{i\geq s+1}$ and $\{\Pi_i\}_{i\leq s}$ are independent of $(X_s,X_t)$. Similarly, we can also show that the second summand in \eqref{eq:joint_selection_sum} is also symmetric to $(X_s,X_t)$, then the symmetric property \eqref{eq:indicator_prod_symmetry} holds. In addition, $s\in \widehat{\gC}_t$ also implies that $\Pi_i(X_s) = \Pi_i(X_t)$ for any $i\in \gN_t^{\rm on}$, hence symmetric property \eqref{eq:cal_set_symmetry} can be satisfied. 

The next theorem shows that CAP with adaptive pick rules in \eqref{eq:Ada_rule_inter} achieves finite-sample SCC and FCR control.

\begin{theorem}\label{thm:FCR_decision_dyn_mSCOP}
    Under Assumption \ref{assum:independent_initial_holdout}, if we use the full holdout set $\{(X_s,Y_s)\}_{s\in \gH_t}$ at time $t$ in Algorithm \ref{alg:main} and set $\Pi_{t,s}^{\rm Ada}(\cdot)$ in \eqref{eq:Ada_rule_inter}, then: (1) $\FCR(T) \leq \alpha$ for any $T \geq 0$; (2) $\sP\LRl{Y_t \in \gI_t^{\rm{CAP}}(X_t;\alpha) \mid S_t=1} \geq 1-\alpha$ for any $t \geq 0$ when $\sP(S_t = 1) > 0$.
\end{theorem}

After modifying the rules to pick calibration points from $\Pi_t(\cdot)$ to \eqref{eq:Ada_rule_inter}, Algorithm \ref{alg:main} was guaranteed to have finite-sample and distribution-free control of FCR in the full holdout set case, as well as that of the SCC. %{\color{red}
However, by the simulation results in Appendix \ref{appen:compare_inter}, we find the adaptive pick rule \eqref{eq:Ada_rule_inter} is conservative, and outputs PI with infinite length sometimes. Hence, for nonincreasing selection rules in Corollary \ref{cor:scop-1} and asymptotically stable selection rules, we suggest advocating Algorithm \ref{alg:main} with the nonadaptive pick rule \eqref{eq:naive_cal_set}.%}

\subsection{Selection with online multiple testing procedure}\label{sec:online_fdr}

In this section, we apply the CAP to online multiple testing problems in the framework of conformal inference. Given any user-specified thresholds $\{c_t\}_{t\geq 0}$, we have a sequence of hypotheses defined as
%\begin{align}\label{eq:online_hypothesis}
\begin{align*}
    H_{0,t}: Y_t \leq c_t,\quad\text{for }t\geq 0.    
\end{align*}
%\end{align}
At time $t$, we need to make the real-time decision whether to reject $H_{0,t}$ or not. In this vein, constructing PIs for the rejected candidates is a post-selection predictive inference problem. The validity of Algorithm \ref{alg:main} holds with any online multiple testing procedure that is decision-driven as Definition \ref{def:decision_driven}.

To control FDR in the online setting, \citet{foster2008alpha} proposed firstly one method called the alpha-investing algorithm. Then \citet{aharoni2014generalized} extended it to the generalized alpha-investing (GAI) algorithm. After that, a series of works developed several variants of GAI, such as LORD, LOND ~\citep{javanmard2015online}, LORD++~\citep{ramdas2017online} and SAFFRON~\citep{ramdas2018saffron}. Suppose we have access to a series of $p$-values $\{p_t\}_{t\geq 0}$, where $p_t$ is independent of samples in holdout set $\gH_t$. Given the target FDR level $\beta \in (0,1)$, these procedures proceed by updating the significance level $\beta_t$ based on historical information and rejecting $H_{0,t}$ if $p_t \leq \beta_t$. %Fortunately, all these online procedures are decision-driven selections, and CAP can naturally provide a theoretical guarantee for FCR control. 
Fortunately, all these online procedures are decision-driven selections {for independent $p$-values. We construct the conformal $p$-values using an additional labeled data set and then those $p$-values are independent conditional on this set. Thus, CAP can naturally provide FCR control guarantee for the online multiple testing procedure.} 
Regarding the $p$-values in the framework of conformal inference, we refer to \citet{bates2023aos} and \citet{jin2022selection} for recent developments. {In Appendix \ref{appen:conformal_p}, we also discuss how to construct conformal $p$-values for online multiple testing procedures that are super-uniform conditional on one additional set, making the online FDR control available.}

%provides an approach to converting the prediction $\widehat{\mu}(X_t)$ to a valid $p$-value for $H_{0,t}$, which is called conformal $p$-value. In the offline setting, \citet{jin2022selection} investigated the multiple testing problem in \eqref{eq:online_hypothesis} by introducing a new conformal $p$-value. The authors also proved that the Benjamini-Hochberg procedure can successfully control the the false discovery rate (FDR) below the target level. 

%\begin{remark}
 %   It is worthwhile noticing that conformal $p$-values $\{p_t\}_{t\geq 0}$ are not independent since they are constructed from the same labeled set. \citet{bates2023aos} verified that $\{p_t\}_{t\geq 0}$ are positively dependent on a subset (PRDS), a special dependence condition in the offline FDR literature \citep{benjamini2001control}. 
  %  To achieve a finite-sample control, most online FDR methods require that the $p$-values are independently super-uniform or conditionally super-uniform. 
  %  In particular, \citet{zrnic2021asynchronous} proved that the original LOND procedure and its generalization can control FDR under the PRDS condition. In our paper, we focus on the real-time FCR control for any decision-driven selection procedures deployed by users or analysts. Therefore, we can still wrap Algorithm \ref{alg:main} around other online multiple testing procedures as long as they are decision-driven.
%\end{remark}

\section{CAP for selection with symmetric thresholds}\label{sec:symmetric_selection}

In the decision-driven selection rules, the influence of historical data on the current selection rule is entirely determined by past decisions. It may be inappropriate in some cases where the analyst wants to use the empirical distribution of historical data to select candidates. To adapt this scenario, we rewrite the selection rule in a threshold form. Let $V(\cdot): \sR^d \to \sR$ be a user-specific or pre-trained score function used for selection, and then denote $V_i = V(X_i)$ for $i\geq -n$. For ease of presentation, we let the selection rule at time $t$ be
\begin{align}\label{eq:def_St}
    \Pi_t(\cdot) = \Indicator{V(\cdot) \leq \gA_t\LRs{\{V_i\}_{i \in \gH_t}}},
\end{align}
where $\{\gA_t: \sR^{t+n} \to \sR\}_{t\geq 0}$ is a sequence of deterministic functions. This class of selection rules has not been studied in \citet{weinstein2020online}. In particular, the selection function is assumed to have the following symmetric property.
\begin{definition}\label{def:symmetric_selction}
    The threshold function $\gA_t$ is symmetric if $\gA_t(\{V_i\}_{i \in \gH_t}) = \gA_t(\{V_{\pi(i)}\}_{i \in \gH_t})$ where $\pi$ is a permutation in $\gH_t$. 
\end{definition}

For example, if $\mathcal{A}_t$ outputs the sample mean or sample quantile of historical scores $\{V_i\}_{i\in \gH_t}$, then the corresponding selection rule $\Pi_t$ is symmetric. {Such selection strategies are commonly used in online recruitment \citep{faliagka2012application} and online recommendation \citep{adomavicius2016classification}.}
Consider the nonadaptive strategy: if $S_t = 1$, we use the same threshold to perform screening on history scores $\{V_i\}_{i\in \gH_t}$, and then obtain picked calibration set $\gN_t = \LRl{s\in \gH_t: \Pi_t(X_s) = 1} = \LRl{s\in \gH_t: V_s \leq \gA_t\LRs{\{V_i\}_{i\in \gH_t}}}$. However, the corresponding product of selection indicators $\Indicator{V_s \leq \gA_t(\{V_i\}_{i\in \gH_t})} \Indicator{V_t \leq \gA_t(\{V_i\}_{i\in \gH_t})}$ is not symmetric with respect to $(X_s,X_t)$ for $s \in \gH_t$, which means \eqref{eq:indicator_prod_symmetry} does not hold. To address the asymmetric issue, one natural and viable solution is swapping the score from the holdout set $V_s$ and the score $V_t$ in the expression of $\Pi_t(X_t)$ in \eqref{eq:def_St}, which leads to the following adaptive pick rule
% $\widehat{\gC}_t = \LRl{s\in \gH_t: V_s \leq \gA_t\LRs{\{V_i\}_{i\in \gH_t, i\neq s}, V_t}}$.
% Hence, we may determine the adaptive pick rules in Algorithm \ref{alg:main} as 
\begin{align}\label{eq:Ada_rule_swap}
    \Pi_{t,s}^{\rm{Ada}}(\cdot) = \Indicator{V(\cdot) \leq \gA_t\LRs{\{V_i\}_{i\in \gH_t, i\neq s}, V_t}}.
\end{align}
We provide the verification of \eqref{eq:indicator_prod_symmetry} and \eqref{eq:cal_set_symmetry} for the above pick rule in Appendix \ref{proof:thm:mFCR_swap_scop}.
The next theorem shows that exact SCC can be guaranteed in finite samples after swapping.

\begin{theorem}\label{thm:mFCR_swap_scop}
    If the selection functions $\{\gA_t\}_{t\geq 0}$ are symmetric as Definition \ref{def:symmetric_selction}, then Algorithm \ref{alg:main} with $\Pi_{t,s}^{\rm Ada}(\cdot)$ defined in \eqref{eq:Ada_rule_swap} satisfies $\sP\LRl{Y_t \in \gI_t^{\rm{CAP}}(X_t;\alpha) \mid S_t=1} \geq 1-\alpha$ for any $t \geq 0$ when $\sP(S_t = 1) > 0$.
\end{theorem}

{JOMI \citep{jin2024confidence} also uses a similar strategy to achieve the finite-sample selection-conditional guarantee without any distributional assumptions in the offline setting. As proved by \citet{jin2024confidence}, the SCC guarantee is not sufficient for FCR control, even in the offline setting.} To analyze the FCR value of CAP, we impose the stability condition to bound the change of $\gA_t$'s output after replacing $V_s$ with an independent copy $V$.

\begin{assumption}\label{assum:swap_sensitivity}
    There exists a sequence of positive real numbers $\{\sigma_t\}_{t\geq 0}$ such that,
    \begin{align*}
        \max_{s\in \gH_t}\E\LRm{\big| \gA_t\LRs{\{V_i\}_{i\in \gH_t}} - \gA_t\LRs{\{V_i\}_{i\in \gH_t, i\neq s}, V} \big| \mid \{V_i\}_{i\in \gH_t, i\neq s}} \leq \sigma_t,
    \end{align*}
    where $V$ is an i.i.d. copy of $V_s$.
\end{assumption}

Since two sets $\{V_i\}_{i\in \gH_t}$ and $\{V_i\}_{i\in \gH_t,i\neq s} \cup \{V\}$ only differ one data point, the definition of $\sigma_t$ in Assumption \ref{assum:swap_sensitivity} is similar to the global sensitivity of $\gA_t$ in the differential privacy literature \citep{dwork2006calibrating}. %{\color{red} Our assumption is weaker than the global sensitivity because the bound is added to the expectation taken on $V_s$ and $V$.}

\begin{theorem}\label{thm:FCR_swap}
    Suppose the density function of $V_i$ is upper bounded by $\rho > 0$. If the symmetric function $\gA_t$ satisfies Assumption \ref{assum:swap_sensitivity} for any $t\geq 0$. Algorithm \ref{alg:main} with $\Pi_{t,s}^{\rm Ada}(\cdot)$ defined in \eqref{eq:Ada_rule_swap} satisfies that
    \begin{align}\label{eq:swap_FCR_bound}
        \FCR(T) \leq \alpha\cdot \LRs{1 + \E\LRm{\frac{\mathds{1}\left\{\sum_{j=0}^T S_j > 0\right\} \epsilon(T)}{\LRl{\sum_{j=0}^T S_j - \epsilon(T)}\vee 1}} + \frac{9}{T+n}},
    \end{align}
    where $\epsilon(T) = 2\sum_{j= 0}^{T-1} \sigma_j + 3(\sqrt{e \rho}+1)\log(T+n) + 2^{-1}$.
\end{theorem}

To deal with the complicated dependence between selection and calibration, \citet{bao2023selective} imposed a condition on the joint distribution for the pair of the residual and selection score $(R_i, V_i)$. Due to the swapping design of $\widehat{\gC}_t$, this assumption is no longer required to obtain FCR bound. The distributional assumption on $V_i$ in Theorem \ref{thm:FCR_swap} is thus quite mild.

\begin{remark}
    To analyze FCR, we need to decouple the dependence between the numerator and the denominator of FCP. The conventional leave-one-out analysis in online error rate control does not work for the selection function $\gA_t$. In the proof of Theorem \ref{thm:FCR_swap}, we address this difficulty by using the exchangeability of data and symmetricity of $\gA_t$. We construct a sequence of virtual decisions $\{S_j^{(s\gets t)}\}_{j=s}^{t-1}$ by replacing $V_s$ with $V_t$ in the real decisions $\{S_j\}_{j=s}^{t-1}$. Since the function $\gA_j$ is symmetric, we can guarantee that $S_j^{(s\gets t)}$ and $S_j$ have the same distribution. The additional error $\epsilon(T)$ in \eqref{eq:swap_FCR_bound} comes from the difference term $\sum_{j=s}^{t-1} S_j - S_j^{(s\gets t)}$, which can be bounded via empirical Bernstein's inequality.
\end{remark}

Next, we will show that the error $\epsilon(T)$ can be upper bounded by a logarithmic factor with high probability when $\gA_t$ returns the historical mean or quantile.

\begin{proposition}\label{pro:mean_selection}
    Suppose $\gA_t$ returns the sample mean of history scores, i.e., $\sum_{i\in \gH_t} V_i/|\gH_t|$. If $\E[|V_i|] \leq \sigma$ for some $\sigma > 0$, then we have $\epsilon(T) \leq 4(\sqrt{e \rho} + \sigma +1)\log(T+n)$.
\end{proposition}

\begin{proposition}\label{pro:quantile_selection}
    Suppose $\gA_t$ returns the $\vartheta$-th sample quantile of history scores $\{V_i\}_{i\in \gH_t}$ for $\vartheta \in (0,1]$. If $\{V_i\}_{i = -n}^T$ are continuous and $n\geq 9$, then with probability at least $1 - (T+n)^{-2}$, we have $\epsilon(T) \leq 36\log^2(T+n)$.
\end{proposition}
%{\color{red}For the mean-based selection, the stability assumption reduces to the bound on the first-order moment of $V_i$, which is quite weak in traditional statistical literature. For the quantile-based selection, the $\rho$-bounded density condition is dropped because we can always apply selection on the transformed scores $\{F_v(V_i)\}_{i\geq 0}$ and get the same decision sequence and selected calibration set, where $F_v$ is the cumulative distribution function of $V_i$.} 
Plugging the upper bounds in Propositions \ref{pro:mean_selection} and \ref{pro:quantile_selection} into \eqref{eq:swap_FCR_bound}, we see that Algorithm \ref{alg:main} is asymptotically valid for FCR control if $\log(T+n)/(\sum_{j=0}^T S_j) = o_p(1)$ for these two cases.

%In the Appendix \ref{subsec:cal_sel_fixed_set}, we also provide the control results of CAP under the symmetric selection 

%\subsection{Application: online marginal testing based on conformal $p$-value}

\section{CAP under distribution shift}\label{sec:dist_shift}

In some online settings, the exchangeable (or i.i.d.) assumption on the data generation process does not hold anymore, in which the distribution of $(X_t,Y_t)$ may vary smoothly over time. Without exchangeability, the marginal coverage cannot even be guaranteed. \citet{gibbs2021adaptive} developed an algorithm named adaptive conformal inference (ACI), which updates the miscoverage level according to the historical feedback on under/over coverage. For a marginal target level $\alpha$, the ACI updates the current miscoverage level by
\begin{align}\label{eq:ACI_rule}
    \alpha_{t} = \alpha_{t-1} + \gamma \LRs{\alpha - \indicator{Y_{t-1} \not\in \gI_{t-1}(X_{t-1};\alpha_{t-1})}},
\end{align}
where $\gamma > 0$ is the step size parameter. \citet{gibbs2022conformal} further showed that the ACI updating rule is equivalent to a gradient descent step on the pinball loss $\ell(\theta;\beta_t) = \alpha (\beta_t - \theta) - \min\{0, \beta_t - \theta\}$, where $\beta_t = \sup\{\beta\in [0,1]: Y_t \in \gI_t(X_t;\beta)\}$. That is, the miscoverage level in \eqref{eq:ACI_rule} can be written as
\begin{align}\label{eq:ACI_gd}
    \alpha_t = \alpha_{t-1} - \gamma \nabla\ell(\alpha_{t-1};\beta_{t-1}),
\end{align}
where $\nabla\ell(\alpha_{t-1};\beta_{t-1})$ is the subgradient of pinball loss.
By re-framing the ACI into an online convex optimization problem over the losses $\{\ell(\cdot;\beta_t)\}_{t\geq 0}$, \citet{gibbs2022conformal} proposed a dynamically-tuned adaptive conformal inference (DtACI) algorithm by employing an exponential reweighting scheme \citep{vovk1990aggregating,wintenberger2017optimal,gradu2023adaptive}, which can dynamically estimate the optimal step size $\gamma$.

\begin{algorithm}[htb]
	\renewcommand{\algorithmicrequire}{\textbf{Input:}}
	\renewcommand{\algorithmicensure}{\textbf{Output:}}
	\caption{Selective DtACI with CAP}
	\label{alg:cond_DtACI}
	\begin{algorithmic}[1]
	\REQUIRE Set of candidate step-sizes $\{\gamma_i\}_{i=1}^k$, starting points $\{\alpha_0^i\}_{i=1}^k$, tuning parameter sequence $\{\phi_t,\eta_t\}_{t=0}^T$.
        \STATE \textbf{Initialize:} $\tau \gets \min\{t: S_t = 1\}$, $w_{\tau}^i \gets 1$, $p_{\tau}^i \gets 1/k$, $\alpha_{\tau} \gets \alpha_{0}^i$ with probability $p_{\tau}^i$;
        \STATE Call Algorithm \ref{alg:main} and return $\gI_{\tau}^{\rm{CAP}}(X_{\tau};\alpha_{\tau})$;
	\FOR{$t=\tau+1,\ldots,T$}
        \IF{$S_t = 1$}
        %\STATE Find $\tau = \max\{s\in \gH_t: S_{s} = 1\} \vee 0$;
        \STATE $\beta_{\tau} \gets \sup\{\beta \in [0,1]: Y_{\tau} \in \gI_{\tau}^{\rm{CAP}}(X_{\tau}; \beta)\}$;
        \FOR{$i=1,\ldots,k$}
        \STATE Call Algorithm \ref{alg:main} and return $\gI_{\tau}^{\rm{CAP}}(X_{\tau};\alpha_{\tau}^i)$;
        \STATE $\operatorname{err}_{\tau}^i\gets \Indicator{Y_{\tau} \not\in \gI_{\tau}^{\rm{CAP}}(X_{\tau}; \alpha^i_{\tau})}$;
        \STATE  $\alpha_{t}^i \gets \alpha_{\tau}^i + \gamma_i (\alpha - \operatorname{err}_{\tau}^i)$;
        \STATE $\bar{w}_{\tau}^i \gets w_{\tau}^i\exp\LRl{-\eta_{\tau} \ell(\beta_{\tau},\alpha_{\tau}^i)}$;
        \ENDFOR
        \STATE $w_{t}^i \gets (1-\phi_{\tau})\bar{w}_{\tau}^i + \phi_{\tau} \sum_{j=1}^k \bar{w}_{\tau}^j /k$ for $1\leq i \leq k$;
        
        \STATE Define $p_{t}^i = w_{t}^i/\sum_{j=1}^k w_{t}^j$ for $1\leq i \leq k$;
        \STATE Assign $\alpha_{t} = \alpha_{t}^i$ with probability $p_{t}^i$;
        \STATE Call Algorithm \ref{alg:main} and return $\gI_t^{\rm{CAP}}(X_t;\alpha_t)$;
        \STATE Set $\tau \gets t$;
        \ENDIF
        \ENDFOR
        \ENSURE Selected PIs: $\{\gI_t^{\rm{CAP}}(X_t;\alpha_t): S_t = 1, 0\leq t \leq T\}$.
	\end{algorithmic}
\end{algorithm}

The original motivation of ACI and DtACI is to achieve approximate marginal coverage by reactively correcting all past mistakes. For the selective inference problem, we aim to control the conditional miscoverage probability through historical feedback. In this vein, we may replace the fixed confidence level $\alpha$ in Algorithm \ref{alg:main} with an adapted value $\alpha_t$ by conditionally correcting past mistakes whenever the selection happens. If $S_t = 1$, we firstly find the most recent selection time $\tau=\max\{0\leq s\leq t-1: S_s=1\}$. %$\tau_t \leq t-1$ such that $S_{\tau_t} = 1$. 
Define a new random variable $\beta_{\tau}^{\rm{CAP}} = \sup\{\beta \in [0,1]: Y_{\tau} \in \gI_{\tau}^{\rm{CAP}}(X_t;\beta)\}$. Parallel with \eqref{eq:ACI_gd}, we update the current confidence level through one step of gradient descent on $\ell(\alpha_{\tau};\beta_{\tau}^{\rm{CAP}})$, i.e.,
\begin{align*}
    \alpha_t = \alpha_{\tau} - \gamma \nabla\ell(\alpha_{\tau};\beta_{\tau}^{\rm{CAP}}).
\end{align*}
Deploying the exponential reweighting scheme, we can also get a selective DtACI algorithm and summarize it in Algorithm \ref{alg:cond_DtACI}. To ensure that Algorithm \ref{alg:cond_DtACI} can be started, we call Algorithm \ref{alg:main} whether $S_0 = 1$ or not in Line 2.
By slightly modifying Theorem 3.2 in \citet{gibbs2022conformal}, we can obtain the following control result on FCR.

\begin{theorem}\label{thm:cond_DtACI}
    Let $\gamma_{\min} = \min_{i}\gamma_i$, $\gamma_{\max} = \max_i \gamma_i$ and $\varrho_t=\frac{(1+2\gamma_{\max})^2}{\gamma_{\min}} \eta_t e^{\eta_t (1+2\gamma_{\max})} + \frac{2(1+\gamma_{\max})}{\gamma_{\min}} \phi_t$. Suppose $\sum_{j=0}^T S_j > 0$ almost surely. Under arbitrary distribution shift on the data $\{(X_i,Y_i)\}_{i= -n}^T$, Algorithm \ref{alg:cond_DtACI} satisfies that
    \begin{align}
        \LRabs{\FCR(T) - \alpha} &\leq \frac{1+2\gamma_{\max}}{ \gamma_{\min}}\E\left[\frac{1}{\sum_{j=0}^T S_j}\right] + \E\left[{\frac{\sum_{t=0}^T S_t \varrho_t}{\sum_{j=0}^T S_j}}\right],\nonumber
    \end{align}
    where the expectation is taken over the randomness from $\{(X_i,Y_i)\}_{i=-n}^T$ and Algorithm \ref{alg:cond_DtACI}.
\end{theorem}

In Theorem \ref{thm:cond_DtACI}, if $\lim_{t\to\infty}\eta_t = \lim_{t\to \infty}\phi_t = 0$ and $\lim_{T\to \infty}\sum_{j=0}^T S_j = \infty$, we can guarantee $\lim_{T\to \infty}\FCR(T) = \alpha$.
While \citet{gibbs2022conformal} advocated using constant or slowly changing values for $\eta_t$ to achieve approximate marginal coverage, it is more appropriate to use the decaying $\eta_t$ in our setting as our goal is to control FCR. Despite the finite-sample guarantee no longer holding for Algorithm \ref{alg:cond_DtACI}, Theorem \ref{thm:cond_DtACI} does not require any conditions on the prediction model $\widehat{\mu}$ or the online selection rule $\Pi_t$. It implies that Algorithm \ref{alg:cond_DtACI} is flexible in practical use. To be specific, we can update the learning model $\widehat{\mu}$ after observing newly labeled data to address the distribution shift. {Moreover, Algorithm 2, when modified by replacing Lines 7 and 15 with ordinary conformal prediction, also exhibits the long-term coverage property as established in Theorem \ref{thm:cond_DtACI}. However, in practice, CAP-DtACI demonstrates superior performance because the PIs constructed using the picked calibration set adapt more effectively to the selective scenario. Detailed discussions and illustrative experiments are provided in Appendix \ref{subsec:discuss_dtaci}.}

%Since the randomness from Algorithm \ref{alg:cond_DtACI} is independent of the data $\{(X_i,Y_i)\}_{i=-n}^T$, using Jensen's inequality and then taking expectation on both sides of \eqref{eq:FCP_cond_DtACI} leads to an upper bound for $\FCR(T)$. 

%For decision-driven selection, the intersecting construction in \eqref{eq:Ct_hat_decision_dyn_cal} becomes $\widehat{\gC}_t^{\rm{inter}} = \{\}$

\section{Synthetic experiments}\label{sec:experiments}

The validity and efficiency of our proposed method will be examined via extensive numerical studies. We focus on using a full holdout set, and the results for the fixed holdout set are provided in Figure \ref{fig:fixed} of Appendix. To mitigate computational costs, we adopt a windowed scheme that utilizes only the most recent 200 data points as the holdout set. Importantly, the theoretical guarantee remains intact; see Appendix \ref{appen:moving_window} for more details. % and additional details regarding the windowed scheme can be found in the Appendix. 
Unless stated otherwise, this windowed scheme is used for all the numerical experiments. 

The evaluation metrics in our experiments are empirical FCR and the average length of constructed PIs across $500$ replications. In each replication, we calculate the current FCP and the average length of all constructed intervals up to the current time $T$ and then derive the real-time FCR level and average length by averaging these values across replications. %{\color{red}Here we address that measuring selection-conditional coverage poses challenges in our online setting due to the varying selection conditions at each time step. Consequently, our evaluation primarily focuses on the FCR metric, which also serves as an empirical proxy for selection-conditional coverage.}

%\begin{itemize}
%    \item \textbf{SwapCP}: Online conformal prediction using the selected calibration set via swapping, i.e. our proposed rule for selected calibration set
%    \begin{align*}
    %\widehat{\gC}_t = \LRl{1\leq s\leq n: V_{-s} > g\LRs{V_t, \{V_{-i}\}_{i\in [n]\setminus \{s\}}}}.
%\end{align*}
%    \item \textbf{SCOP}: The conformal prediction which selects the calibration set by the same threshold for the online data, i.e.
%    \begin{align*}
%    \widehat{\gC}_t = \LRl{1\leq s\leq n: V_{-s} > g\LRs{\{V_{-i}\}_{i\in [n]}}}.
%\end{align*}
%\end{itemize}
\subsection{Results for i.i.d. settings}
%The i.i.d. settings are considered first. 
We first generate i.i.d. 10-dimensional features $X_i$ from uniform distribution ${\rm Unif}([-2,2]^{10})$ and explore three distinct models for the responses $Y_i=\mu(X_i)+\epsilon_i$ with different configurations of $\mu(\cdot)$ and distributions of $\epsilon_i$'s.

\begin{itemize}
    \item \textbf{Scenario A} (Linear model with heterogeneous noise): Let $\mu(X)=X^\top\beta$ where $\beta=(\mathbf{1}_5^{\top},-\mathbf{1}_5^{\top})^\top$ and $\mathbf{1}_5$ is a $5$-dimensional vector with all elements $1$. The noise is heterogeneous and follows the conditional distribution $\epsilon\mid X\sim N(0,\{1+|\mu(X)|\}^2)$. We employ ordinary least squares (OLS) to obtain $\widehat{\mu}(\cdot)$.
    \item \textbf{Scenario B} (Nonlinear model): Let $\mu(X)=X^{(1)}+2X^{(2)}+3(X^{(3)})^2$, where $X^{(k)}$ denotes the $k$-th element of vector $X$ and $\epsilon\sim N(0,1)$ is independent of $X$. The support vector machine (SVM) is applied to train $\widehat{\mu}(\cdot)$.
    \item \textbf{Scenario C} (Aggregation model): Let $\mu(X)=4(X^{(1)}+1)|X^{(3)}|\indicator{{X^{(2)}>-0.4}}+4(X^{(1)}-1)\indicator{X^{(2)}\leq-0.4}.$ The noise follows $\epsilon\sim N(0,1+|X^{(4)}|)$. We use random forest (RF) to obtain $\widehat{\mu}(\cdot)$. 
\end{itemize}
Under each scenario, we utilize an independent labeled set with a size of 200 to train the model $\widehat{\mu}(\cdot)$. We set the initial holdout data size as $n=50$ in the simulations. The results reported in Figure \ref{fig:n_change} show that CAP is not affected too much when the initial size is greater than 10.%For Scenario A, we employ ordinary least squares (OLS) as the basic prediction model. In Scenario B, we utilize support vector machines (SVM), and for Scenario C, we employ random forest (RF). 

To evaluate the performance of our proposed CAP, we conduct comprehensive comparisons with two benchmark methods. The first one is the Online Ordinary Conformal Prediction (OCP), which constructs the PI based on the whole holdout set and ignores the selection effects. The second one is the LORD-CI with default parameters as suggested in \citet{weinstein2020online}. In addition, we have also considered the e-LOND-CI method proposed by \citet{xu2023online}. However, our empirical studies show that it exhibits excessively conservative FCR and yields significantly wider interval lengths compared to other benchmarks. Therefore, we only included the results of this approach in Appendix \ref{appen:e-lord-ci}.

Several selection rules are considered. The first is selection with a fixed threshold. 
\begin{itemize}
%\item \textbf{Mean}: Mean value of scores in calibration set, i.e. $g(\{V_{-i}\}_{1\leq i\leq n})=\sum_{i=1}^n V_{-i}/n$. 
\item[1)] \textbf{Fixed}: A selection rule $\Pi$ with a fixed threshold is posed on the first component of the feature, i.e., $S_t=\Pi(X_t)=\mathds{1}\{X^{(1)}_t>1\}$. Here, we can use $\Pi$ to pick calibration set $\{(X_i,Y_i)\}_{i\in \widehat{\gC}_t}$ where $\widehat{\gC}_t = \{s\in \gH_t: \Pi(X_s) = 1\}$.
\end{itemize}
The next two rules are decision-driven selection in Section \ref{sec:decision_selection}.Here, we consider the nonadaptive pick rule $\Pi_t$ to pick calibration points as \eqref{eq:naive_cal_set}.
%Here, we consider the adaptive pick rules in \eqref{eq:Ada_rule_inter} to pick calibration points.

\begin{itemize}
\item[2)] \textbf{Dec-driven}: At each time $t$, the selection rule is $S_t=\mathds{1}\{\widehat{\mu}(X_t)>\tau(\sum_{i=0}^{t-1} S_i)\}$ where
$\tau(s)=\tau_0-\min\{s/50,2\}$ and $\tau_0$ is fixed for each scenario. 

\item[3)] \textbf{Mul-testing}: Selection with the online multiple testing procedure such as SAFFRON \citep{ramdas2018saffron} with defaulted parameters. We consider the hypotheses as $H_{0t}:Y_t\leq \tau_0-1$ and set the target FDR level as $\beta=20\%$. Additional independent labeled data $\gD_{\rm Add}$ of size $500$ is generalized to construct $p$-values. The detailed procedure is shown in the Appendix \ref{appen:conformal_p}.

\end{itemize}
The following two selection rules are $S_t = \Indicator{\widehat{\mu}(X_t)>\mathcal{A}(\{\widehat{\mu}(X_i)\}_{i=t-200}^{t-1})}$, which are symmetric to the holdout set. Here, we adopt the adaptive pick rules defined in \eqref{eq:Ada_rule_swap} to pick calibration points.
\begin{itemize}
\item[4)] \textbf{Quantile}: $\mathcal{A}(\{\widehat{\mu}(X_i)\}_{i=t-200}^{t-1})$ is the $70\%$-quantile of the $\{\widehat{\mu}(X_i)\}_{i=t-200}^{t-1} $.

\item[5)] \textbf{Mean}: $\mathcal{A}(\{\widehat{\mu}(X_i)\}_{i=t-200}^{t-1})=\sum_{i=t-200}^{t-1} \widehat{\mu}(X_i)/200$.
\end{itemize}

\begin{figure}[htbp]
    \centering
    \includegraphics[width=\textwidth]{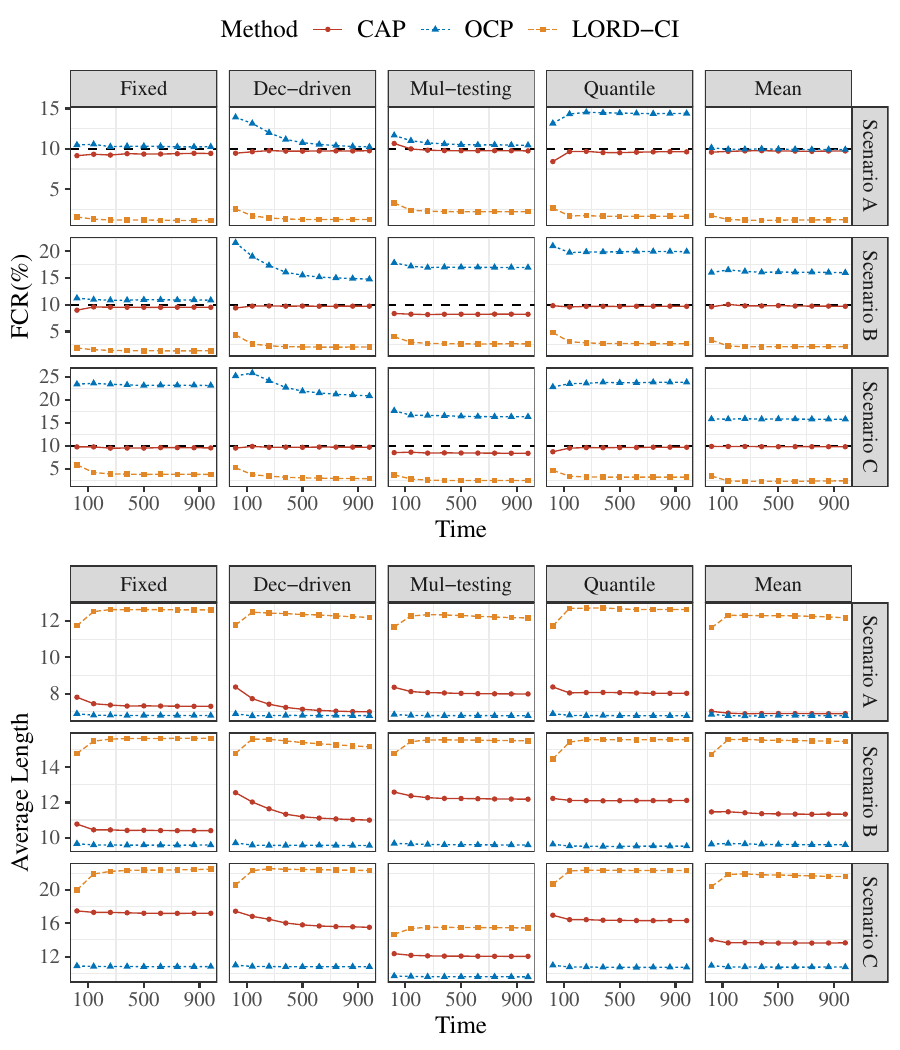}
    \caption{\small\it Real-time FCR and average length from time $20$ to $1,000$ for different scenarios and selection rules. The black dashed line denotes the target FCR level $10\%$.}
    \label{fig:dynamic}
\end{figure}

Figure \ref{fig:dynamic} displays the performance of all benchmarks for the full holdout set across different scenarios and selection rules. All plots indicate that the proposed CAP outperforms the other two methods uniformly in terms of real-time FCR control. This is consistent with the theoretical guarantees of CAP in FCR control. Across all settings, our method achieves stringent FCR control with narrowed PIs. As expected, the OCP yields the shortest PI lengths but much inflated FCR levels under all scenarios. This can be understood since OCP applies all data in the holdout set to build the marginal PIs without consideration of selection effects. The LORD-CI results in considerately conservative FCR levels and accordingly it offers much wider PIs than other methods. Those unsatisfactory PIs are not surprising since the LORD-CI updates the marginal level $\alpha_t$ which may become small as $t$ grows, as discussed in Proposition \ref{pro:LORD_CI}.

%These results affirm the effectiveness of CAP in controlling the real-time FCR. Additionally, the empirical performance of CAP demonstrates anti-conservatism, achieving stringent FCR control with reduced interval lengths. In comparison, it is observed that OCP fails to provide valid FCR control in the majority of cases, and LORD-CI yields notably wide intervals. Consequently, neither method is able to offer satisfactory prediction intervals.

%\textcolor{blue}{[Figure \ref{fig:n_change}  and the corresponding paragraph may be put in the Appendix to save space.]}

%% Maybe only CAP under three selection rules/scenarios is more reasonable

\subsection{Evaluation under distribution shift}\label{subsec:simu_shift}

We further consider four different settings to evaluate the performance of CAP-DtACI under distribution shifts. The first one is the i.i.d. setting which is the same as Scenario B. The second one is a slowly shifting setting where the training and initial labeled data follow the same distribution as that of Scenario B, while the online data gradually drifts over time according to $Y_t=(1-t/500)X_t^{(1)}+(2+\sin{\pi t/200})X_t^{(2)}+(3-t/500)(X_t^{(3)})^2+\varepsilon_t$, where $X_t\sim {\rm Unif}([-2,2])^{10}$ and $\varepsilon\sim N(0,1)$. The third is based on a change point model that generates the same data as in Scenario B when $t\leq 200$, but follows a different pattern when $t>200$, i.e., $Y_t=-2X_t^{(1)}-X_t^{(2)}+3(X_t^{(3)})^2+\varepsilon_t$ . The last shift setting is a time series model, where $Y_t=\{2\sin{\pi X_t^{(1)}X_t^{(2)}}+10(X_t^{(3)})^2+5X_t^{(4)}+2X_t^{(5)}+\xi_t\}/4$ and $\xi_t$ is generated from an ARMA$(0,1)$ process, specifically $\xi_{t+1}=0.99\xi_{t}+\varepsilon_{t+1}+0.99\varepsilon_{t}$. 

We conducted a comparative analysis of the proposed CAP-DtACI in Algorithm \ref{alg:cond_DtACI} with the CAP in Algorithm \ref{alg:main} and the original DtACI. We fix the target FCR level as $\alpha=10\%$. %The selective DtACI employs the same updating rule for $\alpha_t$ as outlined in Algorithm \ref{alg:cond_DtACI}, but utilizes the entire increment holdout set as the calibration set, as opposed to CAP-DtACI using a selected calibration set. 
 To implement DtACI, we fix a candidate number of $k=6$, the starting points $\alpha_0^i=\alpha$ for $i=1,\cdots,6$ and determine other parameters following the suggestions in \citet{gibbs2022conformal}. Typically, we consider the candidate step-sizes $\{\gamma_i\}_{i=1}^6=\{ 0.008, 0.0160, 0.032, 0.064, 0.128,0.256\}$ and let $\phi_t=\phi_0=1/(2I)$, $\eta_t=\eta_0=\sqrt{\{3\log(kI)+6\}/\{I(1-\alpha)^2\alpha^3+I\alpha^2(1-\alpha)^2\}}$ with $I=200$. For the proposed CAP-DtACI, we employ the same parameters except considering decaying learning parameters 
 $\phi_t=\phi_0(\sum_{i=0}^tS_t)^{-0.501}$ and $\eta_t=\eta_0(\sum_{i=0}^tS_t)^{-0.501}$.

\begin{figure}[ht!]
    \centering
    \includegraphics[width=\textwidth]{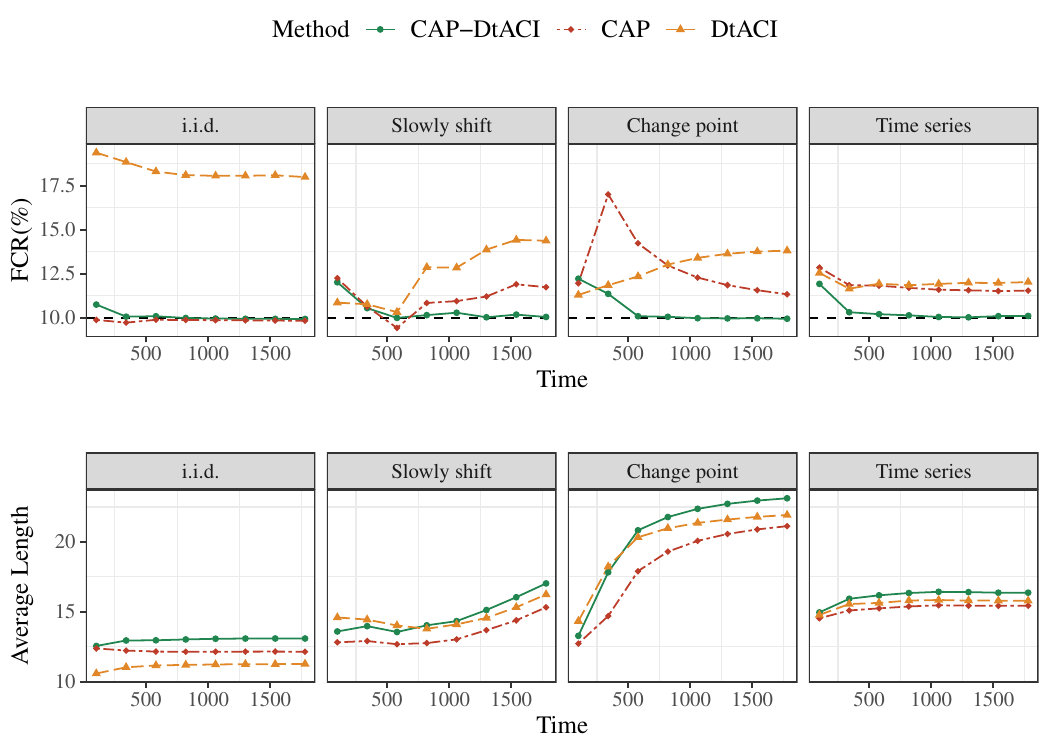}
    \caption{\small\it Comparison for CAP-DtACI, CAP and DtACI by real-time FCR and average length from time $100$ to $2,000$ for quantile selection rule under different data-generating settings. The black dashed line represents the target FCR level $10\%$.}
    \label{fig:comp_DtACI}
\end{figure}

For simplicity, we focus solely on the {\bf Quantile} selection rule as previously described and leave other model settings, including the initial data size, training data size, and prediction algorithm, consistent with those in Scenario B. The results are illustrated in Figure \ref{fig:comp_DtACI}. It is evident that the original DtACI consistently tends to yield an inflated FCR with respect to the target level across all four settings, as it does not account for selection effects. CAP method can only control the FCR under the i.i.d setting, but due to the violation of exchangeability, CAP does not work well in terms of FCR control when distribution shifts exist. In contrast, CAP-DtACI achieves reliable FCR control across various settings by updating an adapted value $\alpha_t$. %Furthermore, our proposed CAP-DtACI outperforms DtACI-sel in terms of efficiency, yielding shorter prediction intervals.

\section{Real data applications}\label{sec:real_data}

\subsection{Drug discovery}
{In drug discovery, researchers examine the binding affinity of drug-target pairs on a case-by-case basis to pinpoint potential drugs with high affinity \citep{huang2022artificial}. With the aid of machine learning tools, we can forecast the affinity for each drug-target pair. If the predicted affinity is high, we can select this pair for further clinical trials. To further quantify the uncertainty by predictions, our method can be employed to construct PIs with a controlled error rate.}
The DAVIS dataset \citep{davis2011comprehensive} consists of $25,772$ drug-target pairs, each accompanied by the binding affinity, structural information of the drug compound, and the amino acid sequence of the target protein. Using the Python library DeepPurpose \citep{huang2020deeppurpose}, we encode the drugs and targets into numerical features and consider the log-scale affinities as response variables. We randomly sample $15,000$ observations from the dataset as the training set to fit a small neural network model with $3$ hidden layers and 5 epochs. Additionally, we set another $2,000$ observations as the online test set, and reserve $50$ data points as the initial labeled data. %is also reserved.

Our objective is to develop real-time prediction intervals for the affinities of selected drug-target pairs. We explore four distinct selection rules in this pursuit, including fixed selection rule $S_t=\Indicator{\widehat{\mu}(X_t)>9}$; decision-driven rule with  $S_t=\Indicator{\widehat{\mu}(X_t)>8+\min\{\sum_{j=0}^{t-1} S_j/400,1\}}$; online multiple testing rule using SAFFRON, which tests $H_{0t}: Y_t\leq 9$ with FDR level at $20\%$ and requires another 1,000 independent labeled samples to construct conformal $p$-values; quantile selection rule, which is $S_t=\Indicator{\widehat{\mu}(X_t)>\mathcal{A}(\{\widehat{\mu}(X_i)\}_{i=t-200}^{t-1})}$, where $\mathcal{A}(\{\widehat{\mu}(X_i)\}_{i=t-200}^{t-1})$ is the $70\%$-quantile of the $\{\widehat{\mu}(X_i)\}_{i=t-200}^{t-1} $. %The quantile selection rule is also used in the experiment for Figure \ref{fig:Drug_plot}. 

\begin{figure}[htb]
    \centering
    \includegraphics[width=\textwidth]{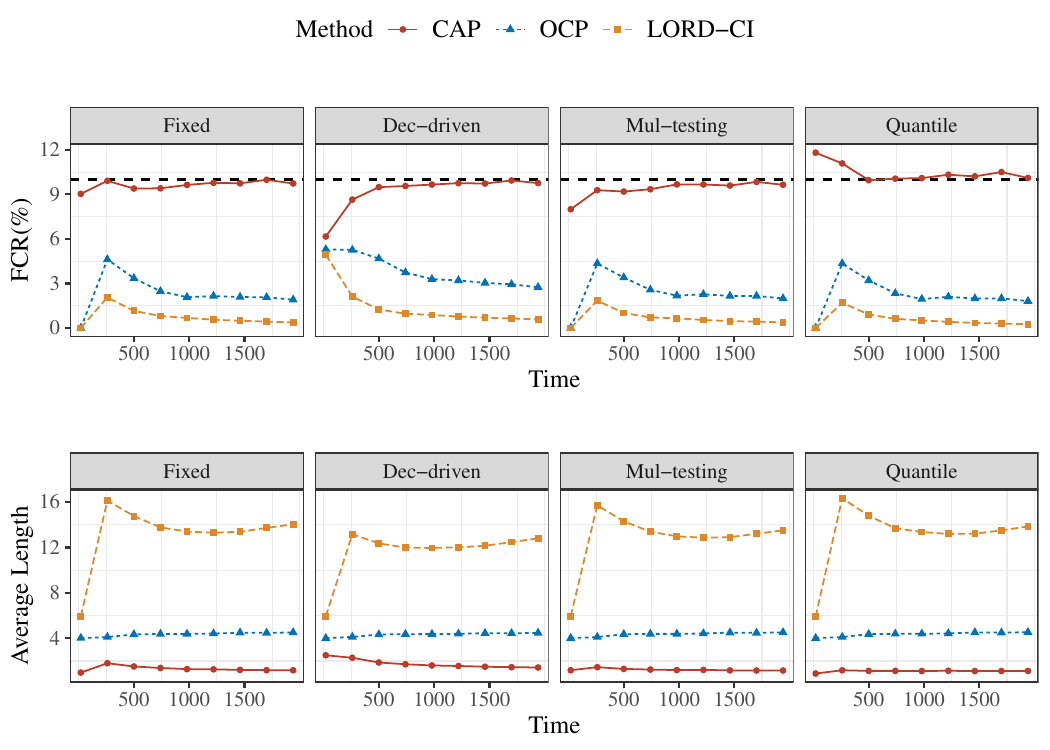}
    \caption{\small\it  Real-time FCR and average length from time $20$ to $2,000$ by $50$ repetitions for drug discovery. The black dashed line denotes the target FCR level $10\%$. }
    \label{fig:drug-FCR}
\end{figure}

Figure \ref{fig:drug-FCR} depicts the real-time FCR and average length of PIs based on the proposed CAP, OCP and LORD-CI across $50$ runs. The results illustrate that the FCR of CAP closely aligns with the nominal level of $10\%$, and CAP can obtain narrowed PIs over time, validating our theoretical findings. In contrast, both OCP and LORD-CI tend to yield conservative FCR values, consequently leading to unsatisfactory PI lengths. Additionally, given that the true log-scale affinities fall within the range of $(-5,10)$, excessively wide intervals would offer limited guidance for further decisions.  By leveraging CAP, researchers can make informed decisions and implement reliable strategies in the pursuit of discovering promising new drugs.

\subsection{Stock volatility}
Stock market volatility exerts a critical role in the global financial market and trading decisions. As an indicator, forecasting future volatility in real-time can provide valuable insights for investors to make informed decisions and account for the potential risk. It is also essential to quantify the uncertainty of predicted volatility. We consider applying our proposed methods to this problem, and the time dependence would have some impact on these methods. In this task, the goal is to use the historical stock prices to predict the volatility the next day. Furthermore,  one is concerned about those days with large volatility. Thereby, we would select those days with large predicted volatility and construct a prediction interval for them.

\begin{figure}[htb]
    \centering
    \includegraphics[width=\textwidth]{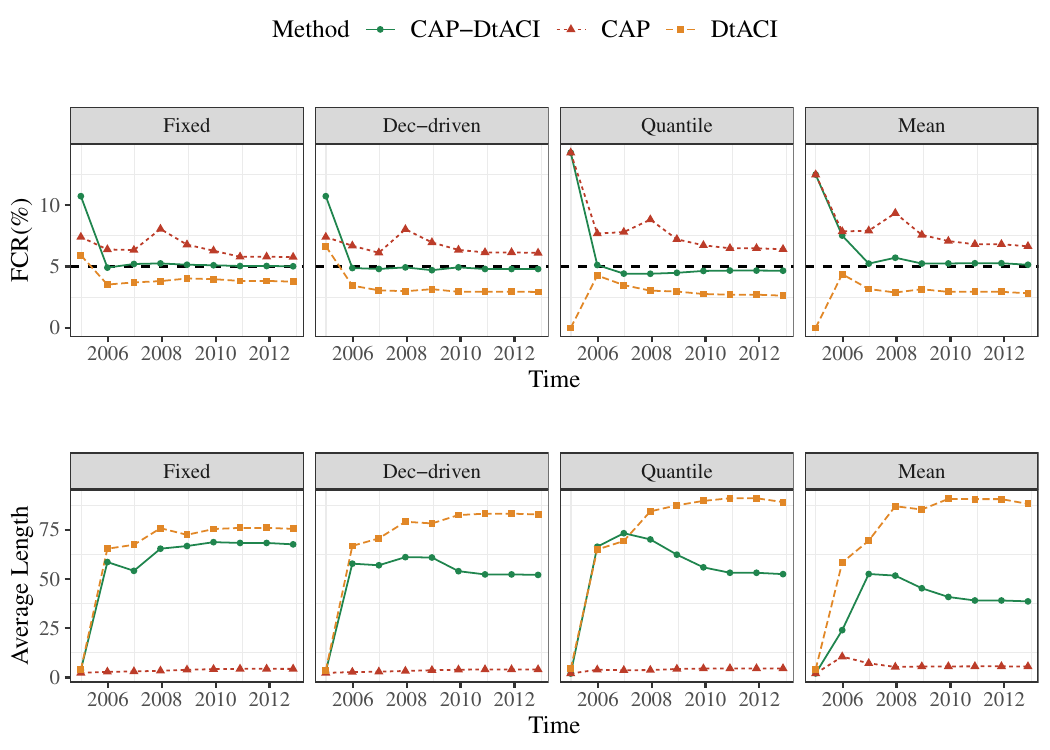}
    \caption{\small\it Real-time FCR and average length from year $2005$ to $2013$ by $20$ replications for four selection rules in stock volatility prediction task. The black dashed line is the target FCR level $5\%$. }
    \label{fig:Stock-FCR}
\end{figure}

We consider the daily price data for NVIDIA from year 1999 to 2021. Denote the price sequence as $\{P_t\}_{t\geq 0}$. We define the return as $X_t:=(P_t-P_{t-1})/P_{t-1}$ and the volatility as $Y_t=X_t^2$. At each time $t$, the predicted volatility $\widehat{Y}_t$ is predicted by a fitted GARCH(1,1) model \citep{bollerslev1986generalized} based on the most recent 1,250 days of
returns $\{X_i\}_{i\in \gH_t}$. And we use a normalized non-conformity score $R_t=|Y_t^2-\widehat{Y}_t^2|/\widehat{Y}^2_t$ instead of the absolute residual as \citet{gibbs2022conformal} suggested. The parameters for implementing CAP-DtACI and DtACI are the same as those in Section \ref{subsec:simu_shift}, except that the parameter $I=1,250$. We set FCR level as $\alpha=5\%$ and use a window size of 1,250.

Four practical selection rules are considered here: fixed selection rule with $S_t=\Indicator{\widehat{Y}_t> 8\times 10^{-4}}$; decision-driven selection rule with $S_t=\Indicator{\widehat{Y}_t>8\times 10^{-4}+\min(\sum_{j=0}^{t-1} S_j/50,4)\times 10^{-4}}$; quantile selection rule with $S_t=\Indicator{\widehat{Y}_t> \mathcal{A}(\{\widehat{Y}_i\}_{i=t-1250}^{t-1})}$ where $\mathcal{A}(\{\widehat{Y}_i\}_{i=t-1250}^{t-1})$ is the 70\%-quantile of $\{\widehat{Y}_i\}_{i=t-1250}^{t-1}$; mean selection rule with $S_t=\Indicator{\widehat{Y}_t>\sum_{i=t-1250}^{t-1}\widehat{Y}_i/1250}$.

Figure \ref{fig:Stock-FCR} shows the FCR and average lengths of CAP-DtACI, CAP and DtACI over 20 replications. The replications are used to ease the randomness generated from DtACI algorithm. As illustrated, CAP-DtACI performs well in delivering FCR close to the target level as time grows. In contrast, CAP has an inflated FCR due to a lack of consideration of distribution shifts and dependent structure of the time series. And the original DtACI delivers much wider PIs as it neglects the selection effects.

\section{Conclusion}\label{sec:conclusion}

This paper addresses the challenge of online selective inference within the framework of conformal prediction. To tackle the non-exchangeability issue introduced by data-driven online selection processes, we introduce CAP, a novel approach that adaptively picks calibration points from historical labeled data to produce reliable PIs for selected observations. Our theoretical analysis and numerical experiments demonstrate the effectiveness of our method in controlling SCC and FCR across various data environments and selection rules.

We point out several future directions. First, while our method targets two common selection rules, further exploration is needed to extend our framework to accommodate arbitrary selection rules. Second, we mainly assume a fixed predictive model for theoretical simplicity. It would be interesting to investigate the feasibility of online updating of machine learning models throughout the process for future study. Third, there may exist a more delicate variant of CAP under some special time series models to obtain tight FCR control.

\bibliography{ref-JMLR-final}

\newpage
\appendix
\include{arxiv-appendix-v4}

\end{document}

%% file: math_commands.tex
\usepackage{amsmath,amsfonts,bm,bbm}

\def\1{\bm{1}}

\def\rz{{\textnormal{z}}}

\def\rmS{{\mathbf{S}}}

\def\ermA{{\textnormal{A}}}

\def\ermV{{\textnormal{V}}}

\DeclareMathAlphabet{\mathsfit}{\encodingdefault}{\sfdefault}{m}{sl}
\SetMathAlphabet{\mathsfit}{bold}{\encodingdefault}{\sfdefault}{bx}{n}

\def\gA{{\mathcal{A}}}

\def\gC{{\mathcal{C}}}
\def\gD{{\mathcal{D}}}
\def\gE{{\mathcal{E}}}
\def\gF{{\mathcal{F}}}

\def\gH{{\mathcal{H}}}
\def\gI{{\mathcal{I}}}
\def\gJ{{\mathcal{J}}}

\def\gN{{\mathcal{N}}}

\def\sP{{\mathbb{P}}}

\def\sR{{\mathbb{R}}}

\newcommand{\E}{\mathbb{E}}

\newcommand{\LRabs}[1]{\left|#1\right|}
\newcommand{\indicator}[1]{\mathbbm{1}\{#1\}}
\newcommand{\Indicator}[1]{\mathbbm{1}\{#1\}}

\newcommand{\Eqmark}[2]{\stackrel{(#1)}{#2}}

\newcommand{\LRs}[1]{\left(#1\right)}
\newcommand{\LRm}[1]{\left[#1\right]}
\newcommand{\LRl}[1]{\left\{#1\right\}}

%% file: arxiv-appendix-v4.tex
%\begin{center}
%    \huge{Supplementary Material for ``CAP: A General Algorithm for Online Selective Conformal Prediction with FCR Control''}
%\end{center}
\allowdisplaybreaks
\numberwithin{equation}{section}
\numberwithin{theorem}{section}

\section{Preliminaries}
In the Appendix, we denote $Z_i = (X_i,Y_i)$ the covariate-label pair for $i\geq -n$. For any index set $\gC$, we write $Q_{\alpha}(\{R_i\}_{i\in \gC})$ as the $\lceil (1-\alpha)|\gC|\rceil$st smallest value in residuals $\{R_i\}_{i\in \gC}$. We also omit the confidence level $\alpha$ in $\gI_t^{\rm{CAP}}(X_t;\alpha)$ whenever the context is clear.

\subsection{Auxiliary lemmas and miscoverage indicator bounds}
%\subsection{Bounds of miscoverage indicator function}
The following two lemmas are usually used in the conformal inference literature \citep{vovk2005algorithmic,lei2018distribution,romano2019conformalized,barber2021predictive,barber2022conformal}.
\begin{lemma}\label{lemma:quantile_inflation}
Let $x_{(\lceil n(1-\alpha)\rceil)}$ is the $\lceil n(1-\alpha)\rceil$smallest value in $\{x_i \in \sR: i \in [n]\}$. Then for any $\alpha \in (0,1)$, it holds that
\begin{align*}
    \frac{1}{n}\sum_{i=1}^n \Indicator{x_i > x_{(\lceil n(1-\alpha)\rceil)}}\leq \alpha.
\end{align*}
If all values in $\{x_i: i \in [n]\}$ are distinct, it also holds that
\begin{align*}
    \frac{1}{n}\sum_{i=1}^n \Indicator{x_i > x_{(\lceil n(1-\alpha)\rceil)}}\geq \alpha - \frac{1}{n},
\end{align*}
\end{lemma}

\begin{lemma}\label{lemma:order_add_one}
Given real numbers $x_1,...,x_n,x_{n+1}$, let $\{x_{(r)}^{[n]}: r\in [n]\}$ be order statistics of $\{x_i: i\in [n]\}$, and $\{x_{(r)}^{[n+1]}: r\in [n+1]\}$ be the order statistics of $\{x_i: i\in [n+1]\}$, then for any $r \in [n]$ we have: $\{x_{n+1} \leq x_{(r)}^{[n]}\} = \{x_{n+1} \leq x_{(r)}^{[n+1]}\}$.
\end{lemma}

According to the definition of $\gI_t^{\rm{CAP}}(X_t)$ in Algorithm \ref{alg:main}, together with Lemma \ref{lemma:order_add_one}, we know
    \begin{align}
        \Indicator{Y_t \not\in \gI_t^{\rm{CAP}}(X_t)} = \Indicator{R_t > q_{\alpha}(\{R_i\}_{i\in \widehat{\gC}_t})} = \Indicator{R_t > Q_{\alpha}(\{R_i\}_{i\in \widehat{\gC}_t \cup \{t\}})}.\nonumber
    \end{align}
    In addition, Lemma \ref{lemma:quantile_inflation} guarantees
    \begin{align}
        \alpha - \frac{1}{|\widehat{\gC}_t|+1} \leq \frac{1}{|\widehat{\gC}_t|+1}\sum_{j \in \widehat{\gC}_t \cup \{t\}}\Indicator{R_j > Q_{\alpha}(\{R_i\}_{i\in \widehat{\gC}_t \cup \{t\}})} \leq \alpha.\nonumber
    \end{align}
    For convenience, we denote $Z_i = (X_i,Y_i), i\geq -n$ and for any index subset $\gC \subseteq \{-n,\ldots,t\}$, we let 
    \begin{align}
        \mathfrak{R}(Z_t,Z_s; \gC) = \mathbbm{1}\{R_t > Q_{\alpha}(\{R_i\}_{i\in \gC})\} - \mathbbm{1}\{R_s > Q_{\alpha}(\{R_i\}_{i\in \gC})\}.\nonumber
    \end{align}
    Combining the two relations above, we have
    \begin{align}\label{eq:miscover_upper}
       \Indicator{Y_t \not\in \gI_t^{\rm{CAP}}(X_t)}\leq \alpha + \frac{1}{|\widehat{\gC}_t|+1}\sum_{s \in \widehat{\gC}_t} \mathfrak{R}(Z_t,Z_s; \widehat{\gC}_t \cup \{t\}),
    \end{align}
    and
    \begin{align}\label{eq:miscover_lower}
        \Indicator{Y_t \not\in \gI_t^{\rm{CAP}}(X_t)}\geq  \alpha - \frac{1}{|\widehat{\gC}_t|+1} + \frac{1}{|\widehat{\gC}_t|+1}\sum_{s \in \widehat{\gC}_t}\mathfrak{R}(Z_t,Z_s; \widehat{\gC}_t \cup \{t\}).
    \end{align}
    Notice that the two bounds above both deterministically hold.

\subsection{Proof of Proposition \ref{pro:selection_exchangeable}}
\begin{proof}
    Notice that, conditioning on the data $\{Z_{\ell}\}_{\ell\neq s,t}$, the selection rules $\Pi_t(\cdot)$ and $\Pi_{t,s}^{\rm Ada}(\cdot)$ depends only on $X_s$ and $X_t$. Let $[Z_s,Z_t]$ be unordered set of $Z_s$ and $Z_t$. Denote $\widehat{\gC}_{t}^{(s)} = \{i\leq t-1,i\neq s: \Pi_{t,i}^{\rm Ada}(X_i) = 1\}$. Clearly, $\widehat{\gC}_{t}^{(s)} \cup \{s\} = \widehat{\gC}_t $ holds if $\Pi_{t,s}^{\rm Ada}(X_s) = 1$. By \eqref{eq:indicator_prod_symmetry}, we know $\Pi_{t,s}^{\rm Ada}(X_s)\Pi_t(X_t)$ is fixed given $[Z_s,Z_t]$ and $\{Z_{\ell}\}_{\ell\neq s,t}$. By \eqref{eq:cal_set_symmetry}, we also know $\widehat{\gC}_{t}^{(s)}$ is fixed given $[Z_s,Z_t]$ and $\{Z_{\ell}\}_{\ell\neq s,t}$ if $s \in \widehat{\gC}_t$.
    Then it follows that
    \begin{align}\label{eq:zero_gap_under_symmetry}
        &\E\LRm{\frac{1}{|\widehat{\gC}_t|+1} \Pi_{t,s}^{\rm Ada}(X_s)\Pi_t(X_t) \cdot\mathfrak{R}(Z_t,Z_s; \widehat{\gC}_t \cup \{t\})}\nonumber\\
        =&\E\LRm{\frac{1}{|\widehat{\gC}_{t}^{(s)}|+2} \Pi_{t,s}^{\rm Ada}(X_s)\Pi_t(X_t)\cdot \mathfrak{R}(Z_t,Z_s; \widehat{\gC}_{t}^{(s)} \cup \{s,t\})}\nonumber\\
        =& \E\LRm{\E\LRm{\frac{1}{|\widehat{\gC}_{t}^{(s)}|+2} \Pi_{t,s}^{\rm Ada}(X_s)\Pi_t(X_t)\cdot \mathfrak{R}(Z_t,Z_s; \widehat{\gC}_{t}^{(s)} \cup \{s,t\}) \mid [Z_s,Z_t], \{Z_{\ell}\}_{\ell \neq s,t}}}\nonumber\\
        =& \E\LRm{\frac{1}{|\widehat{\gC}_{t}^{(s)}|+2} \Pi_{t,s}^{\rm Ada}(X_s)\Pi_t(X_t)\cdot \E\LRm{\mathfrak{R}(Z_t,Z_s; \widehat{\gC}_{t}^{(s)} \cup \{s,t\}) \mid [Z_s,Z_t], \{Z_{\ell}\}_{\ell \neq s,t}}}\nonumber\\
        =& 0,
    \end{align}
    where the last equality follows from the exchangeability between $Z_s$ and $Z_t$, and $Q_{\alpha}(\{R_i\}_{i\in \widehat{\gC}_{t}^{(s)} \cup \{s,t\}})$ is symmetric to $Z_s$ and $Z_t$. Recalling that $S_t = \Pi_t(X_t)$, then we have
    \begin{align}
        \sP&\LRl{Y_t \not\in \gI_t^{\rm CAP}(X_t)\mid S_t = 1} = \frac{1}{\sP(S_t=1)} \E\LRm{S_t\mathbbm{1}\{Y_t \not\in \gI_t^{\rm CAP}\}}\nonumber\\
        &\Eqmark{i}{\leq} \alpha + \frac{1}{\sP(S_t=1)} \E\LRm{\frac{\Pi_t(X_t)}{|\widehat{\gC}_t|+1}\sum_{s \in \widehat{\gC}_t} \mathfrak{R}(Z_t,Z_s; \widehat{\gC}_t \cup \{t\})}\nonumber\\
        &\Eqmark{ii}{=} \alpha + \frac{1}{\sP(S_t=1)}\sum_{s=-n}^{t-1} \E\LRm{\frac{1}{|\widehat{\gC}_t|+1}\Pi_t(X_t)\Pi_{t,s}^{\rm Ada}(X_s) \mathfrak{R}(Z_t,Z_s; \widehat{\gC}_t \cup \{t\})}\nonumber\\
        &\Eqmark{iii}{=} 0,\nonumber
    \end{align}
    where $(i)$ follows from \eqref{eq:miscover_upper}; $(ii)$ holds due to the definition of $\widehat{\gC}_t$; and $(iii)$ holds due to \eqref{eq:zero_gap_under_symmetry}.
\end{proof}

% \begin{definition}
%     Given calibration set $\{X_i\}_{i=1}^n$ and test data $\{X_j\}_{j=n+1}^{n+m}$. A selection rule $\Pi$ is joint-symmetric if $\Pi(\{X_{\sigma(i)}\}_{i=1}^{n+m}) = \Pi(\{X_{i}\}_{i=1}^{n+m})$ holds for any permutation $\sigma: [n+m] \to [n+m]$. 

%     A selection rule $\Pi$ is cal-symmetric if $\Pi(\{X_{\sigma(i)}\}_{i=1}^{n+m}) = \Pi(\{X_{i}\}_{i=1}^{n+m})$ holds for any permutation $\sigma: [n+m] \to [n+m]$ satisfying $\sigma(j) = j$ for any $j = n+1,\ldots,n+m$. 
% \end{definition}

\section{Proofs for decision-driven selection}
\subsection{Proof of Theorem \ref{thm:FCR_decision_select}} \label{proof:thm:FCR_decision_select}
\begin{lemma}\label{lemma:zero_gap_decision_selection}
         Under the conditions of Theorem \ref{thm:FCR_decision_select}, we have
        \begin{align*}
            \E\LRm{\Pi_{t}(X_t)\Pi_{t}(X_s) \mathfrak{R}(Z_t,Z_s; \widehat{\gC}_t \cup \{t\}) \mid \sigma\LRs{\{S_i\}_{i=0}^{t-1}, \{Z_i\}_{i=-n,i\neq s}^{-1}}} = 0.
        \end{align*}
    \end{lemma}
\begin{proof}[Proof of Theorem \ref{thm:FCR_decision_select}]
    Recall that $\widehat{\gC}_t = \{-n\leq s \leq -1: \Pi_t(X_s) = 1\}$.
    Invoking \eqref{eq:miscover_upper}, we can upper bound FCR by
    \begin{align}\label{eq:FCR_expansion_upper}
        \FCR(T) &= \E\LRm{\frac{\sum_{t=0}^T S_t \Indicator{Y_t \not\in \gI_t^{\rm{CAP}}(X_t)}}{1 \vee \sum_{j=0}^T S_j}}\nonumber\\
        %&= \sum_{t=0}^T\E\LRm{\frac{S_t}{1 \vee \sum_{j=0}^T S_j} \Indicator{R_t > Q_{\alpha}\LRs{\frac{1}{|\widehat{\gC}_t|+1}\sum_{i\in \widehat{\gC}_t \cup \{t\}}\delta_{R_i}}}}\nonumber\\
        &\leq \sum_{t=0}^T\E\LRm{\frac{S_t}{1 \vee \sum_{j=0}^T S_j} \LRl{\alpha + \frac{1}{|\widehat{\gC}_t|+1}\sum_{s \in \widehat{\gC}_t} \mathfrak{R}(Z_t,Z_s; \widehat{\gC}_t \cup \{t\})}}\nonumber\\
        &\leq \alpha + \sum_{t=0}^T\E\LRm{\frac{S_t}{1 \vee \sum_{j=0}^T S_j} \frac{1}{|\widehat{\gC}_t|+1} \sum_{s \in \widehat{\gC}_t} \mathfrak{R}(Z_t,Z_s; \widehat{\gC}_t \cup \{t\})}.
    \end{align}
    Similarly, using \eqref{eq:miscover_lower}, we have the following lower bound
    \begin{align}\label{eq:FCR_expansion_lower}
        \FCR(T) \geq \alpha &- \sum_{t=0}^T\E\LRm{\frac{S_t}{1 \vee \sum_{j=0}^T S_j} \frac{1}{|\widehat{\gC}_t|+1}}\nonumber\\
        &+\sum_{t=0}^T\E\LRm{\frac{S_t}{1 \vee \sum_{j=0}^T S_j} \frac{\sum_{s \in \widehat{\gC}_t} \mathfrak{R}(Z_t,Z_s; \widehat{\gC}_t \cup \{t\})}{|\widehat{\gC}_t|+1} }.
    \end{align}
    %Note that the selection rule $\Pi_{t-1}(\cdot)$ is independent of $\{Z_i\}_{i=-1}^{-n}$ and sample $Z_t$.
    Let $\Pi_{j}^{(t)}(\cdot)$ be corresponding selection rule by replacing $X_t$ with $x_{t}^* \in \sigma(\{S_i\}_{i=0}^{t-1})$ such that $\Pi_{t}(x_{t}^*) = 1$. Correspondingly, we denote $S_j^{(t)} = \Pi_j^{(t)}(X_j)$ for any $j\geq 0$. According to our assumption $\Pi_t(\cdot) \in \sigma(\{S_i\}_{i=0}^{t-1})$, we know: (1) $S_j^{(t)} = S_j$ for any $0\leq j \leq t-1$; (2) if $S_t=1$, it holds that $S_j^{(t)} = \Pi_{j}^{(t)}(X_{j}) = \Pi_{j}(X_{j}) = S_j$ for any $j\geq t$. Since $\Pi_{t}(\cdot)$ is independent of $\{Z_i\}_{i=-1}^{-n}$, we have
    \begin{align}\label{eq:zero_gap_fix_cal}
        &\E\LRm{\frac{S_t}{1 \vee \sum_{j=0}^T S_j} \frac{1}{|\widehat{\gC}_t|+1}\sum_{s \in \widehat{\gC}_t} \mathfrak{R}(Z_t,Z_s; \widehat{\gC}_t \cup \{t\})}\nonumber\\ 
        & = \E\LRm{\frac{S_t}{1\vee \sum_{j=0}^T S_j^{(t)} } \sum_{s \in \widehat{\gC}_t}\frac{\mathfrak{R}(Z_t,Z_s; \widehat{\gC}_t \cup \{t\})}{\sum_{j=-n}^{-1}\Pi_t(X_j) +1}}\nonumber\\
        & = \E\LRm{\frac{1}{1\vee \sum_{j=0}^T S_j^{(t)} } \sum_{s=-n}^{-1}\frac{\Pi_{t}(X_t)\Pi_{t}(X_s) \mathfrak{R}(Z_t,Z_s; \widehat{\gC}_t \cup \{t\})}{\sum_{j=-n,j\neq s}^{-1}\Pi_t(X_j)+2}}\nonumber\\
        & = \E\LRm{\frac{1}{1\vee \sum_{j=0}^T S_j^{(t)} } \sum_{s=-n}^{-1}\frac{\E\LRl{\Pi_{t}(X_t)\Pi_{t}(X_s) \mathfrak{R}(Z_t,Z_s; \widehat{\gC}_t \cup \{t\}) \mid \sigma\LRs{\{S_i\}_{i=0}^{t-1}, \{Z_i\}_{i=t+1}^T, \{Z_i\}_{i=-n,i\neq s}^{-1}}}}{\sum_{j=-n,j\neq s}^{-1}\Pi_t(X_{j})+2} }\nonumber\\
        & = \E\LRm{\frac{1}{1\vee \sum_{j=0}^T S_j^{(t)} } \sum_{s=-n}^{-1}\frac{\E\LRl{\Pi_{t}(X_t)\Pi_{t}(X_s) \mathfrak{R}(Z_t,Z_s; \widehat{\gC}_t \cup \{t\}) \mid \sigma\LRs{\{S_i\}_{i=0}^{t-1}, \{Z_i\}_{i=-n,i\neq s}^{-1}}}}{\sum_{j=-n,j\neq s}^{-1}\Pi_t(X_{j})+2} }\nonumber\\
        &= 0,
    \end{align}
    where the last equality follows from Lemma \ref{lemma:zero_gap_decision_selection}. Plugging \eqref{eq:zero_gap_fix_cal} into \eqref{eq:FCR_expansion_upper} gives the desired upper bound $\FCR(T) \leq \alpha$.
    Let $p_t = \sP\LRl{\Pi_{t}(X_t) = 1 \mid \sigma(\{S_i\}_{i=0}^{t-1})}$. From the i.i.d. assumption, we know $|\widehat{\gC}_t| \sim \operatorname{Binomial}(n, p_t)$ given $\sigma(\{S_i\}_{i=0}^{t-1})$. Then we have
    \begin{align}\label{eq:lower_gap}
        \E\LRm{\frac{S_t}{1 \vee \sum_{j=0}^T S_j} \frac{1}{|\widehat{\gC}_t|+1}} &= \E\LRm{\frac{1}{1+ \sum_{j\neq t}^T S_j} \E\LRl{\frac{S_t}{|\widehat{\gC}_t|+1}  \mid \sigma(\{S_i\}_{i=0}^{t-1}), \{Z_i\}_{i\geq t+1}}}\nonumber\\
        &\Eqmark{i}{=} \E\LRm{\frac{1}{1+ \sum_{j\neq t}^T S_j} \E\LRl{\frac{S_t}{|\widehat{\gC}_t|+1}  \mid \sigma(\{S_i\}_{i=0}^{t-1})}}\nonumber\\
        &\Eqmark{ii}{=} \E\LRm{\frac{1}{1+ \sum_{j\neq t}^T S_j} \E\LRl{S_t  \mid \sigma(\{S_i\}_{i=0}^{t-1})}\cdot \E\LRl{\frac{1}{|\widehat{\gC}_t|+1}  \mid \sigma(\{S_i\}_{i=0}^{t-1})}}\nonumber\\
        &= \E\LRm{\frac{1}{1+ \sum_{j\neq t}^T S_j} \E\LRl{S_t  \mid \sigma(\{S_i\}_{i=0}^{t-1})} \cdot \frac{1 - (1 - p_t)^{n+1}}{(n+1) p_t}}\nonumber\\
        &= \E\LRm{\frac{S_t}{1\vee \sum_{j=0}^T S_j}\frac{1 - (1 - p_t)^{n+1}}{(n+1)p_t}},
    \end{align}
    where $(i)$ holds since $S_t, |\widehat{\gC}_t| \independent \{Z_i\}_{i\geq t+1}$ given $\sigma(\{S_i\}_{i=0}^{t-1})$; and $(ii)$ holds due to i.i.d. assumption.
    Plugging \eqref{eq:zero_gap_fix_cal} and \eqref{eq:lower_gap} into \eqref{eq:FCR_expansion_lower} yields the desired lower bound
    \begin{align*}
        \FCR(T) &\geq \alpha \cdot \E\LRm{\frac{\sum_{t = 0}^T S_t}{1 \vee \sum_{j = 0}^T S_j}} - \sum_{t=0}^T \E\LRm{\frac{S_t}{1\vee\sum_{j=0}^T S_j}\frac{1 - (1 - p_t)^{n+1}}{(n+1)p_t}}\\
        &=\alpha - \sum_{t=0}^T \E\LRm{\frac{S_t}{\sum_{j=0}^T S_j}\frac{1 - (1 - p_t)^{n+1}}{(n+1)p_t}},
    \end{align*}
    where the last inequality follows from the assumption $\sum_{j=0}^T S_j > 0$ with probability 1.
    Therefore, we have finished the proof.
\end{proof}

\subsection{Proof of Lemma \ref{lemma:zero_gap_decision_selection}}
\begin{proof}
    We first notice that $\Pi_t(\cdot)$ is fixed given $\sigma(\{S_i\}_{i=0}^{t-1})$. It also means that $\{\Pi_t(X_{-i})\}_{i=1,i\neq s}^n$ are also fixed given $\sigma(\{S_i\}_{i=0}^{t-1}, \{Z_i\}_{i=-n,i\neq s}^{-1})$. Let $z_1 = (x_1,y_1)$ and $z_2 = (x_2,y_2)$. Now define the event $\gE_z = \LRl{[Z_s, Z_t] = [z_1, z_2]}$, where $[Z_s, Z_t]$ and $[z_1,z_2]$ are two unordered sets. Clearly, we know $Q_{\alpha}(\{R_i\}_{i\in \widehat{\gC}_t \cup \{t\}})$ is fixed given $\gE_z$ and $\sigma(\{S_i\}_{i=0}^{t-1}, \{Z_i\}_{i=-n,i\neq s}^{-1})$.
    Recalling the definition of $\mathfrak{R}(Z_t,Z_s; \widehat{\gC}_t \cup \{t\})$, we can get
    \begin{align*}
        &\E\LRm{\Pi_{t}(X_t)\Pi_{t}(X_s) \mathfrak{R}(Z_t,Z_s; \widehat{\gC}_t \cup \{t\}) \mid \sigma\LRs{\{S_i\}_{i=0}^{t-1}, \{Z_i\}_{i=-n,i\neq s}^{-1}}, \gE_z}\nonumber\\
        &\qquad = \sP\LRl{R_t > Q_{\alpha}(\{R_i\}_{i\in \widehat{\gC}_t \cup \{t\}}), \Pi_t(X_t) = 1, \Pi_t(X_s) = 1 \mid \sigma\LRs{\{S_i\}_{i=0}^{t-1}, \{Z_i\}_{i=-n,i\neq s}^{-1}}, \gE_z}\nonumber\\
        &\qquad - \sP\LRl{R_s > Q_{\alpha}(\{R_i\}_{i\in \widehat{\gC}_t \cup \{t\}}), \Pi_t(X_t) = 1, \Pi_t(X_s) = 1 \mid \sigma\LRs{\{S_i\}_{i=0}^{t-1}, \{Z_i\}_{i=-n,i\neq s}^{-1}}, \gE_z}\nonumber\\
        &\qquad \Eqmark{*}{=} \frac{1}{2}\Indicator{\Pi_t(x_1) = 1, \Pi_t(x_2)=1}\LRm{\Indicator{r_1 > Q_{\alpha}(\{R_i\}_{i\in \widehat{\gC}_t \cup \{t\}})} + \Indicator{r_2 > Q_{\alpha}(\{R_i\}_{i\in \widehat{\gC}_t \cup \{t\}})}}\nonumber\\
        &\qquad - \frac{1}{2}\Indicator{\Pi_t(x_1) = 1, \Pi_t(x_2)=1}\LRm{\Indicator{r_1 > Q_{\alpha}(\{R_i\}_{i\in \widehat{\gC}_t \cup \{t\}})} + \Indicator{r_2 > Q_{\alpha}(\{R_i\}_{i\in \widehat{\gC}_t \cup \{t\}})}}\nonumber\\
        &\qquad = 0,
    \end{align*}
    where $r_1 = |y_1 - \widehat{\mu}(x_1)|$ and $r_2 = |y_2 - \widehat{\mu}(x_2)|$; the equality $(*)$ holds since $(R_s,R_t)$ are exchangeable and $\gE_z \independent (\{S_i\}_{i=0}^{t-1}, \{Z_i\}_{i=-n,i\neq s}^{-1})$. Through marginalizing over $\gE_z$, we can prove the desired result.
\end{proof}

\subsection{Proof of Proposition \ref{pro:LORD_CI}}\label{proof:pro:LORD_CI}
\begin{proof}
    Recall that $\gI_t^{\rm{marg}}(X_t;\alpha_t) = \widehat{\mu}(X_t) \pm q_{\alpha_t}\LRs{\{R_i\}_{i\in \gH_0}}$ with $\gH_0 = \{-n,\ldots,-1\}$. It follows that
    \begin{align}\label{eq:miscover_upper_marg}
        \Indicator{Y_t\not\in \gI_t^{\rm{marg}}(X_t;\alpha_t)}\leq \alpha_t + \frac{1}{n+1}\sum_{s=-n}^{-1}\mathfrak{R}(Z_t,Z_s; \gH_0 \cup \{t\}),
     \end{align}
     where $\mathfrak{R}(Z_t,Z_s; \gH_0 \cup \{t\}) = \Indicator{R_t > Q_{\alpha_t}(\{R_i\}_{i\in \gH_0\cup \{t\}})} - \Indicator{R_s > Q_{\alpha_t}(\{R_i\}_{i\in \gH_0\cup \{t\}})}$.
    We follow the notation $S_j^{(t)}$ in Section \ref{proof:thm:FCR_decision_select}.
    By the definition, we have
    \begin{align}
        &\FCR(T) = \sum_{t=0}^T\E\LRm{\frac{S_t \Indicator{Y_t \not\in \gI_t^{\rm{marg}}(X_t;\alpha_t) }}{1\vee \sum_{j=0}^T S_j}}\nonumber\\
        %&= \sum_{t=0}^T\E\LRm{\E\LRm{\frac{S_t \Indicator{Y_t \not\in \gI_t^{\rm{marg}}(X_t) }}{1\vee \sum_{j=0}^T S_j} \mid \sigma\LRs{\{S_i\}_{i=0}^{t-1}}} }\nonumber\\
        &= \sum_{t=0}^T\E\LRm{\frac{S_t \Indicator{Y_t \not\in \gI_t^{\rm{marg}}(X_t;\alpha_t) }}{1\vee \sum_{j=0}^{T} S_j^{(t)}}}\nonumber\\
        &\leq \sum_{t=0}^T\E\LRm{\frac{\Indicator{Y_t \not\in \gI_t^{\rm{marg}}(X_t;\alpha_t) }}{1\vee \sum_{j=0}^{T} S_j^{(t)}}}\nonumber\\
        %&\Eqmark{i}{\leq} \sum_{t=0}^T\E\LRm{\frac{1}{1\vee (\sum_{j=0}^{t-1} S_j + \sum_{j=t}^T S_j^{(t)})} \LRs{\alpha_t + \sum_{s=1}^n \mathfrak{R}(Z_t,Z_s; \gH_0 \cup \{t\})}}\nonumber\\
        &\Eqmark{i}{\leq} \sum_{t=0}^T \E\LRm{\frac{\alpha_t}{1\vee \sum_{j=0}^{T} S_j^{(t)}}} + \sum_{t=0}^T\E\LRm{\frac{1}{1\vee \sum_{j=0}^{T} S_j^{(t)}} \frac{\sum_{s=-n}^{-1} \mathfrak{R}(Z_t,Z_s; \gH_0 \cup \{t\})}{n+1}}\nonumber\\
        &\Eqmark{ii}{\leq} \alpha + \sum_{t=0}^T\E\LRm{\E\LRl{\frac{1}{1\vee \sum_{j=0}^{T} S_j^{(t)}} \frac{\sum_{s=-n}^{-1} \mathfrak{R}(Z_t,Z_s; \gH_0 \cup \{t\})}{n+1} \mid \sigma\LRs{\{S_i\}_{i=0}^{t-1}}} }\nonumber\\
        &\Eqmark{iii}{\leq} \alpha + \sum_{t=0}^T\E\LRm{\E\LRl{\frac{1}{1\vee \sum_{j=0}^{T} S_j^{(t)}} \mid \sigma\LRs{\{S_i\}_{i=0}^{t-1}}}\frac{\sum_{s=-n}^{-1}\E\LRm{ \mathfrak{R}(Z_t,Z_s; \gH_0 \cup \{t\}) \mid \sigma\LRs{\{S_i\}_{i=0}^{t-1}}}}{n+1} }\nonumber\\
        &\Eqmark{iv}{=} \alpha,\nonumber
    \end{align}
    where $(i)$ follows from \eqref{eq:miscover_upper_marg}; $(ii)$ holds due to the LORD-CI's invariant $\sum_{t=0}^T \alpha_t/(\sum_{j=0}^T S_j) \leq \alpha$ and $S_j^{(t)} \geq S_j$ for any $j\geq t$; $(iii)$ holds since $\alpha_t \in \sigma\LRs{\{S_i\}_{i=0}^{t-1}}$ and $(Z_t, Z_{s}) \independent S_j^{(t)}$ for $j\geq t$; and $(iv)$ follows from the exchangeability between $Z_t$ and $Z_{s}$ such that
    \begin{align*}
        &\E\LRm{\mathfrak{R}(Z_t,Z_s; \gH_0 \cup \{t\}) \mid \sigma\LRs{\{S_i\}_{i=0}^{t-1}}} = \E\LRm{\mathfrak{R}(Z_t,Z_s; \gH_0 \cup \{t\})}\\
        &= \sP\LRm{R_t > Q_{\alpha}(\{R_i\}_{i\in \gH_0\cup \{t\}})} - \sP\LRm{R_s > Q_{\alpha}(\{R_i\}_{i\in \gH_0\cup \{t\}})} = 0.
    \end{align*}
\end{proof}

\subsection{Proof of Theorem \ref{thm:FCR_decision_dyn_SCOP}}\label{proof:them:FCR_decision_dyn_SCOP}

\begin{proof}
Recall that $\widehat{\gC}_t = \{-n\leq i\leq t-1: \Pi_t(X_i) = 1\}$. For convenience, we let $\Pi_i(\cdot)\equiv 1$ for $-n\leq i \leq -1$. 
Denote
\begin{align*}
    \widehat{\gC}_t^{(s)} = \{-n\leq i\leq t-1,i\neq s: \Pi_t(X_i) = 1\}.
\end{align*}
Clearly, it holds that $\widehat{\gC}_t^{(s)} \cup \{s\} = \widehat{\gC}_t$ if $s\in \widehat{\gC}_t$.
\subsubsection{Proof of selection-conditional coverage}
Notice that
\begin{align}\label{eq:scop_SCC}
    \sP\LRl{Y_t \not\in \gI_t^{\rm CAP}(X_t) \mid S_t = 1} &\leq \alpha + \frac{1}{\sP(S_t=1)}\sum_{s=-n}^{t-1}\E\LRm{\frac{\Pi_t(X_t)\mathbbm{1}\{s\in \widehat{\gC}_t\}}{|\widehat{\gC}_t|+1}  \mathfrak{R}(Z_t,Z_s; \widehat{\gC}_t \cup \{t\})}\nonumber\\
    &= \alpha + \frac{1}{\sP(S_t=1)}\sum_{s=-n}^{t-1}\E\LRm{\frac{\Pi_t(X_t)\Pi_t(X_s)}{|\widehat{\gC}_t^{(s)}|+2}  \mathfrak{R}(Z_t,Z_s; \widehat{\gC}_t^{(s)} \cup \{s,t\})}\nonumber\\
    &= \alpha + \frac{1}{\sP(S_t=1)}\sum_{s=0}^{t-1}\E\LRm{\frac{\Pi_t(X_t)\Pi_t(X_s)}{|\widehat{\gC}_t^{(s)}|+2}  \mathfrak{R}(Z_t,Z_s; \widehat{\gC}_t^{(s)} \cup \{s,t\})}.
\end{align}
In fact, the last equality holds due to for any offline point $-n\leq s \leq -1$,
\begin{align}\label{eq:exchange_gap_scop_fix}
    &\E\LRm{\frac{\Pi_t(X_t) \Pi_t(X_s)}{|\widehat{\gC}_t^{(s)}| + 2} \mathfrak{R}(Z_t,Z_s; \widehat{\gC}_t^{(s)} \cup \{s,t\})}\nonumber\\
    =& \E\LRm{ \E\LRm{\frac{\Pi_t(X_t) \Pi_t(X_s)}{|\widehat{\gC}_t^{(s)}| + 2} \mathfrak{R}(Z_t,Z_s; \widehat{\gC}_t \cup \{t\}) \mid \{Z_{\ell}\}_{\ell \neq s,t}, [Z_s,Z_t]}}\nonumber\\
    \Eqmark{i}{=}& \E\LRm{\frac{\Pi_t(X_t) \Pi_t(X_s)}{|\widehat{\gC}_t^{(s)}| + 2} \E\LRm{ \mathfrak{R}(Z_t,Z_s; \widehat{\gC}_t^{(s)} \cup \{s,t\}) \mid \{Z_{\ell}\}_{\ell \neq s,t}, [Z_s,Z_t]}}
    \Eqmark{ii}{=} 0,
\end{align}
where $S_j^{(t)}$ is defined in Section \ref{proof:pro:LORD_CI}; $(i)$ holds due to $\Pi_t(\cdot)$ is independent of $Z_s$ and $Z_t$, hence $\widehat{\gC}_t^{(s)}$ is symmetric to $(Z_s,Z_t)$; and $(ii)$ holds due to exchangeability between $Z_s$ and $Z_t$.

\textbf{\emph{Decoupling dependence over $X_s,\ 0\leq s \leq t-1$.}}
Let $x_{s,1},x_{s,0} \in \sigma(S_0,\ldots,S_{s-1})$ be the values such that $\Pi_{s}(x_{s,1}) = 1$ and $\Pi_s(x_{s,0})=0$ for $0\leq s \leq t-1$. Denote $\{\widetilde{\Pi}_{i,1}^{(s)}\}_{i\geq 0}$ and $\{\widetilde{\Pi}_{i,0}^{(s)}\}_{i\geq 0}$ the virtual selection rules generated by replacing $X_s$ with $x_{s,1}$ and $x_{s,0}$, respectively. Let $\{S_{i}^{(s,1)}\}_{i\geq 0}$ and $\{S_{i}^{(s,0)}\}_{i\geq 0}$ be the corresponding virtual decision sequences.
Then we have the following conclusions:
\begin{itemize}
        \item[(1)] $\widetilde{\Pi}_{i,1}^{(s)}(\cdot) \equiv \Pi_i(\cdot)$ for any $0 \leq i \leq s$;
        
        \item[(2)] If $S_s = 1$, then $\widetilde{\Pi}_{i,1}^{(s)}(\cdot) = \Pi_i(\cdot)$ for any $i \geq s+1$.
        
        \item[(3)] If $S_s = 0$, then $\widetilde{\Pi}_{i,0}^{(s)} = \Pi_i$ for any $i \geq s+1$.
\end{itemize}
Denote
\begin{align}
    \widetilde{\gC}_{t,1}^{(s)} = \LRl{-n\leq i\leq t-1,i\neq s: \widetilde{\Pi}_{t,1}(X_i)=1},\nonumber\\
    \widetilde{\gC}_{t,0}^{(s)} = \LRl{-n\leq i\leq t-1,i\neq s: \widetilde{\Pi}_{t,0}(X_i)=1}.\nonumber
\end{align}
Then we know $\widetilde{\gC}_{t,1}^{(s)} = \widehat{\gC}_t^{(s)}$ if $\Pi_s(X_s) = 1$, and $\widetilde{\gC}_{t,0}^{(s)} = \widehat{\gC}_t^{(s)}$ if $\Pi_s(X_s) = 0$. In addition, we also know $\widetilde{\gC}_{t,1}^{(s)}$ and $\widetilde{\gC}_{t,1}^{(s)}$ are independent of $(Z_s,Z_t)$.
For any online point $0\leq s \leq t-1$, we have
\begin{align}\label{eq:exchange_gap}
    &\E\LRm{\frac{\Pi_t(X_t)\Pi_t(X_s)}{|\widehat{\gC}_t^{(s)}|+2}  \mathfrak{R}(Z_t,Z_s; \widehat{\gC}_t^{(s)} \cup \{s,t\}) \mid \{Z_{\ell}\}_{\ell \neq s,t}, [Z_s,Z_t]}\nonumber\\ 
    &= \underbrace{\E\LRm{\frac{\Pi_t(X_t) \Pi_t(X_s) \Pi_s(X_t) \Pi_s(X_s)}{|\widehat{\gC}_t^{(s)}|+2} \mathfrak{R}(Z_t,Z_s; \widehat{\gC}_t^{(s)} \cup \{s,t\}) \mid \{Z_{\ell}\}_{\ell \neq s,t}, [Z_s,Z_t]}}_{\mathrm{I}}\nonumber\\
    & + \underbrace{\E\LRm{\frac{\Pi_t(X_t) \Pi_t(X_s) [1-\Pi_s(X_t)] [1-\Pi_s(X_s)]}{|\widehat{\gC}_t^{(s)}|+2} \mathfrak{R}(Z_t,Z_s; \widehat{\gC}_t^{(s)} \cup \{s,t\}) \mid \{Z_{\ell}\}_{\ell \neq s,t}, [Z_s,Z_t]}}_{\mathrm{II}}\nonumber\\
    & + \underbrace{\E\LRm{\frac{\Pi_t(X_t) \Pi_t(X_s)}{|\widehat{\gC}_t^{(s)}|+2} \mathbbm{1}\{\Pi_s(X_t) \neq \Pi_s(X_s)\}\cdot \mathfrak{R}(Z_t,Z_s; \widehat{\gC}_t^{(s)} \cup \{s,t\}) \mid \{Z_{\ell}\}_{\ell \neq s,t}, [Z_s,Z_t]}}_{\mathrm{III}},
    %&\qquad + \underbrace{\E\LRm{\frac{1}{|\widehat{\gC}_t|+1}{\Pi_t(X_t) \Pi_t(X_s) \Pi_s(X_s) [1 - \Pi_s(X_t)]} \mathfrak{R}(Z_t,Z_s; \widehat{\gC}_t \cup \{t\}) }}_{\mathrm{III}}\nonumber\\
    %&\qquad+\underbrace{\E\LRm{ \frac{1}{|\widehat{\gC}_t|+1}{\Pi_t(X_t) \Pi_t(X_s) [1-\Pi_s(X_s)] \Pi_s(X_t)} \mathfrak{R}(Z_t,Z_s; \widehat{\gC}_t \cup \{t\}) }}_{\mathrm{IV}},
\end{align}
where the first equality holds due to \eqref{eq:exchange_gap_scop_fix}. Because $\Pi_s, \widetilde{\Pi}_{t,1}^{(s)}, \widetilde{\gC}_{t, 1}^{(s)}$ are fixed given $\{Z_{\ell}\}_{\ell\neq s,t}$, using the exchangeability between $Z_s$ and $Z_t$, we can verify
\begin{align}\label{eq:exchange_zero_gap_1}
    \mathrm{I}&= \frac{\widetilde{\Pi}_{t,1}^{(s)}(X_t) \widetilde{\Pi}_{t,1}^{(s)}(X_s) \Pi_s(X_s) \Pi_s(X_t)}{|\widetilde{\gC}_{t, 1}^{(s)}|+2}\times\nonumber\\
    &\qquad\qquad\qquad\qquad\E\LRm{\mathfrak{R}(Z_t,Z_s; \widetilde{\gC}_{t, 1}^{(s)}\cup \{s,t\})\mid \{Z_{\ell}\}_{\ell \neq s,t}, [Z_s,Z_t]} = 0.
\end{align}
Similarly, we can show that
\begin{align}\label{eq:exchange_zero_gap_2}
    \mathrm{II} &= \frac{\widetilde{\Pi}_{t,0}^{(s)}(X_t) \widetilde{\Pi}_{t,0}^{(s)}(X_s) [1-\Pi_s(X_s)] [1-\Pi_s(X_t)]}{|\widetilde{\gC}_{t, 0}^{(s)}|+2}\times\nonumber\\
    &\qquad\qquad\qquad\qquad\E\LRm{\mathfrak{R}(Z_t,Z_s; \widetilde{\gC}_{t, 0}^{(s)}\cup \{s,t\})\mid \{Z_{\ell}\}_{\ell \neq s,t}, [Z_s,Z_t]}= 0.
\end{align}
Plugging \eqref{eq:exchange_zero_gap_1} and \eqref{eq:exchange_zero_gap_2} into \eqref{eq:exchange_gap}, together with \eqref{eq:scop_SCC}, we can get
\begin{align}
    \sP&\LRl{Y_t \not\in \gI_t^{\rm CAP}(X_t) \mid S_t = 1}\nonumber\\
    &\leq \alpha + \frac{\sum_{s=0}^{t-1}\E\LRm{\frac{\Pi_t(X_t) \Pi_t(X_s)}{|\widehat{\gC}_t^{(s)}|+2} \mathbbm{1}\{\Pi_s(X_t) \neq \Pi_s(X_s)\}\cdot \mathfrak{R}(Z_t,Z_s; \widehat{\gC}_t^{(s)} \cup \{s,t\})}}{\sP(S_t = 1)}.\nonumber
\end{align}
By the definition of $\Delta_t$ in Theorem \ref{thm:FCR_decision_dyn_SCOP}, we can prove the conclusion.

\subsubsection{Proof of FCR}
By the definition of FCR, we have
\begin{align}\label{eq:scop_FCR}
    \FCR(T) &= \sum_{t=0}^T\E\LRm{\frac{S_t \mathbbm{1}\{Y_t \not\in \gI_t^{\rm CAP}(X_t)\}}{ 1 \vee \sum_{j=0}^T S_j} }\nonumber\\
    &\leq \alpha + \sum_{t=0}^T\sum_{s=-n}^{t-1}\E\LRm{\frac{1}{ 1 \vee \sum_{j=0}^T S_j} \frac{1}{|\widehat{\gC}_t| + 1} \Pi_t(X_t) \Pi_t(X_s) \mathfrak{R}(Z_t,Z_s; \widehat{\gC}_t \cup \{t\})}.
\end{align}
For any $-n\leq s \leq -1$, we can show
\begin{align}\label{eq:scop_FCR_gap_offline}
    &\E\LRm{\frac{1}{ 1 \vee \sum_{j=0}^T S_j} \frac{1}{|\widehat{\gC}_t| + 1} \Pi_t(X_t) \Pi_t(X_s) \mathfrak{R}(Z_t,Z_s; \widehat{\gC}_t \cup \{t\})}\nonumber\\
    =& \E\LRm{\frac{1}{ 1 \vee \sum_{j=0}^T S_j^{(t)}} \E\LRm{\frac{1}{|\widehat{\gC}_t| + 1} \Pi_t(X_t) \Pi_t(X_s) \mathfrak{R}(Z_t,Z_s; \widehat{\gC}_t \cup \{t\}) \mid \{Z_{\ell}\}_{\ell \neq s,t}, [Z_s,Z_t]}} = 0,
\end{align}
where $S_j^{(t)}$ is defined in Section \ref{proof:pro:LORD_CI} and the last equality holds due to \eqref{eq:exchange_gap_scop_fix}.

\textbf{\emph{Decoupling dependence over both $X_s$ and $X_t$.}}
Let $x_{t}^{(s,1)} \in \sigma(S_{0,1}^{(s)},\ldots,S_{t-1,1}^{(s)})$ and $x_{t}^{(s,0)} \in \sigma(S_{0,0}^{(s)},\ldots,S_{t-1,0}^{(s)})$ be the values such that $\widetilde{\Pi}_{t,1}^{(s)}(x_{t}^{(s,1)}) = 1$ and $\widetilde{\Pi}_{t,0}^{(s)}(x_{t}^{(s,0)}) = 1$ for $t > s$, respectively. Let $\{S_{i,1}^{(s,t)}\}_{i\geq 0}$ be the virtual decision sequence generated by firstly replacing $X_s$ with $x_{s,1}$, and then replacing $X_t$ with $x_{t}^{(s,1)}$. Let $\{S_{i,0}^{(s,t)}\}_{i\geq 0}$ be the virtual decision sequence generated by firstly replacing $X_s$ with $x_{s,0}$, and then replacing $X_t$ with $x_{t}^{(s,1)}$. In this case, we can guarantee that $S_{i,1}^{(s,t)}, S_{i,0}^{(s,t)} \independent (Z_s,Z_t)$ for any $i\geq 0$ because $x_{s,1},x_{s,0} \independent (Z_s,Z_t)$ and $x_{t}^{(s,1)},x_{t}^{(s,0)} \independent (Z_s,Z_t)$. We have
\begin{itemize}
    \item[(1)] $S_{i,1}^{(s,t)} \equiv S_{i,0}^{(s,t)} \equiv S_i$ for $i\leq s-1$;
    \item[(2)] $S_{i,1}^{(s,t)} \equiv S_{i}^{(s,1)}$ and $S_{i,0}^{(s,t)} \equiv S_{i}^{(s,0)}$ for $s\leq i \leq t-1$;
    \item[(3)] If $S_t = 1$, $S_{i,1}^{(s,t)} = S_{i}^{(s,1)}$ for $i \geq t$.
    \item[(4)] If $S_t = 0$, $S_{i,0}^{(s,t)} = S_{i}^{(s,0)}$ for $i \geq t$.
\end{itemize}
Then for any $0\leq s\leq t-1$, we have
\begin{align}\label{eq:scop_FCR_gap_online}
    &\E\LRm{\frac{1}{ 1 \vee \sum_{j=0}^T S_j} \frac{1}{|\widehat{\gC}_t| + 1} \Pi_t(X_t) \Pi_t(X_s) \mathfrak{R}(Z_t,Z_s; \widehat{\gC}_t \cup \{t\})}\nonumber\\
    &= \underbrace{\E\LRm{\frac{1}{ 1 \vee \sum_{j=0}^T S_j^{(s,t)}}\frac{\Pi_t(X_t) \Pi_t(X_s) \Pi_s(X_t) \Pi_s(X_s)}{|\widehat{\gC}_t^{(s)}|+2} \mathfrak{R}(Z_t,Z_s; \widehat{\gC}_t^{(s)} \cup \{s,t\}) }}_{\mathrm{I}^{\prime}}\nonumber\\
    & + \underbrace{\E\LRm{\frac{1}{ 1 \vee \sum_{j=0}^T S_j^{(s,t)}}\frac{\Pi_t(X_t) \Pi_t(X_s) [1-\Pi_s(X_t)] [1-\Pi_s(X_s)]}{|\widehat{\gC}_t^{(s)}|+2} \mathfrak{R}(Z_t,Z_s; \widehat{\gC}_t^{(s)} \cup \{s,t\}) }}_{\mathrm{II}^{\prime}}\nonumber\\
    & + \underbrace{\E\LRm{\frac{1}{ 1 \vee \sum_{j=0}^T S_j^{(s,t)}}\frac{\Pi_t(X_t) \Pi_t(X_s)}{|\widehat{\gC}_t^{(s)}|+2} \mathbbm{1}\{\Pi_s(X_t) \neq \Pi_s(X_s)\}\cdot \mathfrak{R}(Z_t,Z_s; \widehat{\gC}_t^{(s)} \cup \{s,t\}) }}_{\mathrm{III}^{\prime}}.
\end{align}
Since $\{S_j^{(s,t)}\}_{j\geq 0}$ are independent of $Z_s,Z_t$, we have
\begin{align}
    \mathrm{I}^{\prime} &= \E\LRm{\frac{1}{ 1 \vee \sum_{j=0}^T S_j^{(s,t)}}\frac{\Pi_t(X_t) \Pi_t(X_s) \Pi_s(X_t) \Pi_s(X_s)}{|\widehat{\gC}_t^{(s)}|+2} \E\LRm{\mathfrak{R}(Z_t,Z_s; \widehat{\gC}_t^{(s)} \cup \{s,t\}) \mid \{Z_{\ell}\}_{\ell \neq s,t}, [Z_s,Z_t]}}\nonumber\\
    &= 0.\nonumber
\end{align}
Similarly, we can also show $\mathrm{II}^{\prime} = 0$. Then plugging \eqref{eq:scop_FCR_gap_offline} and \eqref{eq:scop_FCR_gap_online} into \eqref{eq:scop_FCR} yields the conclusion.

\subsection{Proof of Corollary \ref{cor:scop-1}}
If $\Pi_t(x) \leq \Pi_s(x)$ for any $x$ and $s\leq t-1$, we can guarantee that
\begin{align}
    \Pi_t(X_t)\Pi_t(X_s) \Indicator{\Pi_s(X_t) \neq \Pi_s(X_s)} = 0,\nonumber
\end{align}
since $\Pi_t(X_t)\Pi_t(X_s) = 1$ implies $\Pi_s(X_t) = \Pi_s(X_s) = 1$. Then we can show $\Delta_t = 0$.

\subsection{Proof of Corollary \ref{cor:scop-2}}
Notice the fact
\begin{align}
    \mathbbm{1}\{\Pi_s(X_t) = \Pi_s(X_s)\} = \Pi_s(X_s)[1 - \Pi_s(X_t)] + [1 - \Pi_s(X_s)]\Pi_s(X_t).\nonumber
\end{align}
Then we can decompose $\mathrm{III}$ in \eqref{eq:exchange_gap} as
\begin{align}\label{eq:gap_3_decompose}
    \mathrm{III} &= \underbrace{\E\LRm{\frac{\Pi_t(X_t) \Pi_t(X_s)  \Pi_s(X_s)[1 - \Pi_s(X_t)]}{|\widehat{\gC}_t^{(s)}|+2}  \mathbbm{1}\LRl{R_t > Q_{\alpha}\LRs{\{R_i\}_{i\in \widehat{\gC}_t^{(s)} \cup \{s,t\}}}} \mid \{Z_{\ell}\}_{\ell \neq s,t}, [Z_s,Z_t]}}_{\mathrm{A}_1}\nonumber\\
    &- \underbrace{\E\LRm{\frac{\Pi_t(X_t) \Pi_t(X_s)  \Pi_s(X_t)[1 - \Pi_s(X_s)]}{|\widehat{\gC}_t^{(s)}|+2}  \mathbbm{1}\LRl{R_s > Q_{\alpha}\LRs{\{R_i\}_{i\in \widehat{\gC}_t^{(s)} \cup \{s,t\}}}} \mid \{Z_{\ell}\}_{\ell \neq s,t}, [Z_s,Z_t]}}_{\mathrm{A}_2}\nonumber\\
    &+ \underbrace{\E\LRm{\frac{\Pi_t(X_t) \Pi_t(X_s)  \Pi_t(X_s)[1 - \Pi_s(X_s)]}{|\widehat{\gC}_t^{(s)}|+2}  \mathbbm{1}\LRl{R_t > Q_{\alpha}\LRs{\{R_i\}_{i\in \widehat{\gC}_t^{(s)} \cup \{s,t\}}}} \mid \{Z_{\ell}\}_{\ell \neq s,t}, [Z_s,Z_t]}}_{\mathrm{B}_1}\nonumber\\
    &- \underbrace{\E\LRm{\frac{\Pi_t(X_t) \Pi_t(X_s)  \Pi_s(X_s)[1 - \Pi_s(X_t)]}{|\widehat{\gC}_t^{(s)}|+2}  \mathbbm{1}\LRl{R_s > Q_{\alpha}\LRs{\{R_i\}_{i\in \widehat{\gC}_t^{(s)} \cup \{s,t\}}}} \mid \{Z_{\ell}\}_{\ell \neq s,t}, [Z_s,Z_t]}}_{\mathrm{B}_2}.
\end{align}
Recall the original generation mechanism of decision rules,
\begin{align}
    \begin{array}{ccccccccc}
         X_0 &&\cdots &X_s &{}& X_{s+1} & &\cdots &X_t  \\
         %\downarrow && \cdots &\downarrow && \downarrow &&\cdots &\downarrow\\
         \Pi_0(\cdot)& \stackrel{\Pi_0(X_0)}{\longrightarrow}& \cdots &\Pi_s(\cdot)&\stackrel{\Pi_s(X_s)}{\longrightarrow} &\Pi_{s+1}(\cdot)&\stackrel{\Pi_{s+1}(X_{s+1})}{\longrightarrow} &\cdots & \Pi_t(\cdot)
    \end{array}.\nonumber
\end{align}
Now we swap $X_s$ and $X_t$ in the data sequence, and denote the generated decision rule as
\begin{align}
    \begin{array}{ccccccccc}
         X_0 &&\cdots &X_t &{}& X_{s+1} & &\cdots &X_s  \\
         %\downarrow && \cdots &\downarrow && \downarrow &&\cdots &\downarrow\\
         \Pi_0(\cdot)& \stackrel{\Pi_0(X_0)}{\longrightarrow}& \cdots &\Pi_s(\cdot)&\stackrel{\Pi_s(X_t)}{\longrightarrow} &\Pi_{s+1}^{s\leftrightarrow t}(\cdot)&\stackrel{\Pi_{s+1}^{s\leftrightarrow t}(X_{s+1})}{\longrightarrow} &\cdots & \Pi_{t}^{s\leftrightarrow t}(\cdot)
    \end{array}.\nonumber
\end{align}
The corresponding picked calibration set is
\begin{align}
    \widehat{\gC}^{s \leftrightarrow t} = \LRl{-n\leq i \leq t-1, i\neq s: \Pi_t^{s \leftrightarrow t}(X_i) = 1}.\nonumber
\end{align}
According to our assumption, we know $\Pi_t^{(s\gets t)}(\cdot) = \Pi_t(\cdot)$ for $t \geq t_0+1$.
Hence, we have $ \widehat{\gC}^{s \leftrightarrow t} = \widehat{\gC}_{t}^{(s)}$ and
\begin{align}
    \mathrm{A}_1 &= \E\LRm{\frac{\Pi_t(X_t) \Pi_t(X_s)\Pi_s(X_s)[1-\Pi_{s}(X_t)]}{|\widehat{\gC}_t^{(s)}|+2}  \mathbbm{1}\LRl{R_t > Q_{\alpha}\LRs{\{R_i\}_{i\in \widehat{\gC}_t^{(s)} \cup \{s,t\}}}}\mid \{Z_{\ell}\}_{\ell \neq s,t}, [Z_s,Z_t]}\nonumber\\
    &= \E\LRm{\frac{\Pi_t^{s \leftrightarrow t}(X_t) \Pi_t^{s \leftrightarrow t}(X_s)\Pi_s(X_t)[1-\Pi_{s}(X_s)]}{|\widehat{\gC}^{s \leftrightarrow t}|+2} \mathbbm{1}\LRl{R_s > Q_{\alpha}\LRs{\{R_i\}_{i\in \widehat{\gC}^{s \leftrightarrow t} \cup \{s,t\}}}} \mid \{Z_{\ell}\}_{\ell \neq s,t}, [Z_s,Z_t]}\nonumber\\
    &= \E\LRm{\frac{\Pi_t(X_t) \Pi_t(X_s)\Pi_s(X_t)[1-\Pi_{s}(X_s)]}{|\widehat{\gC}^{s \leftrightarrow t}|+2} \mathbbm{1}\LRl{R_s > Q_{\alpha}\LRs{\{R_i\}_{i\in \widehat{\gC}_t^{(s)} \cup \{s,t\}}}} \mid \{Z_{\ell}\}_{\ell \neq s,t}, [Z_s,Z_t]}\nonumber\\
    &= \mathrm{A}_2.
\end{align}
Similarly, we can also show that $\mathrm{B}_1 = \mathrm{B}_2$. By recalling \eqref{eq:gap_3_decompose}, we have showed
$\mathrm{III} = 0$. Together with \eqref{eq:exchange_zero_gap_1} and \eqref{eq:exchange_zero_gap_2}, we conclude
\begin{align}\label{eq:zero_gap_t0}
    \E\LRm{\frac{\Pi_t(X_t)\Pi_t(X_s)}{|\widehat{\gC}_t^{(s)}|+2}  \mathfrak{R}(Z_t,Z_s; \widehat{\gC}_t^{(s)} \cup \{s,t\})\mid \{Z_{\ell}\}_{\ell \neq s,t}, [Z_s,Z_t]} = 0,\quad \forall t\geq t_0+1.
\end{align}
Recall the SCC bound \eqref{eq:scop_SCC}, we can show
\begin{align}\nonumber
    \sP\LRl{Y_t \not\in \gI_t^{\rm CAP}(X_t) \mid S_t = 1} &\leq \alpha, \quad \forall t\geq t_0+1.
\end{align}
Recall the FCR bound in Theorem \ref{thm:FCR_decision_dyn_SCOP} and \eqref{eq:scop_FCR_gap_online}, we also have
\begin{align}
    \FCR(T) 
    &\leq \alpha + \sum_{t=0}^{t_0}\E\LRm{\frac{S_t \Delta_t}{ 1 \vee \sum_{j=0}^T S_j} }\nonumber\\
    &\qquad + \sum_{t=t_0+1}^{T}\sum_{s=-n}^{t-1}\E\LRm{\frac{1}{ 1 \vee \sum_{j=0}^T S_j^{(s,t)}}\frac{1}{|\widehat{\gC}_t^{(s)}|+2}\Pi_t(X_t)\Pi_t(X_s)  \mathfrak{R}(Z_t,Z_s; \widehat{\gC}_t^{(s)} \cup \{s,t\}) }\nonumber\\
    &=\alpha + \sum_{t=0}^{t_0}\E\LRm{\frac{S_t \Delta_t}{ 1 \vee \sum_{j=0}^T S_j} },\nonumber
\end{align}
where the equality holds due to \eqref{eq:zero_gap_t0}. Notice that $\Delta_t \leq 1$, we further have
\begin{align}
    \FCR(T) \leq \alpha + \E\LRm{\frac{\sum_{t=0}^{t_0} S_t}{ 1 \vee \sum_{j=0}^T S_j} } \to \alpha,\nonumber
\end{align}
as long as $t_0$ is finite and $\sum_{j=0}^T S_j \to \infty$.

\end{proof}

\subsection{Proof of Theorem \ref{thm:FCR_decision_dyn_mSCOP}}\label{proof:thm:FCR_decision_dyn_mSCOP}
At time $t$, for $-n \leq s \leq t-1$ we define the following candidate set
\begin{align}\label{eq:Nts}
    \gN_t^{(s)} = \LRl{-n\leq j\leq t-1,j\neq s: \Pi_t(X_j) = 1}.
\end{align}
In addition, we let $\Pi_s(\cdot) \equiv 1$ for any $-n\leq s \leq -1$.
According to \eqref{eq:Ada_rule_inter}, the picked calibration set can be rewritten as
\begin{align}\label{eq:online_Ct}
    \widehat{\gC}_{t} = \Big\{-n\leq s \leq t-1: \Pi_t(X_s) \mathbbm{1}\{\Pi_s(X_s) = \Pi_s(X_t)\} = 1,&\nonumber\\
    \prod_{i\in \gN_t^{(s)}} \mathbbm{1}\{\Pi_i(X_s) = \Pi_i(X_t)\}=1\Big\}&.
\end{align}
\subsubsection{Proof of selection-conditional coverage}\label{proof:SCC_dec_driven_ada}
Next, we will prove the following relation: for $-n\leq s\leq t-1$,
\begin{align}\label{eq:zero_gap}
    \E\LRm{\frac{1}{|\widehat{\gC}_t|+1} \Pi_t(X_t) \mathbbm{1}\{s\in \widehat{\gC}_t\}\cdot \mathfrak{R}(Z_t,Z_s; \widehat{\gC}_t \cup \{t\}) \mid \{Z_{\ell}\}_{\ell\neq s,t}} = 0.
\end{align}
Define the leave-one-out picked calibration set as
\begin{align}\label{eq:reference_set_ts}
    \widehat{\gC}_{t}^{(s)} = \Bigg\{\substack{-n\leq i\leq t-1,\\ i\neq s}:\ &\Pi_t(X_i) \mathbbm{1}\{\Pi_i(X_i) = \Pi_i(X_t)\}{\mathbbm{1}\{\Pi_i(X_i) = \Pi_i(X_s)\}} = 1,\nonumber\\
    &\prod_{j\in \gN_t^{(s)}} \mathbbm{1}\{\Pi_j(X_i) = \Pi_j(X_t)\}{\mathbbm{1}\{\Pi_j(X_i) = \Pi_j(X_s)\}}=1\Bigg\}.
\end{align}
By the definition of $\widehat{\gC}_t$ in \eqref{eq:online_Ct}, we know if $s\in \widehat{\gC}_t$ then $\Pi_i(X_s) = \Pi_i(X_t), \forall i\in \gN_t^{(s)}$. It implies that for any $i\leq t-1$ such that $\Pi_t(X_i) = 1$ (i.e., $i\in \gN_t^{(s)}$), we have
\begin{align}
    \mathbbm{1}\{\Pi_i(X_i) = \Pi_i(X_t)\}\mathbbm{1}\{\Pi_i(X_i) = \Pi_i(X_s)\} = \mathbbm{1}\{\Pi_i(X_i) = \Pi_i(X_t)\}.\nonumber
\end{align}
By comparing \eqref{eq:online_Ct} and \eqref{eq:reference_set_ts}, we can guarantee that for any $-n\leq s \leq t-1$
\begin{align}\label{eq:cal_set_ts}
    \widehat{\gC}_t = \widehat{\gC}_{t}^{(s)} \cup \{s\} \text{ if }s\in \widehat{\gC}_t.
\end{align}
Then we can rewrite \eqref{eq:zero_gap} as
\begin{align}\label{eq:zero_gap_mid}
    \E\LRm{\frac{1}{|\widehat{\gC}_{t}^{(s)}|+2} \Pi_t(X_t) \mathbbm{1}\{s\in \widehat{\gC}_t\}\cdot \mathfrak{R}(Z_t,Z_s; \widehat{\gC}_{t}^{(s)} \cup \{s,t\}) \mid \{Z_{\ell}\}_{\ell\neq s,t}} = 0.
\end{align}
Due to the fact $\mathbbm{1}\{\Pi_s(X_s) = \Pi_s(X_t)\} = \Pi_s(X_s)\Pi_s(X_t) + (1-\Pi_s(X_s))(1-\Pi_s(X_t))$, then we can decompose the joint selection indicator in \eqref{eq:zero_gap_mid} as
\begin{align}\label{eq:joint_selection}
    \Pi_t(X_t)\mathbbm{1}\{s\in \widehat{\gC}_{t}\} &= \underbrace{\Pi_t(X_t)\Pi_t(X_s)\Pi_s(X_t)\Pi_s(X_s) \prod_{i\in \gN_t^{(s)}} \mathbbm{1}\{\Pi_i(X_s) = \Pi_i(X_t)\}}_{\gJ_1(X_s,X_t)}\nonumber\\
    &+\underbrace{\Pi_t(X_t)\Pi_t(X_s)(1-\Pi_s(X_s))(1-\Pi_s(X_t))\prod_{i\in \gN_t^{(s)}} \mathbbm{1}\{\Pi_i(X_s) = \Pi_i(X_t)\}}_{\gJ_0(X_s,X_t)}.
\end{align}
Notice that, if $-n\leq s \leq -1$, we know $\{\Pi_i(\cdot)\}_{-n\leq i\leq t-1}$ are independent of $Z_s$ and $Z_t$. Hence $\gN_t^{(s)}$ is independent of $Z_s$ and $Z_t$ by recalling \eqref{eq:Nts}. It follows that $\gJ_1(X_s,X_t)$, $\gJ_0(X_s,X_t)$ and $\widehat{\gC}_{t}^{(s)}$ are all symmetric to $Z_s$ and $Z_t$ conditioning on $\{Z_{\ell}\}_{\ell \neq s,t}$. Then we can show for any $-n\leq s \leq -1$,
\begin{align}\label{eq:zero_gap_offline}
    \E\LRm{\frac{1}{|\widehat{\gC}_{t}^{(s)}|+2} \Pi_t(X_t) \mathbbm{1}\{s\in \widehat{\gC}_t\}\cdot \mathfrak{R}(Z_t,Z_s; \widehat{\gC}_{t}^{(s)} \cup \{s,t\}) \mid \{Z_{\ell}\}_{\ell\neq s,t}} = 0.
\end{align}
Next, we will prove \eqref{eq:zero_gap_mid} for $0\leq s \leq t-1$ by separating the left-hand side into two terms according to the decomposition in \eqref{eq:joint_selection}. 

% \textbf{\emph{Decoupling over $X_s,\ 0\leq s \leq t-1$.}}
% Let $x_{s,1},x_{s,0} \in \sigma(S_0,\ldots,S_{s-1})$ be the values such that $\Pi_{s}(x_{s,1}) = 1$ and $\Pi_s(x_{s,0})=0$ for $0\leq s \leq t-1$. Denote $\{\widetilde{\Pi}_{i,1}^{(s)}\}_{i\geq -n}$ and $\{\widetilde{\Pi}_{i,0}^{(s)}\}_{i\geq -n}$ the virtual selection rules generated by replacing $X_s$ with $x_{s,1}$ and $x_{s,0}$, respectively. 

%Let $\{S_{i,1}^{(s)}\}_{i\geq 0}$ and $\{S_{i,0}^{(s)}\}_{i\geq 0}$ be the corresponding virtual decision sequences.

%Let $x_{t}^{(s,1)} \in \sigma(S_{0,1}^{(s)},\ldots,S_{t-1,1}^{(s)})$ and $x_{t}^{(s,0)} \in \sigma(S_{0,0}^{(s)},\ldots,S_{t-1,0}^{(s)})$ be the values such that $\widetilde{\Pi}_{t,1}^{(s)}(x_{t}^{(s,1)}) = 1$ and $\widetilde{\Pi}_{t,0}^{(s)}(x_{t}^{(s,0)}) = 1$ for $t > s$, respectively. Let $\{S_{i,1}^{(s,t)}\}_{i\geq 0}$ be the virtual decision sequence generated by firstly replacing $X_s$ with $x_{s,1}$, and then replacing $X_t$ with $x_{t}^{(s,1)}$. Let $\{S_{i,0}^{(s,t)}\}_{i\geq 0}$ be the virtual decision sequence generated by firstly replacing $X_s$ with $x_{s,0}$, and then replacing $X_t$ with $x_{t}^{(s,1)}$.
\paragraph{Online term 1.}
Recall the construction in Section \ref{proof:them:FCR_decision_dyn_SCOP}, it holds that $\widetilde{\Pi}_{i,1}^{(s)}(\cdot) = \Pi_i(\cdot),\forall i\leq t$ under the event $\{S_s = \Pi_s(X_s) = 1\}$. Define the decoupled sets
\begin{align}
    \widetilde{\gN}_{t,1}^{(s)} &= \LRl{0\leq j\leq t-1, j\neq s: \widetilde{\Pi}_{t,1}^{(s)}(X_j) = 1},\nonumber\\
    \widetilde{\gC}_{t, 1}^{(s)} = \Bigg\{\substack{-n\leq i\leq t-1,\\ i\neq s}:\ &\widetilde{\Pi}_{t,1}^{(s)}(X_i) \mathbbm{1}\{\widetilde{\Pi}_{i,1}^{(s)}(X_i) = \widetilde{\Pi}_{i,1}^{(s)}(X_t)\}\mathbbm{1}\{\widetilde{\Pi}_{i,1}^{(s)}(X_i) = \widetilde{\Pi}_{i,1}^{(s)}(X_s)\} = 1,\nonumber\\
    &\prod_{j\in \widetilde{\gN}_{t,1}^{(s)}} \mathbbm{1}\{\widetilde{\Pi}_{j,1}^{(s)}(X_i) = \widetilde{\Pi}_{j,1}^{(s)}(X_t)\}\mathbbm{1}\{\widetilde{\Pi}_{j,1}^{(s)}(X_i) = \widetilde{\Pi}_{j,1}^{(s)}(X_s)\}=1\Bigg\}.\nonumber
\end{align}
Then $\widetilde{\gN}_{t,1}^{(s)} = \gN_t^{(s)}$ and $ \widetilde{\gC}_{t, 1}^{(s)} = \widehat{\gC}_{t}^{(s)}$ hold under the event $\{S_s = \Pi_s(X_s) = 1\}$. Importantly, the virtual set $\widetilde{\gC}_{t, 1}^{(s)}$ is symmetric to $(X_s,X_t)$ since $\widetilde{\gN}_{t,1}^{(s)}$ is independent of $Z_s$ and $Z_t$. With the ingredients above, we define the decoupled version of $\gJ_1(X_s,X_t)$ in \eqref{eq:joint_selection},
\begin{align}
    \widetilde{\gJ}_1(X_s,X_t) = \widetilde{\Pi}_{t,1}^{(s)}(X_t) \widetilde{\Pi}_{t,1}^{(s)}(X_s)\widetilde{\Pi}_{s,1}^{(s)}(X_s)\widetilde{\Pi}_{s,1}^{(s)}(X_t)\prod_{i\in \widetilde{\gN}_{t,1}^{(s)}}\mathbbm{1}\{\widetilde{\Pi}_{i,1}^{(s)}(X_s) = \widetilde{\Pi}_{i,1}^{(s)}(X_t)\}.\nonumber
\end{align}
Clearly, $\widetilde{\gJ}_1(X_s,X_t)$ is also symmetric to $(X_s,X_t)$. By the definition of $\gJ_1(X_s,X_t)$ in \eqref{eq:joint_selection}, we also know $\Pi_s(X_s)\gJ_1(X_s,X_t) = \gJ_1(X_s,X_t)$. Using the exchangeability, we can show
\begin{align}\label{eq:zero_gap_1}
    &\E\LRm{\frac{1}{|\widehat{\gC}_{t}^{(s)}|+2} \gJ_1(X_s,X_t)\cdot  \mathfrak{R}(Z_t,Z_s; \widehat{\gC}_{t}^{(s)} \cup \{s,t\}) \mid \{Z_{\ell}\}_{\ell \neq s,t}}\nonumber\\
    =& \E\LRm{\frac{1}{|\widetilde{\gC}_{t, 1}^{(s)}|+2} \Pi_s(X_s)\widetilde{\gJ}_1(X_s,X_t)\cdot  \mathfrak{R}(Z_t,Z_s; \widetilde{\gC}_{t, 1}^{(s)} \cup \{s,t\}) \mid \{Z_{\ell}\}_{\ell \neq s,t}}\nonumber\\
    =& \E\LRm{\frac{1}{|\widetilde{\gC}_{t, 1}^{(s)}|+2} \widetilde{\gJ}_1(X_s,X_t)  \cdot  \mathfrak{R}(Z_t,Z_s; \widetilde{\gC}_{t, 1}^{(s)} \cup \{s,t\}) \mid \{Z_{\ell}\}_{\ell \neq s,t}}\nonumber\\
    =& 0,
\end{align}
where the second last equality is true because $\Pi_s(X_s) \widetilde{\gJ}_1(X_s,X_t) = \widetilde{\gJ}_1(X_s,X_t)$ due to the fact $\Pi_s(\cdot) = \widetilde{\Pi}_{s,1}^{(s)}(\cdot)$; and the last equality holds since conditioning on $\{Z_{\ell}\}_{\ell\neq s,t}$, $\widetilde{\gJ}_1(X_s,X_t)$ and $\widetilde{\gC}_{t, 1}^{(s)}$ are symmetric to $Z_s$ and $Z_t$.

\paragraph{Online term 2.} Similarly, it holds that $\widetilde{\Pi}_{i,0}^{(s)}(\cdot) = \Pi_i(\cdot),\forall i\leq t$ under the event $\{S_s = 0\}$. Define
\begin{align}
    \widetilde{\gN}_{t,0}^{(s)} &= \LRl{0\leq j\leq t-1, j\neq s: \widetilde{\Pi}_{t,0}^{(s)}(X_j) = 1},\nonumber\\
    \widetilde{\gC}_{t, 0}^{(s)} = \Bigg\{\substack{-n\leq i \leq t-1,\\ i\neq s}:\ &\widetilde{\Pi}_{t,0}^{(s)}(X_i) \mathbbm{1}\{\widetilde{\Pi}_{i,0}^{(s)}(X_i) = \widetilde{\Pi}_{i,0}^{(s)}(X_t)\}\mathbbm{1}\{\widetilde{\Pi}_{i,0}^{(s)}(X_i) = \widetilde{\Pi}_{i,0}^{(s)}(X_s)\} = 1,\nonumber\\
    &\prod_{j\in \widetilde{\gN}_{t,0}^{(s)}} \mathbbm{1}\{\widetilde{\Pi}_{j,0}^{(s)}(X_i) = \widetilde{\Pi}_{j,0}^{(s)}(X_t)\}\mathbbm{1}\{\widetilde{\Pi}_{j,0}^{(s)}(X_i) = \widetilde{\Pi}_{j,0}^{(s)}(X_s)\}=1\Bigg\}.\nonumber
\end{align}
Then $\widetilde{\gN}_{t,0}^{(s)} = \gN_t^{(s)}$ and $ \widetilde{\gC}_{t, 0}^{(s)} = \widehat{\gC}_{t}^{(s)}$ hold under the event $\{S_s = \Pi_s(X_s) = 0\}$. Importantly, the virtual set $\widetilde{\gC}_{t, 0}^{(s)}$ is symmetric to $(X_s,X_t)$ conditioning on $\{Z_{\ell}\}_{\ell\neq s,t}$ since $\widetilde{\gN}_{t,0}^{(s)}$ is independent of $Z_s$ and $Z_t$. With the ingredients above, we define
\begin{align}
    \widetilde{\gJ}_0(X_s,X_t) = \widetilde{\Pi}_{t,0}^{(s)}(X_t) \widetilde{\Pi}_{t,0}^{(s)}(X_s)[1-\widetilde{\Pi}_{s,0}^{(s)}(X_s)][1-\widetilde{\Pi}_{s,0}^{(s)}(X_t)]\prod_{i\in \widetilde{\gN}_{t,0}^{(s)}}\mathbbm{1}\{\widetilde{\Pi}_{i,0}^{(s)}(X_s) = \widetilde{\Pi}_{i,0}^{(s)}(X_t)\}.\nonumber
\end{align}
Clearly, $\widetilde{\gJ}_0(X_s,X_t)$ is also symmetric to $(X_s,X_t)$ conditioning on $\{Z_{\ell}\}_{\ell\neq s,t}$. By the definition of $\gJ_0(X_s,X_t)$ in \eqref{eq:joint_selection}, we also know $[1-\Pi_s(X_s)]\gJ_0(X_s,X_t) = \gJ_0(X_s,X_t)$. Using the exchangeability between $Z_s$ and $Z_t$, we can show
\begin{align}\label{eq:zero_gap_0}
    &\E\LRm{\frac{1}{|\widehat{\gC}_{t}^{(s)}|+2} \gJ_0(X_s,X_t)  \cdot  \mathfrak{R}(Z_t,Z_s; \widehat{\gC}_{t}^{(s)} \cup \{s,t\}) \mid \{Z_{\ell}\}_{\ell \neq s,t}}\nonumber\\
    =& \E\LRm{\frac{1-\Pi_s(X_s)}{|\widetilde{\gC}_{t, 0}^{(s)}|+2} \widetilde{\gJ}_0(X_s,X_t)  \cdot  \mathfrak{R}(Z_t,Z_s; \widetilde{\gC}_{t, 0}^{(s)} \cup \{s,t\}) \mid \{Z_{\ell}\}_{\ell \neq s,t}}\nonumber\\
    =& \E\LRm{\frac{1}{|\widetilde{\gC}_{t, 0}^{(s)}|+2} \widetilde{\gJ}_0(X_s,X_t)  \cdot  \mathfrak{R}(Z_t,Z_s; \widetilde{\gC}_{t, 0}^{(s)} \cup \{s,t\}) \mid \{Z_{\ell}\}_{\ell \neq s,t}}\nonumber\\
    =& 0,
\end{align}
where the second last equality is true because $[1-\Pi_s(X_s)] \widetilde{\gJ}_0(X_s,X_t) = \widetilde{\gJ}_0(X_s,X_t)$ due to the fact $\Pi_s(\cdot) = \widetilde{\Pi}_{s,0}^{(s)}(\cdot)$.
Combining \eqref{eq:zero_gap_1} and \eqref{eq:zero_gap_0}, we can show for any $0\leq s \leq t-1$,
\begin{align}
    \E\LRm{\frac{\Pi_t(X_t) \mathbbm{1}\{s\in \widehat{\gC}_t\}}{|\widehat{\gC}_{t}^{(s)}|+2}  \LRs{\mathbbm{1}\{R_t > Q_{\alpha}(\{R_i\}_{i\in \widehat{\gC}_{t}^{(s)} \cup \{s,t\}})\} - \mathbbm{1}\{R_s > Q_{\alpha}(\{R_i\}_{i\in \widehat{\gC}_{t}^{(s)} \cup \{s,t\}})\}}} = 0.\nonumber
\end{align}
Recalling the equivalence in \eqref{eq:zero_gap_mid}, we can prove the relation \eqref{eq:zero_gap}.

\paragraph{Conclusion.} 
Now we proceed to prove the results of selection-conditional coverage. Notice that
    \begin{align}
        &\sP\LRl{Y_t\not\in \gI_t^{\rm CAP}(X_t)\mid S_t = 1}\nonumber\\
        &\leq \alpha + \E\LRm{\frac{1}{|\widehat{\gC}_t|+1}\sum_{s=-n}^{t-1} \mathbbm{1}\{s\in \widehat{\gC}_t\} \mathfrak{R}(Z_t,Z_s; \widehat{\gC} \cup \{t\}) \mid S_t = 1}\nonumber\\
        &= \alpha + \sum_{s=-n}^{-1} \frac{1}{\sP(S_t=1)}\E\LRm{\frac{\Pi_t(X_t) \mathbbm{1}\{s\in \widehat{\gC}_t\}}{|\widehat{\gC}_t|+1} \mathfrak{R}(Z_t,Z_s; \widehat{\gC} \cup \{t\}) } \nonumber\\
        &\qquad + \sum_{s=0}^{t-1}\frac{1}{\sP(S_t=1)}\E\LRm{\frac{\Pi_t(X_t)\mathbbm{1}\{s\in \widehat{\gC}_t\}}{|\widehat{\gC}_t|+1}\mathfrak{R}(Z_t,Z_s; \widehat{\gC} \cup \{t\}) }\nonumber\\
        &=\alpha,\nonumber
    \end{align}
    where the last equality follows from taking full expectation on \eqref{eq:zero_gap} and \eqref{eq:zero_gap_offline}. 

    \subsubsection{Verification of two symmetric properties}\label{appen:verification_P12}
    By \eqref{eq:joint_selection} and analysis in the previous subsection, we have
    \begin{align}
        \Pi_{t,s}^{\rm Ada}(X_s) \Pi_t(X_t) &= \gJ_1(X_s,X_t) + \gJ_0(X_s,X_t)\nonumber\\
        &= \widetilde{\gJ}_{1}(X_s, X_t) + \widetilde{\gJ}_0(X_s,X_t).\nonumber
    \end{align}
    Since $\widetilde{\gJ}_{1}(X_s, X_t)$ and $\widetilde{\gJ}_{0}(X_s, X_t)$ are both symmetric to $(X_s,X_t)$, we have verified \eqref{eq:indicator_prod_symmetry}. Recalling \eqref{eq:reference_set_ts}, under the event $\Pi_t(X_t)\Pi_{t,s}^{\rm Ada}(X_s) = 1$, we have
    \begin{align}
        \widehat{\gC}_t^{(s)}\cdot\Pi_t(X_t)\Pi_{t,s}^{\rm Ada}(X_s) &= \widehat{\gC}_t^{(s)}\cdot \gJ_1(X_s,X_t) + \widehat{\gC}_t^{(s)}\cdot \gJ_0(X_s,X_t)\nonumber\\
        &= \widetilde{\gC}_{t,1}^{(s)}\cdot \widetilde{\gJ}_1(X_s,X_t) + \widetilde{\gC}_{t,0}^{(s)}\cdot \widetilde{\gJ}_0(X_s,X_t).\nonumber
    \end{align}
    Since $\widetilde{\gC}_{t,1}^{(s)}$ and $\widetilde{\gC}_{t,0}^{(s)}$ are symmetric to $(X_s,X_t)$, we have verified \eqref{eq:cal_set_symmetry}.

\subsubsection{Proof of FCR control} 
    % Let $\Pi_{j}^{(t)}(\cdot)$ be corresponding selection rule by replacing $X_t$ with $x_{t}^* \in \sigma(\{S_i\}_{i=0}^{t-1})$ such that $\Pi_{t}(x_{t}^*) = 1$. Correspondingly, we denote $S_j^{(t)} = \Pi_j^{(t)}(X_j)$ for any $j\geq 0$. According to our assumption $\Pi_t(\cdot) \in \sigma(\{S_i\}_{i=0}^{t-1})$, we know:
    % \begin{itemize}
    %     \item[(1)] $S_j^{(t)} = S_j$ for any $0\leq j \leq t-1$;
    %     \item[(2)] if $S_t=1$, it holds that $S_j^{(t)} = \Pi_{j}^{(t)}(X_{j}) = \Pi_{j}(X_{j}) = S_j$ for any $j\geq t$.
    % \end{itemize}
    Since $\Pi_{t}(\cdot)$ is independent of $\{(X_i,Y_i)\}_{i=-1}^{-n}$,
    for any $-n \leq s \leq -1$, using the exchangeability between $Z_s$ and $Z_t$ we have
    \begin{align}\label{eq:offline_FCR_gap}
        &\E\LRm{\frac{1}{ 1 \vee \sum_{j=0}^T S_j} \frac{\Pi_t(X_t) \mathbbm{1}\{s\in \widehat{\gC}_t\}}{|\widehat{\gC}_t| + 1} \cdot \mathfrak{R}(Z_t,Z_s;\widehat{\gC}_t\cup \{t\})}\nonumber\\
        &= \E\LRm{\frac{1}{ 1 \vee \sum_{j=0}^T S_j^{(t)}} \frac{\Pi_t(X_t) \mathbbm{1}\{s\in \widehat{\gC}_t\}}{|\widehat{\gC}_{t}^{(s)}| + 2} \cdot \mathfrak{R}(Z_t,Z_s;\widehat{\gC}_{t}^{(s)}\cup \{s,t\})}\nonumber\\
        &= \E\LRm{\frac{1}{ 1 \vee \sum_{j=0}^T S_j^{(t)}} \E\LRm{\frac{1}{|\widehat{\gC}_{t}^{(s)}| + 2} \Pi_t(X_t) \Pi_t(X_s)\cdot \mathfrak{R}(Z_t,Z_s;\widehat{\gC}_{t}^{(s)}\cup \{s,t\}) \mid \{Z_{\ell}\}_{\ell \neq s,t}}}\nonumber\\
        &=0,
    \end{align}
    where the first equality holds due to \eqref{eq:cal_set_ts}; and the second equality holds since $\{S_j^{(t)}\}_{t\geq 0}$ are independent of $Z_s$ and $Z_t$ for $-1\leq s \leq -n$; and the last equality holds due to \eqref{eq:zero_gap_offline}.
    Then we can bound $\FCR$ by
    \begin{align}\label{eq:FCR_upper_online_part}
        \FCR(T) 
        &\leq \alpha + \E\LRm{\sum_{t = 0}^T \frac{1}{1 \vee \sum_{j=0}^T S_j}\frac{1}{|\widehat{\gC}_t|+1}\sum_{s=-n}^{-1}\Pi_t(X_t) \mathbbm{1}\{s\in \widehat{\gC}_t\} \cdot \mathfrak{R}(Z_t,Z_s;\widehat{\gC}_t\cup \{t\})}\nonumber\\
        &\qquad+\E\LRm{\sum_{t = 0}^T \frac{1}{1 \vee \sum_{j=0}^T S_j}\frac{1}{|\widehat{\gC}_t|+1}\sum_{s=0}^{t-1}\Pi_t(X_t) \mathbbm{1}\{s\in \widehat{\gC}_t\} \cdot \mathfrak{R}(Z_t,Z_s;\widehat{\gC}_t\cup \{t\})}\nonumber\\
        &= \alpha+\E\LRm{\sum_{t = 0}^T \frac{1}{1 \vee \sum_{j=0}^T S_j}\frac{1}{|\widehat{\gC}_t|+1}\sum_{s=0}^{t-1}\Pi_t(X_t) \mathbbm{1}\{s\in \widehat{\gC}_t\} \cdot \mathfrak{R}(Z_t,Z_s;\widehat{\gC}_t\cup \{t\})}\nonumber\\
        &= \alpha + \sum_{t = 0}^T \sum_{s = 0}^{t-1} \E\LRm{\frac{1}{1 \vee \sum_{j=0}^T S_j}\frac{1}{|\widehat{\gC}_{t}^{(s)}|+2} \Pi_t(X_t) \mathbbm{1}\{s\in \widehat{\gC}_t\} \cdot \mathfrak{R}(Z_t,Z_s;\widehat{\gC}_{t}^{(s)}\cup \{s,t\})}\nonumber\\
        &= \alpha + \sum_{t = 0}^T \sum_{s = 0}^{t-1} \E\LRm{\frac{1}{1 \vee \sum_{j=0}^T S_j}\frac{\gJ_1(X_s,X_t)}{|\widehat{\gC}_{t}^{(s)}|+2}  \cdot \mathfrak{R}(Z_t,Z_s;\widehat{\gC}_{t}^{(s)}\cup \{s,t\})}\nonumber\\
        &\qquad + \sum_{t = 0}^T \sum_{s = 0}^{t-1} \E\LRm{\frac{1}{1 \vee \sum_{j=0}^T S_j}\frac{\gJ_0(X_s,X_t)}{|\widehat{\gC}_{t}^{(s)}|+2}  \cdot \mathfrak{R}(Z_t,Z_s;\widehat{\gC}_{t}^{(s)}\cup \{s,t\})},
    \end{align}
    where the first equality holds due to \eqref{eq:offline_FCR_gap}; the second last equality holds due to \eqref{eq:cal_set_ts}; and the last equality holds due to \eqref{eq:joint_selection}.

    % \textbf{\emph{Decoupling over both $X_s$ and $X_t$.}}
    % Let $\{S_{i,1}^{(s)} = \widetilde{\Pi}_{i,1}^{(s)}(X_i)\}_{i\geq 0}$ and $\{S_{i,0}^{(s)}= \widetilde{\Pi}_{i,0}^{(s)}(X_i)\}_{i\geq 0}$ be the corresponding virtual decision sequences by replacing $X_s$ with $x_{s,1}$ and $x_{s,0}$, respectively.
    % Let $x_{t}^{(s,1)} \in \sigma(S_{0,1}^{(s)},\ldots,S_{t-1,1}^{(s)})$ and $x_{t}^{(s,0)} \in \sigma(S_{0,0}^{(s)},\ldots,S_{t-1,0}^{(s)})$ be the values such that $\widetilde{\Pi}_{t,1}^{(s)}(x_{t}^{(s,1)}) = 1$ and $\widetilde{\Pi}_{t,0}^{(s)}(x_{t}^{(s,0)}) = 1$ for $t > s$, respectively. Let $\{S_{i,1}^{(s,t)}\}_{i\geq 0}$ be the virtual decision sequence generated by firstly replacing $X_s$ with $x_{s,1}$, and then replacing $X_t$ with $x_{t}^{(s,1)}$. Let $\{S_{i,0}^{(s,t)}\}_{i\geq 0}$ be the virtual decision sequence generated by firstly replacing $X_s$ with $x_{s,0}$, and then replacing $X_t$ with $x_{t}^{(s,1)}$. In this case, we can guarantee that $S_{i,1}^{(s,t)}, S_{i,0}^{(s,t)} \independent (Z_s,Z_t)$ for any $i\geq 0$ because $x_{s,1},x_{s,0} \independent (Z_s,Z_t)$ and $x_{t}^{(s,1)},x_{t}^{(s,0)} \independent (Z_s,Z_t)$. We have the following implications:
    % \begin{itemize}
    %     \item[(1)] $S_{i,1}^{(s,t)} = S_{i,0}^{(s,t)} = S_i$ for $i\leq s-1$;
    %     \item[(2)] $S_{i,1}^{(s,t)} = S_{i,1}^{(s)}$ and $S_{i,0}^{(s,t)} = S_{i,0}^{(s)}$ for $s\leq i \leq t-1$;
    %     \item[(3)] If $S_t = 1$, $S_{i,1}^{(s,t)} = S_{i,1}^{(s)}$ for $i \geq t$.
    %     \item[(4)] If $S_t = 0$, $S_{i,0}^{(s,t)} = S_{i,0}^{(s)}$ for $i \geq t$.
    % \end{itemize}
    Since $\{S_{j,1}^{(s,t)}\}_{j\geq 0}$ are independent of $Z_s$ and $Z_t$, we have
    \begin{align}\label{eq:zero_gap_intersect_2}
        &\E\LRm{\frac{1}{1 \vee \sum_{j=0}^T S_j}\frac{ \gJ_1(X_s,X_t)}{|\widehat{\gC}_{t}^{(s)}|+2}  \cdot \mathfrak{R}(Z_t,Z_s;\widehat{\gC}_{t}^{(s)}\cup \{s,t\})}\nonumber\\
        %=& \E\LRm{\frac{1}{1 \vee \sum_{j=0}^T S_j}\frac{ \gJ_1(X_s,X_t)}{|\widehat{\gC}_{t}^{(s)}|+2}  \cdot \mathfrak{R}(Z_t,Z_s;\widehat{\gC}_{t}^{(s)}\cup \{s,t\})}\nonumber\\
        =& \E\LRm{\frac{1}{1 \vee \sum_{j=0}^T S_{j,1}^{(s,t)}}\E\LRm{\frac{\gJ_1(X_s,X_t)}{|\widehat{\gC}_{t}^{(s)}|+2}  \cdot \mathfrak{R}(Z_t,Z_s;\widehat{\gC}_{t}^{(s)}\cup \{s,t\})\mid \{Z_{\ell}\}_{\ell \neq s,t}}}\nonumber\\
        =& 0,
    \end{align}
    where the first equality holds since $S_tS_s \gJ_1(X_s,X_t) = \gJ_1(X_s,X_t)$; and the last equality holds due to \eqref{eq:zero_gap_1}.
    Similarly, using \eqref{eq:zero_gap_0}, we also have
    \begin{align}\label{eq:zero_gap_intersect_3}
        \E\LRm{\frac{1}{1 \vee \sum_{j=0}^T S_j}\frac{\gJ_0(X_s,X_t)}{|\widehat{\gC}_{t}^{(s)}|+2}  \cdot \mathfrak{R}(Z_t,Z_s;\widehat{\gC}_{t}^{(s)}\cup \{s,t\})} = 0.
    \end{align}
    Substituting \eqref{eq:zero_gap_intersect_2} and \eqref{eq:zero_gap_intersect_3} into \eqref{eq:FCR_upper_online_part}, we can prove $\FCR \leq \alpha$.

\section{Additional settings in Section \ref{sec:symmetric_selection}}
%{\color{red}[Maybe we should replace this section with Appendix D]}
\subsection{CAP with a fixed holdout set}\label{subsec:cal_sel_fixed_set}
In this section, we provide the FCR control results of CAP for the selection procedure in Section \ref{sec:symmetric_selection} when the selection and calibration depend only on the fixed holdout set.

The selection indicators are given as
\begin{align}\label{eq:def_selection_t}
    S_t = \Pi_t(X_t) = \Indicator{V_t \leq  \gA\LRs{\{V_i\}_{-n \leq i\leq -1}}},\quad \text{for any } t \geq 0,
\end{align}
where $\gA: \sR^{n} \to \sR$ is some symmetric function. In this case, the selected calibration set is given by $\widehat{\gC}_t = \LRl{-n\leq s\leq -1: V_{s} > \gA\LRs{V_t, \{V_i\}_{-n \leq i\leq -1, i\neq s}}}$. Then we can construct the $(1-\alpha)$-conditional PI for $Y_t$:
\begin{align}\label{eq:cond_PI}
    \gI_t^{\rm{CAP}}(X_t) = \widehat{\mu}(X_t) \pm q_{\alpha}\LRs{\{R_i\}_{i\in \widehat{\gC}_t}}.
\end{align}

\begin{theorem}\label{thm:FCR_cal_selection}
    Suppose $\{(X_t,Y_t)\}_{t\geq -n}$ are i.i.d. data points. If the function $\gA$ is invariant to the permutation to its inputs, we can guarantee that for any $T \geq 0$,
    \begin{align}\label{eq:FCR_cal_selection}
        \FCR(T) \leq \alpha + \sum_{t=0}^T \E\LRm{\frac{S_t}{1 \vee \sum_{j=0}^T S_j} \sum_{s=-n}^{-1} \frac{C_{t,s} }{|\widehat{\gC}_t|+1}  \frac{2\LRabs{\widehat{q}^{(s\gets t)} - \widehat{q}} }{1 - \widehat{q}^{(s\gets t)}} },
    \end{align}
    where $\widehat{q}^{(s\gets t)} = F_V\{\gA\LRs{\{V_i\}_{i\neq s}, V_t}\}$, $\widehat{q} = F_V\{\gA\LRs{\{V_i\}_{i=1}^n}\}$, and $F_V(\cdot)$ is the cumulative distribution function of $\{V_i\}_{i\geq -n}$.
\end{theorem}

If $\gA$ in \eqref{eq:def_selection_t} returns the sample quantile, the next corollary shows CAP can exactly control $\FCR$ below the target level.

\begin{corollary}\label{cor:FCR_cal_quantile}
    If $\gA(\{V_i\}_{i=1}^n)$ is the $\ell$-th smallest value in $\{V_i\}_{i=1}^n$ for any $\ell\leq n-1$, then the FCR value can be controlled at $\FCR(T) \leq \alpha$ for any $T \geq 0$.
\end{corollary}

\begin{proof}
    We write $\ermV_{(\ell)}^{[n+1]}$, $\ermV_{(\ell)}^{[n+1]\setminus \{s\}}$, and $\ermV_{(\ell)}^{[n+1]\setminus \{t\}}$ as the $\ell$-th smallest values in $\{V_i\}_{i=1}^n \cup \{V_t\}$, $\{V_i\}_{i=1,i\neq s}^n \cup \{V_t\}$, and $\{V_i\}_{i=1}^n$, respectively. Notice that,
    \begin{align*}
        S_t &= \Indicator{V_t \leq  \ermV_{(\ell)}^{[n+1]\setminus \{t\}}} = \Indicator{V_t \leq  \ermV_{(\ell)}^{[n+1]}},\\
        C_{t,s} &= \Indicator{V_{s} > \ermV_{(\ell)}^{[n+1]\setminus \{s\}}} = \Indicator{V_{s} > \ermV_{(\ell)}^{[n+1]}}.
    \end{align*}
    Under event $S_t C_{t,s} = 1$, removing $V_t$ or $V_{s}$ will not change the ranks of $V_i$ for $i \neq s$. Hence we have 
    \begin{align*}
        S_t C_{t,s}\cdot \ermV_{(\ell)}^{[n+1]\setminus \{t\}} = S_t C_{t,s}\cdot \ermV_{(\ell)}^{[n+1]\setminus \{s\}}.
    \end{align*}
    Together with the definitions $\widehat{q}^{(s\gets t)} = F_V(\gA\LRs{\{V_i\}_{i\neq s}, V_t})$, and $\widehat{q} = F_V(\gA\LRs{\{V_i\}_{i\neq s}, V_s})$, we can conclude that
    \begin{align*}
        S_t C_{t,s}\cdot \widehat{q}^{(s\gets t)} = S_t C_{t,s}\cdot \widehat{q}.
    \end{align*}
    Plugging it into \eqref{eq:FCR_cal_selection}, we get the desired bound $\FCR(T) \leq \alpha$.
\end{proof}

The next corollary provides the error bound for $\FCR(T)$ if $\gA$ returns the sample mean.

\begin{corollary}\label{cor:FCR_cal_mean}
    Let $f_{V}(\cdot)$ be the density function of $\{V_i\}_{i\geq -n}$.
    Suppose $f_{V}(\cdot) \leq \rho_v$ and $|\gA(\{V_i\}_{i=1}^n) - \gA(\{V_i\}_{i=1, i\neq s}^n, V_t)| \leq \gamma_v/n$ for some positive constants $\rho_v$ and $\gamma_v$. Then we have
    \begin{align*}
        \FCR(T) \leq \alpha + \frac{2\rho_v \gamma_v}{n} \E\LRm{\frac{1}{1 - \widehat{q} - \rho_v\gamma_v/n}}.
    \end{align*}
\end{corollary}
\begin{proof}
    By the definitions of $\widehat{q}^{(s\gets t)}$ and $\widehat{q}$, we can bound their difference by
    \begin{align}\label{eq:q_st_diff}
        \LRabs{\widehat{q}^{(s\gets t)} - \widehat{q}} &\leq \LRabs{F_V(\gA\LRs{\{V_i\}_{i\neq s}, V_t}) - F_V(\gA\LRs{\{V_i\}_{i\neq s}, V_s})}\nonumber\\
        &\leq \rho_v \LRabs{\gA\LRs{\{V_i\}_{i\neq s}, V_t} - \gA\LRs{\{V_i\}_{i\neq s}, V_s}}\nonumber\\
        &\leq \frac{\rho_v \gamma_v}{n},
    \end{align}
    where we used the assumptions $F_V^{\prime} \leq \rho_v$ and $|\gA(\{V_i\}_{i=1}^n) - \gA(\{V_i\}_{i=1, i\neq s}^n, V_t)| \leq \frac{\gamma_v}{n}$. Plugging \eqref{eq:q_st_diff} into the error term in \eqref{eq:FCR_cal_selection} gives
    \begin{align}
        &\E\LRm{\sum_{t=0}^T \E\LRm{\frac{S_t}{\sum_{j=0,j\neq t} S_j +1} \sum_{s=-n}^{-1} \frac{C_{t,s} }{|\widehat{\gC}_t|+1}  \frac{2\LRabs{\widehat{q}^{(s\gets t)} - \widehat{q}} }{1 - \widehat{q}^{(s\gets t)}} }  }\nonumber\\
        &\qquad \leq \frac{2\rho_v \gamma_v}{n}\E\LRm{\sum_{t=0}^T\frac{S_t}{\sum_{j=0}^T S_j \vee 1} \sum_{s=-n}^{-1} \frac{C_{t,s}}{|\widehat{\gC}_t|+1} \frac{1 }{1 - \widehat{q}^{(s\gets t)}} }\nonumber\\
        &\qquad \leq \frac{2\rho_v \gamma_v}{n}\E\LRm{\sum_{t=0}^T\frac{S_t}{\sum_{j=0}^T S_j \vee 1} \frac{1}{1 - \widehat{q} - \rho_v\gamma_v/n}}\nonumber\\
        &\qquad = \frac{2\rho_v \gamma_v}{n}\E\LRm{\sum_{t=0}^T\frac{S_t}{\sum_{j=0,j\neq t}^T S_j + 1} \frac{1}{1 - \widehat{q} - \rho_v\gamma_v/n}}\nonumber\\
        &\qquad = \frac{2\rho_v \gamma_v}{n}\E\LRm{\sum_{t=0}^T\frac{1}{\sum_{j=0,j\neq t}^T S_j + 1} \frac{1 - \widehat{q}}{1 - \widehat{q} - \rho_v\gamma_v/n}}\nonumber\\
        &\qquad = \frac{2\rho_v \gamma_v}{n} \frac{1}{T+1} \sum_{t=0}^T\E\LRm{\frac{1 - \widehat{q}^{T+1}}{1 - \widehat{q}} \frac{1 - \widehat{q}}{1 - \widehat{q} - \rho_v\gamma_v/n}}\nonumber\\
        &\qquad \leq \frac{2\rho_v \gamma_v}{n} \E\LRm{\frac{1}{1 - \widehat{q} - \rho_v\gamma_v/n}},
    \end{align}
    where the last equality holds due to $\sum_{j=0,j\neq t}^T S_j \sim \operatorname{Binomial}(T, 1-\widehat{q})$ given the calibration set such that
    \begin{align*}
        \E\LRm{\LRs{\sum_{j=0,j\neq t}^T S_j + 1}^{-1} \mid \{Z_i\}_{i=1}^n} = \frac{1}{T+1}\frac{1 - \widehat{q}^{T+1}}{1 - \widehat{q}}.
    \end{align*}
\end{proof}

\subsection{CAP with a moving-window holdout set}\label{appen:moving_window}

%\textbf{\textcolor{red}{[Need modification.]}}

In Sections \ref{sec:decision_selection} and \ref{sec:symmetric_selection}, we construct the selected holdout set $\widehat{\gC}_t$ based on the full calibration set $\gC_t^{\rm{incre}} = \{-n,\ldots,t-1\}$, which may lead to a heavy burden on computation and memory when $t$ is large. Now we consider an efficient online scheme by setting the holdout set as a moving window with fixed length $n$, that is $\gC_t = \gC_t^{\rm{window}} = \{t-n,\ldots,t-1\}$. As for the symmetric selection rule, we allow the selection rule $\Pi_t(\cdot)$ to depend on the data in $\gC_t^{\rm{window}}$ only, which means
\begin{align*}
    S_t = \Pi_t(X_t) = \Indicator{V_t \leq \gA_t\LRs{\{V_i\}_{t-n \leq i \leq t-1}}}.
\end{align*}
In this case, the selected calibration set is given by
\begin{align*}
    \widehat{\gC}_t = \LRl{t-n \leq s \leq t-1: V_s \leq \gA_t\LRs{\{V_i\}_{t-n \leq i \leq t-1, i\neq s}, V_s}}.
\end{align*}
Then the memory cost will be kept at $n$ during the online process. The following theorem reveals the property of Algorithm \ref{alg:main} under symmetric selection rules.

\begin{theorem}\label{thm:FCR_swap_window}
    Under the conditions of Theorem \ref{thm:FCR_swap}. The Algorithm \ref{alg:main} with $\widehat{\gC}_t = \widehat{\gC}_t^{\rm{ada}}$ satisfies
    \begin{align}\label{eq:swap_FCR_bound_window}
        \FCR(T) \leq \alpha\cdot \LRl{1 + \E\LRm{\max_{0\leq t\leq T}\frac{S_t \epsilon_n(t)}{\LRs{\sum_{j=0}^T S_j - \epsilon_n(t)}\vee 1}} + \frac{9}{T + n}},
    \end{align}
    where $\epsilon_n(t) = 2\sum_{j=(t-n) \vee 0}^{t-1} \sigma_j + (\sqrt{e\rho} + 1)\log(1/\delta) + 2^{-1}$.
\end{theorem}

Since the window size of the full calibration set is fixed at $n$, the perturbation to $\sum_{j=0}^T S_j$ caused by replacing $V_s$ with $V_t$ will be limited to $\sum_{j= t-n}^{t-1} \sigma_j$.

\section{Proofs for selection with symmetric thresholds}

\subsection{Proof of Theorem \ref{thm:mFCR_swap_scop}}\label{proof:thm:mFCR_swap_scop}
\begin{proof}
According to the adaptive rule in \eqref{eq:Ada_rule_swap}, we have
\begin{align}
    \Pi_{t,s}^{\rm Ada}(X_s) \Pi_t(X_t) &= \mathbbm{1}\{V_s \leq \gA\LRs{\{V_{\ell}\}_{\ell\neq s}, V_t}\} \mathbbm{1}\{V_s \leq \gA\LRs{\{V_{\ell}\}_{\ell\neq s}, V_s}\}.\nonumber
\end{align}
Since $\gA$ is symmetric to its input, the symmetric property \eqref{eq:indicator_prod_symmetry} holds because $V_i = V(X_i)$. Then notice that
\begin{align}
    \widehat{\gC}_t\setminus \{s\} = \LRl{-n\leq j \leq t-1, j\neq s: V_j \leq \gA\LRs{\{V_{\ell}\}_{\ell\neq j,s}, V_s, V_t}},\nonumber
\end{align}
which is also symmetric to $(X_s,X_t)$, hence \eqref{eq:cal_set_symmetry} holds. Using Proposition \ref{pro:selection_exchangeable}, we can prove the conclusion.
\end{proof}

%\subsubsection{Verification of symmetric properties}

\subsection{Proof of Theorem \ref{thm:FCR_swap}}\label{proof:thm:FCR_swap}
In this section, we denote $C_{t,s} = \Indicator{V_s \leq \gA_t(\{V_j\}_{j\leq t-1, j\neq s}, V_s)}$ the selection indicator of calibration set $\widehat{\gC}_t$.
To prove Theorem \ref{thm:FCR_swap}, we introduce the following virtual decision sequence. Given each pair $(s,t)$ with $s \leq t-1$: if $s \geq 0$, we define
\begin{align*}
    S_j^{(s \gets t)} = \begin{cases}
        S_j & 0\leq j\leq s-1\\
        \Indicator{V_t \leq  \gA_s\LRs{\{V_i\}_{i\leq s-1}}} & j=s\\
        \Indicator{V_j \leq \gA_j\LRs{\{V_i\}_{i\leq j-1, i\neq s}, V_t}} & s+1\leq j \leq t-1\\
        S_j & t \leq j \leq T
    \end{cases};
\end{align*}
if $s \leq -1$, we define
\begin{align*}
    S_j^{(s \gets t)} = \begin{cases}
        \Indicator{V_j \leq \gA_j\LRs{\{V_i\}_{i\leq j-1, i\neq s}, V_t}} & 0 \leq j \leq t-1\\
        S_j & t \leq j \leq T
    \end{cases}.
\end{align*}

The following proof is used to prove Theorem \ref{thm:FCR_swap}, whose proof is deferred to Section \ref{proof:lemma:zero_gap_swap}.

\begin{lemma}\label{lemma:zero_gap_swap}
    Under the conditions of Theorem \ref{thm:FCR_swap}, it holds that
    \begin{align*}
        \E\LRm{\frac{S_t C_{t,s}}{\rmS_t(T) + 1} \frac{\Indicator{R_t > Q_{\alpha}(\{R_i\}_{i\in \widehat{\gC}_t \cup \{t\}})}}{|\widehat{\gC}_t|+1} } = \E\LRm{\frac{S_t C_{t,s}}{\rmS_t^{(s\gets t)}(T) + 1} \frac{\Indicator{R_{s} > Q_{\alpha}(\{R_i\}_{i\in \widehat{\gC}_t \cup \{t\}})}}{|\widehat{\gC}_t|+1} },
    \end{align*}
    where $\rmS_t(T) = \sum_{j=0,j\neq t}^T S_j$ and $\rmS_t^{(s\gets t)}(T) = \sum_{j=0, j\neq t}^{T} S_j^{(s \gets t)}$.
\end{lemma}

\begin{proof}[Proof of Theorem \ref{thm:FCR_swap}]
Under the event $S_t = 1$, we know $\rmS_t(T) + 1 = \sum_{j=0}^T S_j$. Using the upper bound \eqref{eq:miscover_upper}, we can get
\begin{align}\label{eq:FCR_expansion_dyn_cal}
    &\FCR(T)\nonumber\\
    &\leq \alpha + \sum_{t=0}^T\sum_{s = -n}^{t-1} \E\LRm{\frac{1}{\rmS_t(T) + 1}  \frac{S_t C_{t,s}}{|\widehat{\gC}_t|+1}  \LRs{\Indicator{R_t > Q_{\alpha}(\{R_i\}_{i\in \widehat{\gC}_t \cup \{t\}})} - \Indicator{R_{s} > Q_{\alpha}(\{R_i\}_{i\in \widehat{\gC}_t \cup \{t\}})}}}\nonumber\\
    &\Eqmark{i}{=} \alpha + \sum_{t=0}^T\sum_{s = -n}^{t-1} \E\LRm{\frac{S_t C_{t,s} \Indicator{R_{s} > Q_{\alpha}(\{R_i\}_{i\in \widehat{\gC}_t \cup \{t\}})}}{|\widehat{\gC}_t|+1} \LRl{\frac{1}{ \rmS_t^{(s\gets t)}(T) + 1} - \frac{1}{\rmS_t(T) + 1}}}\nonumber\\
    %&\leq \alpha + \sum_{t=0}^T \E\LRm{S_t \sum_{s = -n}^{t-1}\frac{ C_{t,s} \Indicator{R_{s} > Q_{\alpha}(\{R_i\}_{i\in \widehat{\gC}_t \cup \{t\}})}}{|\widehat{\gC}_t|+1}\cdot \LRl{\frac{1}{ \rmS_t^{(s\gets t)}(T) + 1} - \frac{1}{\rmS_t(T) + 1}} }\nonumber\\
    &= \alpha + \sum_{t=0}^T \E\LRm{\frac{S_t}{1 \vee \sum_{j=0}^T S_j}\sum_{s = -n}^{t-1}\frac{C_{t,s}\Indicator{R_{s} > Q_{\alpha}(\{R_i\}_{i\in \widehat{\gC}_t \cup \{t\}})}}{|\widehat{\gC}_t|+1}\cdot \frac{\sum_{j= s\vee 0}^{t-1}( S_j - S_j^{(s \gets t)})}{ \rmS^{(s\gets t)}(T) \vee 1} }\nonumber\\
    &\Eqmark{ii}{\leq} \alpha + \alpha\cdot\E\LRm{\sum_{t=0}^T\frac{S_t}{1 \vee \sum_{j=0}^T S_j} \max_{-n \leq s\leq t-1}\LRl{\frac{S_t\sum_{j= s\vee 0}^{t-1} (S_j - S_j^{(s \gets t)})}{ \rmS^{(s\gets t)}(T) \vee 1}} },
    %&\leq \alpha + \alpha\cdot\E\LRm{\max_{0\leq t \leq T}\max_{-n \leq s\leq t-1}\LRl{\frac{S_t\sum_{j= s\vee 0}^{t-1}( S_j - S_j^{(s \gets t)})}{ \rmS^{(s\gets t)}(T) \vee 1}} },
\end{align}
where $(i)$ follows from Lemma \ref{lemma:zero_gap_swap}; and $(ii)$ holds due to the definition of $Q_{\alpha}(\{R_i\}_{i\in \widehat{\gC}_t \cup \{t\}})$ such that $\frac{1}{|\widehat{\gC}_t|+1}\sum_{s=-n}^{t} C_{t,s} \Indicator{R_s > Q_{\alpha}(\{R_i\}_{i\in \widehat{\gC}_t \cup \{t\}})} \leq \alpha$.

When $s \geq 0$, let $\widehat{q}_j = F_V(\gA_j(\{V_i\}_{i\leq j-1}))$ and $\widehat{q}_j^{(s \gets t)} = F_V(\gA_j(\{V_i\}_{i\leq j-1, i\neq s}, V_t))$ for any $s+1 \leq j \leq t-1$.
Define a new filtration as $\gF_{j}^{(s)} = \sigma(\{Z_i\}_{i\leq j, i\neq s})$ for $s\leq j \leq t-2$. Then we notice that for $j=s$,
\begin{align*}
    \E\LRm{S_s - S_s^{(s \gets t)} \mid \gF_{s-1}^{(s)}} &= \E\LRm{\Indicator{V_s \leq  \gA_s(\{V_i\}_{i\leq s-1})} - \Indicator{V_t \leq  \gA_s(\{V_i\}_{i\leq s-1})} \mid \gF_{s-1}^{(s)}}\\
    &= 1 - \widehat{q}_{s} - (1 - \widehat{q}_{s}) = 0,
\end{align*}
and for any $s+1 \leq j \leq t-1$
\begin{align*}
    \E\LRm{S_j - S_j^{(s \gets t)} \mid \gF_{j-1}^{(s)}} &= \E\LRm{\E\LRm{S_j - S_j^{(s \gets t)} \mid \gF_{j-1}^{(s)}, Z_s, Z_t} \mid \gF_{j-1}^{(s)}}\\
    &= \E\Big[\sP\LRs{V_j \leq \gA_j(\{V_i\}_{i\leq j-1}) \mid \gF_{j-1}^{(s)}, Z_s, Z_t} \\
    &\qquad- \sP\LRs{V_j \leq \gA_j(\{V_i\}_{i\leq j-1,i\neq s}, V_t) \mid \gF_{j-1}^{(s)}, Z_s, Z_t}\Big]\\
    &= \E\LRm{\widehat{q}_j^{(s \gets t)} - \widehat{q}_j \mid \gF_{j-1}^{(s)}} .
\end{align*}
When $-n \leq s \leq -1$, let $\widehat{q}_j = F_V(\gA_j(\{V_i\}_{i\leq j-1}))$ and $\widehat{q}_j^{(s \gets t)} = F_V(\gA_j(\{V_i\}_{i\leq j-1, i\neq s}, V_t))$ for $0\leq j \leq t-1$. Then it holds for any $0\leq j \leq t-1$
\begin{align*}
    \E\LRm{S_j - S_j^{(s \gets t)} \mid \gF_{j-1}^{(s)}} &= \E\LRm{\E\LRm{S_j - S_j^{(s \gets t)} \mid \gF_{j-1}^{(s)}, Z_s, Z_t} \mid \gF_{j-1}^{(s)}}\\
    &= \E\LRm{\widehat{q}_j^{(s \gets t)} - \widehat{q}_j \mid \gF_{j-1}^{(s)}}.
\end{align*}
Now denote $\mu_j = \E\LRm{\widehat{q}_j^{(s \gets t)} - \widehat{q}_j \mid \gF_{j-1}^{(s)}}$ for $s+1\leq j\leq t-1$ and $\mu_s = 0$. We also write $M_j = S_j - S_j^{(s \gets t)} - \mu_j$ for $s\vee 0 \leq j \leq t-1$. Hence it holds that $\E[M_j \mid \gF_{j-1}^{(s)}] = 0$ for $s\vee 0 \leq j \leq t-1$.
%Hence $\{M_j\}_{j = s}^{t-1}$ is a martingale difference sequence adapted to the filtration $\{\gF_{j}^{(s)}\}_{j = s-1}^{t-2}$. 
In addition, when $s \geq 0$, we also have
\begin{align*}
    \E\LRm{M_s^2 \mid \gF_{s-1}^{(s)}} &= \E\LRm{S_s + S_s^{(s \gets t)} - 2 S_s S_s^{(s \gets t)} \mid \gF_{s-1}^{(s)}} = 2\widehat{q}_s (1 - \widehat{q}_s) \leq \frac{1}{2},
\end{align*}
and for any $(s+1) \vee 0 \leq j \leq t-1$,
\begin{align*}
    \E\LRm{M_j^2 \mid \gF_{j-1}^{(s)}} &\leq \E\LRm{S_j + S_j^{(s \gets t)} - 2 S_j S_j^{(s \gets t)} \mid \gF_{j-1}^{(s)}}\\
    &= \E\LRm{1- \widehat{q}_j + 1 - \widehat{q}_j^{(s \gets t)} - 2 \LRs{1 - \max\LRl{\widehat{q}_j , \widehat{q}_j^{(s \gets t)}}} \mid \gF_{j-1}^{(s)}}\\
    &= \E\LRm{\LRabs{\widehat{q}_j - \widehat{q}_j^{(s \gets t)}} \mid \gF_{j-1}^{(s)}}\\
    &= \E\LRm{\LRabs{F_V(\gA_j(\{V_i\}_{i\leq j-1})) - F_V(\gA_j(\{V_i\}_{i\leq j-1, i\neq s}, V_t))} \mid \gF_{j-1}^{(s)}}\\
    &\leq \rho \E\LRm{\LRabs{\gA_j(\{V_i\}_{i\leq j-1}) - \gA_j(\{V_i\}_{i\leq j-1, i\neq s}, V_t)} \mid \gF_{j-1}^{(s)}}\\
    &\leq \rho\sigma_j,
\end{align*}
where the last two inequalities hold since the density of $V_i$ is bounded by $\rho$ and the definition of $\sigma_j$ in Assumption \ref{assum:swap_sensitivity}. For $(s+1) \vee 0 \leq j \leq t-1$, it follows that for any $\lambda > 0$,
\begin{align}\label{eq:Sj_diff_MGF}
    \E\LRm{e^{\lambda M_j} \mid \gF_{j-1}^{(s)}} &\leq 1 + \E\LRm{\lambda M_j + \lambda^2 M_j^2 e^{\lambda |M_j|} \mid \gF_{j-1}^{(s)}}\nonumber\\
    &= 1 + \lambda^2\E\LRm{M_j^2 e^{\lambda |M_j|} \mid \gF_{j-1}^{(s)}}\nonumber\\
    %&= 1 + \lambda^2 e^{\lambda}\sP\LRs{S_j \neq S_j^{(s \gets t)} \mid \gF_{j-1}^{(s)}}\nonumber\\
    &\leq 1 + \lambda^2 e^{2\lambda} \E\LRm{M_j^2 \mid \gF_{j-1}^{(s)}}\nonumber\\
    &\leq 1 + \lambda^2 e^{2\lambda} \rho \sigma_j\nonumber\\
    &\leq \exp\LRs{\lambda^2 e^{2\lambda} \rho \sigma_j},
\end{align}
where the first inequality holds due to the basic inequality $e^y \leq 1 + y + y^2 e^{|y|}$ for any $y\in \sR$. Now let
\begin{align*}
    W_{\ell} = \exp\LRl{\lambda\sum_{j=s\vee 0}^{\ell} M_j - \lambda^2 e^{2\lambda} \rho \LRs{2^{-1}+\sum_{j=s\vee 0}^{\ell} \sigma_j}},\quad \text{for }s\vee 0\leq \ell \leq t-1.
\end{align*}
Invoking \eqref{eq:Sj_diff_MGF}, for $(s+1)\vee 0 \leq \ell \leq t-1$ we have
\begin{align*}
    \E\LRm{W_{\ell} \mid \gF_{\ell - 1}^{(s)}} &= W_{\ell - 1}\E\LRm{\exp\LRl{\lambda M_j - \lambda^2 e^{\lambda}\rho \sigma_{\ell}} \mid \gF_{\ell - 1}^{(s)}}\leq W_{\ell - 1},
\end{align*}
which yields $\E\LRm{W_{t-1} \mid \gF_{s-1}^{(s)}} \leq \cdots \leq \E\LRm{W_{s} \mid \gF_{s-1}^{(s)}} \leq 1$ for $s \geq 0$ and $\E\LRm{W_{t-1} \mid \gF_{s-1}^{(s)}} \leq \cdots \leq \E\LRm{W_{0} \mid \gF_{-1}^{(s)}} \leq 1$. Applying Markov's inequality, for any $\delta > 0$, we have
\begin{align}
    &\sP\LRl{\sum_{j= s\vee 0}^{t-1} M_j \leq \LRs{2^{-1}+\lambda e^{\lambda} \rho \sum_{j= s\vee 0}^{t-1} \sigma_j} + \frac{\log(1/\delta)}{\lambda}}\nonumber\\
    &\qquad= \sP\LRm{\exp\LRl{\lambda\sum_{j= s\vee 0}^{t-1} M_j - \lambda^2 e^{\lambda} \rho \LRs{2^{-1} + \sum_{j= s\vee 0}^{t-1} \sigma_j}} > \frac{1}{\delta}}\nonumber\\
    &\qquad = \sP\LRs{W_{t-1} > \frac{1}{\delta}}\nonumber\\
    &\qquad \leq \delta\cdot \E[W_{t-1}]\nonumber\\
    &\qquad \leq \delta.\nonumber
\end{align}
Now we take $\lambda = \min\LRl{\frac{1}{\sqrt{e}\rho}, 1}$, which means $(\lambda^2e^{\lambda}+1) \rho \leq \lambda^2 e \rho + \rho \leq \rho^{-1} + \rho \leq 2$. Let $\epsilon(t) = 2\sum_{j=0}^{t-1} \sigma_j + \LRs{\sqrt{e}\rho + 1}\log(1/\delta) + 2^{-1}$. Together with the fact $|\mu_j| \leq \rho \sigma_j$, we have
\begin{align}\label{eq:hp_bound_1}
    \sP\LRl{\LRabs{\sum_{j= s\vee 0}^{t-1} S_j - S_j^{(s \gets t)}} \leq \epsilon(t)} \geq 1-2\delta,
\end{align}
and
\begin{align}\label{eq:hp_bpund_2}
    \sP\LRl{\rmS^{(s \gets t)}(T) \geq \sum_{j=0}^T S_j - \epsilon(t)} \geq 1-\delta.
\end{align}
Define the good event $\gE_{t,s} = \{\text{the events in \eqref{eq:hp_bound_1} and \eqref{eq:hp_bpund_2} happen}\}$. In conjunction with \eqref{eq:FCR_expansion_dyn_cal}, we have
\begin{align}
    \FCR(T) &\leq \alpha + \alpha\cdot \E\LRm{\sum_{t=0}^T\frac{S_t}{1 \vee \sum_{j=0}^T S_j} \max_{s\leq t-1} \LRl{\LRs{\Indicator{\gE_{t,s}} + \Indicator{\gE_{t,s}^{c}}} \frac{\LRabs{\sum_{j= s\vee 0}^{t-1} S_j - S_j^{(s \gets t)}}}{ \rmS^{(s\gets t)}(T) \vee 1}}}\nonumber\\
    %&= \alpha  + \alpha \cdot\E\LRm{\max_{s,t} \LRl{\LRs{\Indicator{\gE_{t,s}} + \Indicator{\gE_{t,s}^{c}}} \frac{S_t\LRabs{\sum_{j= s\vee 0}^{t-1} S_j - S_j^{(s \gets t)}}}{ \rmS^{(s\gets t)}(T) \vee 1}}}\nonumber\\
    &\Eqmark{i}{\leq} \alpha  + \alpha \cdot\E\LRm{\max_{s,t} \LRl{\frac{S_t \epsilon(t)}{\LRs{\sum_{j=0}^T S_j - \epsilon(t)}\vee 1} + \Indicator{\gE_{t,s}^{c}}}}\nonumber\\
    &\Eqmark{ii}{\leq} \alpha  + \alpha \cdot\E\LRm{\frac{\Indicator{\sum_{j=0}^T S_j > 0}\epsilon(t)}{\LRs{\sum_{j=0}^T S_j -\epsilon(t) }\vee 1}} + \E\LRm{\max_{s,t} \Indicator{\gE_{t,s}^{c}}},\nonumber\\
    &\Eqmark{iii}{\leq} \alpha  + \alpha \cdot\E\LRm{\frac{\Indicator{\sum_{j=0}^T S_j > 0}\epsilon(t)}{\LRs{\sum_{j=0}^T S_j -\epsilon(t) }\vee 1}} + 3(T+n+1)^2\delta,\nonumber
\end{align}
where $(i)$ holds due to the definition of $\gE_{t,s}$; $(ii)$ follows from $\max_t S_t = \Indicator{\sum_{j=0}^T S_j > 0}$; \eqref{eq:hp_bound_1}, \eqref{eq:hp_bpund_2} and union's bound. Taking $\delta = (T+n+1)^{-3}$ can prove the desired bound.
\end{proof}

\subsection{Proof of Theorem \ref{thm:FCR_swap_window}}
\begin{proof}
    Notice that, Lemma \ref{lemma:zero_gap_swap} still holds.
    Following the notations in Section \ref{proof:thm:FCR_swap}, we can expand $\FCR$ by
    \begin{align}\label{eq:FCR_upper_window}
    &\FCR(T)\nonumber\\
    &\leq \alpha + \sum_{t=0}^T\sum_{s = t-n}^{t-1} \E\LRm{\frac{1}{\rmS_t(T) + 1}  \frac{S_t C_{t,s}}{|\widehat{\gC}_t|+1} \LRs{\Indicator{R_t > Q_{\alpha}(\{R_i\}_{i\in \widehat{\gC}_t \cup \{t\}})} - \Indicator{R_{s} > Q_{\alpha}(\{R_i\}_{i\in \widehat{\gC}_t \cup \{t\}})}}}\nonumber\\
    &= \alpha + \sum_{t=0}^T\sum_{s = t-n}^{t-1} \E\LRm{\LRl{\frac{1}{ \rmS_t^{(s\gets t)}(T) + 1} - \frac{1}{\rmS_t(T) + 1}}\frac{S_t C_{t,s} \Indicator{R_{s} > Q_{\alpha}(\{R_i\}_{i\in \widehat{\gC}_t \cup \{t\}})}}{|\widehat{\gC}_t|+1} }\nonumber\\
    &\leq \alpha + \sum_{t=0}^T \E\LRm{S_t \sum_{s = t-n}^{t-1}\frac{ C_{t,s} \Indicator{R_{s} > Q_{\alpha}(\{R_i\}_{i\in \widehat{\gC}_t \cup \{t\}})}}{|\widehat{\gC}_t|+1}\cdot \LRl{\frac{1}{ \rmS_t^{(s\gets t)}(T) + 1} - \frac{1}{\rmS_t(T) + 1}} }\nonumber\\
    &= \alpha + \sum_{t=0}^T \E\LRm{\frac{S_t}{1 \vee \sum_{j=0}^T S_j}\sum_{s = t-n}^{t-1}\frac{C_{t,s}\Indicator{R_{s} > Q_{\alpha}(\{R_i\}_{i\in \widehat{\gC}_t \cup \{t\}})}}{|\widehat{\gC}_t|+1}\cdot \frac{\sum_{j= s\vee 0}^{t-1} S_j - S_j^{(s \gets t)}}{ \rmS^{(s\gets t)}(T) \vee 1} }\nonumber\\
    &\leq \alpha + \alpha\cdot\E\LRm{\sum_{t=0}^T\frac{S_t}{1 \vee \sum_{j=0}^T S_j} \max_{t-n \leq s\leq t-1}\LRl{\frac{\sum_{j= s\vee 0}^{t-1} S_j - S_j^{(s \gets t)}}{ \rmS^{(s\gets t)}(T) \vee 1}} },
    %&\leq \alpha + \alpha\cdot\E\LRm{\max_{0\leq t \leq T}\max_{t-n \leq s\leq t-1}\LRl{\frac{\sum_{j= s\vee 0}^{t-1} S_j - S_j^{(s \gets t)}}{ \rmS^{(s\gets t)}(T) \vee 1}} }.
\end{align}
Let $\epsilon_n(t) = 2\sum_{j=(t-n) \vee 0}^{t-1} \sigma_j + (\sqrt{e\rho} + 1)\log(1/\delta) + 2^{-1}$.
Similar to \eqref{eq:hp_bound_1} and \eqref{eq:hp_bpund_2}, we can show
\begin{align}
    \sP\LRl{\LRabs{\sum_{j= s\vee 0}^{t-1} S_j - S_j^{(s \gets t)}} \leq \epsilon_n(t)} \geq 1-2\delta,\nonumber
\end{align}
and
\begin{align}
    \sP\LRl{\rmS^{(s \gets t)}(T) \geq \sum_{j=0}^T S_j - \epsilon_n(t)} \geq 1-\delta.\nonumber
\end{align}
Then taking $\delta = (T \vee n)^{-3}$, together with \eqref{eq:FCR_upper_window}, we can guarantee
\begin{align*}
    \FCR(T) 
    %&\leq \alpha + \alpha \cdot\E\LRm{\sum_{t=0}^T\frac{S_t}{1 \vee \sum_{j=0}^T S_j} \max_{t-n \leq s\leq t-1}\LRl{\LRs{\Indicator{\gE_{t,s}} + \Indicator{\gE_{t,s}^{c}}} \frac{\LRabs{\sum_{j= s\vee 0}^{t-1} S_j - S_j^{(s \gets t)}}}{ \rmS^{(s\gets t)}(T) \vee 1}}}\nonumber\\
    &\leq \alpha \cdot \LRs{1 + \frac{3}{T\vee n} + \E\LRm{\max_{0\leq t \leq T}\frac{S_t \epsilon_n(t) }{\LRl{\sum_{j=0}^T S_j - \epsilon_n(t)} \vee 1}}}.
\end{align*}
\end{proof}

\subsection{Proof of Lemma \ref{lemma:zero_gap_swap}}\label{proof:lemma:zero_gap_swap}
\begin{proof}
Denote $\gC_{t,+} = \{s\leq t: C_{t,s} = 1\}$ with $C_{t,t} \equiv S_t$ and let $Q_{1-\alpha}\LRs{\frac{1}{|\gC_{t,+}|}\sum_{i\in \gC_{t,+}}\delta_{R_i} }$ be the $(1-\alpha)$-quantile of the empirical distribution $\frac{1}{|\gC_{t,+}|}\sum_{i\in \gC_{t,+}}\delta_{R_i}$, where $\delta_{R_i}$ is the point mass function at $R_i$. Because $\gC_{t,+} = \widehat{\gC}_t \cup \{t\}$ holds under the event $ S_t = 1$, it suffices to show
\begin{align}\label{eq:zero_gap_Ss_new}
    &\E\LRm{\frac{C_{t,t} C_{t,s}}{\rmS_t(T) + 1} \frac{\Indicator{R_t > Q_{\alpha}(\{R_i\}_{i\in \widehat{\gC}_t \cup \{t\}})}}{|\gC_{t,+}|} } = \E\LRm{\frac{C_{t,t} C_{t,s}}{\rmS_t^{(s\gets t)}(T) + 1} \frac{\Indicator{R_{s} > Q_{\alpha}(\{R_i\}_{i\in \widehat{\gC}_t \cup \{t\}})}}{|\gC_{t,+}|}  }.
\end{align}
We define the event
\begin{align*}
    \gE(\rz) = \LRl{[Z_{s},Z_{t}] = \rz} = \LRl{\LRm{Z_s, Z_t} = \LRm{z_{1},z_{2}}},
\end{align*}
where $[Z_{s},Z_{t}]$ and $\rz = [z_{1},z_{2}]$ are both unordered sets. Under $\gE(\rz)$, define the random indexes $I_{t},I_s \in \{1,2\}$ such that $Z_t = z_{I_t}$ and $Z_{s} = z_{I_s}$. Notice that $[V_s, V_t]$ and $[R_s, R_t]$ are fixed under the event $\gE(\rz)$, we denote the corresponding observations as $[v_{1},v_{2}]$ and $[r_{1}, r_2]$. Then we know
\begin{align}
    C_{t,s} \mid \gE(\rz) &= \Indicator{V_t \leq  \gA_t\LRs{V_t, \{V_{i}\}_{i\leq t-1, i\neq s}}} \mid \gE(\rz)\nonumber\\
    &= \begin{cases}
    \Indicator{v_{1} > \gA_t(v_2, \{V_{i}\}_{i\leq t-1, i\neq s })}, & I_s = 1\\
    \Indicator{v_{2} > \gA_t(v_1, \{V_{i}\}_{i\leq t-1, i\neq s })}, & I_s = 2
    \end{cases},\nonumber
\end{align}
and
\begin{align}
    C_{t,t} \mid \gE(\rz) &= \Indicator{V_t \leq  \gA_t\LRs{V_{s}, \{V_{i}\}_{i\leq t-1, i\neq s }}} \mid \gE(\rz)\nonumber\\
    &= \begin{cases}
    \Indicator{v_{1} > \gA_t(v_2, \{V_{i}\}_{i\leq t-1, i\neq s})}, & I_t = 1\\
    \Indicator{v_{2} > \gA_t(v_1, \{V_{i}\}_{i\leq t-1, i\neq s})}, & I_t = 2
    \end{cases}.\nonumber
\end{align}
It follows that
\begin{align}
    C_{t,s}C_{t,t} \mid \gE(\rz) = \Indicator{v_{1} > \gA_t(v_2, \{V_{i}\}_{i\leq t-1, i\neq s})} \Indicator{v_{2} > \gA_t(v_1, \{V_{i}\}_{i\leq t-1, i\neq s})},\nonumber
\end{align}
which is fixed given $\sigma(\{Z_i\}_{i \neq s,t})$. Further, $C_{t,j} = \Indicator{V_{j} > \gA_t(\{V_i\}_{i\leq t-1,i\neq j,s},v_1,v_2)}$ is also fixed for any $j\neq s,t$ given $\gE(\rz)$ and $\sigma(\{Z_i\}_{i \neq s,t})$ since $\gA_t$ is symmetric. Hence, the unordered set $[\{C_{t,i}\delta_{R_i}\}_{i\leq t}]$ is fixed, as well as $|\gC_{t,+}| = \sum_{i\leq t} C_{t,i}$. As a consequence, we can write 
\begin{align}\label{eq:product_fixed}
    \frac{C_{t,s}C_{t,t}}{|\gC_{t,+}|} \mid \gE(\rz), \{Z_i\}_{i\neq s,t} =: F(\rz, \{Z_i\}_{i \neq s,t}),
\end{align}
and
\begin{align*}
    Q_{\alpha}(\{R_i\}_{i\in \widehat{\gC}_t \cup \{t\}}) \mid \gE(\rz) = Q_{1-\alpha}\LRs{\frac{1}{|\gC_{t,+}|}\sum_{i \leq t} C_{t,i} \delta_{R_i} } \mid \gE(\rz) =: Q(\rz, \{Z_i\}_{i \neq s,t}).
\end{align*}
Then we can write
\begin{align}
    &\Indicator{R_s > Q_{\alpha}(\{R_i\}_{i\in \widehat{\gC}_t \cup \{t\}})} \mid \gE(\rz) = \Indicator{r_{I_s} > Q(\rz,\{Z_i\}_{i \neq s,t})},\label{eq:Rs_exceed_under_Ez_new}\\ 
    &\Indicator{R_t > Q_{\alpha}(\{R_i\}_{i\in \widehat{\gC}_t \cup \{t\}})} \mid \gE(\rz) = \Indicator{r_{I_t} > Q(\rz,\{Z_i\}_{i \neq s,t})}.\label{eq:Rt_exceed_under_Ez_new}
\end{align}
In addition, it holds that
\begin{align}\label{eq:vitual_Ss_under_Ez_new}
    \sum_{j = s+1}^{t-1} S_j^{(s \gets t)} \mid \gE(z) &= \sum_{j = s+1}^{t-1}\Indicator{V_j \leq \gA_j\LRs{\{V_{i}\}_{i\leq j-1,i\neq s}, V_t}} \mid \gE(z)\nonumber\\
    &= \sum_{j = s+1}^{t-1}\Indicator{V_j \leq \gA_j\LRs{\{V_{i}\}_{i\leq j-1,i\neq s}, v_{I_t}}},
\end{align}
which is a function of $I_t$ given $\sigma(\{Z_i\}_{i \neq s,t})$. Similarly, we also have
\begin{align}\label{eq:vitual_St_under_Ez_new}
    \sum_{j = s+1}^{t-1} S_j \mid \gE(z) &= \sum_{j = s+1}^{t-1}\Indicator{V_j \leq \gA_j\LRs{\{V_{i}\}_{i\leq j-1,i\neq s}, V_s}} \mid \gE(z)\nonumber\\
    &= \sum_{j = s+1}^{t-1}\Indicator{V_j \leq \gA_j\LRs{\{V_{i}\}_{i\leq j-1,i\neq s}, v_{I_s}}},
\end{align}
which is a function of $I_s$ given $\sigma(\{Z_i\}_{i \neq s,t})$.
In addition, notice that
\begin{align}\label{eq:virtual_Ss}
    S_s^{(s \gets t)} \mid \gE(\rz) = \Indicator{V_t \leq  \gA_s\LRs{\{V_i\}_{i\leq s-1}}} \mid \gE(\rz) = \Indicator{v_{I_t} > \gA_s\LRs{\{V_i\}_{i\leq s-1}}},
\end{align}
and
\begin{align}\label{eq:real_Ss}
    S_s \mid \gE(\rz) = \Indicator{V_t \leq  \gA_s\LRs{\{V_i\}_{i\leq s-1}}} \mid \gE(\rz) = \Indicator{v_{I_s} > \gA_s\LRs{\{V_i\}_{i\leq s-1}}}.
\end{align}
Now define $$\mathrm{S}_{s:(t-1)}(v_k; \{Z_i\}_{i\neq s,t}) := \Indicator{v_{k} > \gA_s\LRs{\{V_i\}_{i\leq s-1}}}+ \sum_{j = s+1}^{t-1}\Indicator{V_j \leq \gA_j\LRs{\{V_{i}\}_{i\neq j, s}, v_{k}}}$$ for $k=1,2$. From \eqref{eq:vitual_Ss_under_Ez_new}--\eqref{eq:real_Ss}, we can write
\begin{align}
    &\sum_{j= s\vee 0}^{t-1} S_j \mid \gE(\rz), \{Z_i\}_{i\neq s,t} = \mathrm{S}_{s:(t-1)}(v_{I_s}; \{Z_i\}_{i\neq s,t}),\label{eq:sum_Sj_real}\\
    &\sum_{j= s\vee 0}^{t-1} S_j^{(s\gets t)} \mid \gE(\rz), \{Z_i\}_{i\neq s,t} = \mathrm{S}_{s:(t-1)}(v_{I_t}; \{Z_i\}_{i\neq s,t})\label{eq:sum_Sj_virtual}.
\end{align}
For any $j\leq s-1$, $S_j$ is fixed given $\{Z_i\}_{i\leq s-1}$. And for any $j \geq t+1$, $S_j$ is fixed given $\{Z_i\}_{i\neq s,t}$ and $\gE(\rz)$ since $\gA_j(\cdot)$ is symmetric. Therefore, we can write
\begin{align}\label{eq:sum_Sj_independent_st}
    \sum_{j=0}^{s-1} S_j + \sum_{j= t+1}^{T} S_j \mid \gE(\rz), \{Z_i\}_{i\neq s,t} =: \mathrm{S}(\rz, \{Z_i\}_{i\neq s,t}),
\end{align}
where $\sum_{j=0}^{s-1} S_j = 0$ if $s \leq 0$.

Now using \eqref{eq:product_fixed}, \eqref{eq:Rs_exceed_under_Ez_new}, \eqref{eq:Rt_exceed_under_Ez_new}, \eqref{eq:sum_Sj_real}, \eqref{eq:sum_Sj_virtual} and \eqref{eq:sum_Sj_independent_st}, we can have  
\begin{align}
    &\E\LRm{\frac{1}{\rmS_t(T) + 1} \frac{C_{t,t} C_{t,s}}{|\gC_{t,+}|} \Indicator{R_t > Q_{\alpha}(\{R_i\}_{i\in \widehat{\gC}_t \cup \{t\}})}\mid \{Z_i\}_{i\neq s, t}, \gE(\rz)}\nonumber\\
    &\qquad \Eqmark{i}{=}  \E\LRm{\frac{C_{t,t} C_{t,s}}{|\gC_{t,+}|}\frac{\Indicator{R_t > Q_{\alpha}(\{R_i\}_{i\in \widehat{\gC}_t \cup \{t\}})}}{\sum_{j=0}^{s-1} S_j + \sum_{j= s\vee 0}^{t-1} S_j + \sum_{j= t+1}^{T} S_j + 1}  \mid \{Z_i\}_{i\neq s, t}, \gE(\rz)}\nonumber\\
    &\qquad \Eqmark{ii}{=} F(\rz, \{Z_i\}_{i \neq s,t})\cdot \E\LRm{\frac{\Indicator{r_{I_t} > Q(\rz,\{Z_i\}_{i \neq s,t})}}{\mathrm{S}(\rz, \{Z_i\}_{i\neq s,t}) + \mathrm{S}_{s:(t-1)}(v_{I_s}) + 1} \mid \{Z_i\}_{i\neq s, t}}\nonumber\\
    &\qquad \Eqmark{iii}{=} F(\rz, \{Z_i\}_{i \neq s,t})\cdot\Bigg[ \frac{\Indicator{r_{1} > Q(\rz,\{Z_i\}_{i \neq s,t})}}{\mathrm{S}(\rz, \{Z_i\}_{i\neq s,t}) + \mathrm{S}_{s:(t-1)}(v_{2}) + 1}\cdot \sP\LRs{I_t = 1}\nonumber\\
    &\qquad\qquad\qquad\qquad\qquad + \frac{\Indicator{r_{2} > Q(\rz,\{Z_i\}_{i \neq s,t})}}{\mathrm{S}(\rz, \{Z_i\}_{i\neq s,t}) + \mathrm{S}_{s:(t-1)}(v_{1}) + 1}\cdot \sP\LRs{I_s = 1}\Bigg]\nonumber\\
    &\qquad \Eqmark{iv}{=} F(\rz, \{Z_i\}_{i \neq s,t})\cdot\Bigg[ \frac{\Indicator{r_{1} > Q(\rz,\{Z_i\}_{i \neq s,t})}}{\mathrm{S}(\rz, \{Z_i\}_{i\neq s,t}) + \mathrm{S}_{s:(t-1)}(v_{2}) + 1} \cdot\sP\LRs{I_s = 1}\nonumber\\
    &\qquad\qquad\qquad\qquad\qquad + \frac{\Indicator{r_{2} > Q(\rz,\{Z_i\}_{i \neq s,t})}}{\mathrm{S}(\rz, \{Z_i\}_{i\neq s,t}) + \mathrm{S}_{s:(t-1)}(v_{1}) + 1}\cdot \sP\LRs{I_t = 1} \Bigg]\nonumber\\
    %&\qquad = \E\LRm{F(\rz, \{Z_i\}_{i \neq s,t})\cdot \frac{\Indicator{r_{I_s} > Q(\rz,\{Z_i\}_{i \neq s,t})}}{\sum_{j=0}^{s-1} S_j + \mathrm{S}_{s:(t-1)}(v_{I_t})  + \sum_{j= t+1}^{T} S_j + 1} \mid \{Z_i\}_{i\neq s, t}}\nonumber\\
    &\qquad = F(\rz, \{Z_i\}_{i \neq s,t})\cdot \E\LRm{\frac{\Indicator{r_{I_s} > Q(\rz,\{Z_i\}_{i \neq s,t})}}{\mathrm{S}(\rz, \{Z_i\}_{i\neq s,t}) + \mathrm{S}_{s:(t-1)}(v_{I_t}) + 1} \mid \{Z_i\}_{i\neq s, t}}\nonumber\\
    &\qquad = \E\LRm{\frac{1}{\rmS_t^{(s \gets t)}(T) + 1} \frac{C_{t,t} C_{t,s}}{|\gC_{t,+}|} \Indicator{R_s > Q_{\alpha}(\{R_i\}_{i\in \widehat{\gC}_t \cup \{t\}})}\mid \{Z_i\}_{i\neq s, t}, \gE(\rz)}.\nonumber
\end{align}
where $(i)$ holds due to $S_t \equiv C_{t,t}$; and $(ii)$ follows from \eqref{eq:product_fixed}; $(iii)$ holds because $(I_s,I_t) \independent \sigma(\{Z_i\}_{i\neq s,t})$; and $(iv)$ holds due to exchangeability between $Z_s$ and $Z_t$ such that $\sP(I_s = 1) = \sP(I_t = 1)$. Then we can verify \eqref{eq:zero_gap_Ss_new} by marginalizing over $\gE(\rz)$ and the tower's rule.
\end{proof}

\subsection{Proof of Proposition \ref{pro:mean_selection}}
\begin{proof}
    In this case, $\gA_j(\{V_i\}_{i\leq j-1}) = \frac{1}{n+j}\sum_{i=-n}^{j-1} V_i$ and $\gA_j(\{V_i\}_{i\leq j-1,i\neq s},V_t) = \frac{1}{n+j}\sum_{i=-n}^{j-1} V_i + \frac{V_t - V_s}{n+j}$. It follows that
    \begin{align*}
        \E\LRm{\LRabs{\gA_j(\{V_i\}_{i\leq j-1}) - \gA_j(\{V_i\}_{i\leq j-1,i\neq s},V_t)} \mid \{V_i\}_{i\leq j-1,i\neq s}} &= \E\LRm{\LRabs{\frac{V_t - V_s}{n+j}}}\leq \frac{2\sigma}{n+j}.
    \end{align*}
    Now let $\sigma_j = 2\sigma/(n+j)$ for $j\geq 0$.
    Recall the definition of $\epsilon(t)$ in \eqref{eq:swap_FCR_bound}, we have
    \begin{align*}
        \epsilon(t) &= 2\sum_{j=0}^{T-1}\sigma_j + 3(\sqrt{e\rho} + 1)\log(T+n)\\
        &= \sum_{j=0}^{T-1}\frac{4\sigma}{n+j} + 3(\sqrt{e\rho} + 1)\log(T+n)\\
        &\leq 4\sigma \log(T+n) + 3(\sqrt{e\rho} + 1)\log(T+n)\\
        &\leq 4(\sqrt{e\rho} + \sigma + 1)\log(T+n).
    \end{align*}
    The proof is finished.
\end{proof}

\subsection{Proof of Proposition \ref{pro:quantile_selection}}
\begin{lemma}\label{lemma:order_swap_one}
    For almost surely distinct random variables $x_1,...,x_n,x_{n+1}$, let $\{x_{(r)}: r\in [n]\}$ be the $r$-th smallest value in $\{x_i: i\in [n]\}$, and $\{x_{(r)}^{j \gets (n+1)}: r\in [n]\}$ be the $r$-th smallest value in $\{x_i: i\in [n]\setminus\{j\}\} \cup \{x_{n+1}\}$. Then for any $r\in [n]$ and $j\in [n]$, we have
    \begin{align*}
        \LRabs{x_{(r)}^{j \gets (n+1)} - x_{(r)}} \leq \max\LRl{x_{(r)} - x_{(r-1)}, x_{(r+1)} - x_{(r)}}.
    \end{align*}
\end{lemma}
\begin{proof}
    If $x_j > x_{(r)}$ and $x_{n+1} > x_{(r)}$ or $x_j < x_{(r)}$ and $x_{n+1} < x_{(r)}$, it is easy to see $x_{(r)}^{j\gets (n+1)} = x_{(r)}$. If $x_j < x_{(r)}$ and $x_{n+1} > x_{(r)}$, we know $x_{(r)}^{j\gets (n+1)} = \min\{x_{(r+1)}, x_{n+1}\}$, which means $x_{(r)}^{j\gets (n+1)} - x_{(r)} \leq x_{(r+1)} - x_{(r)}$. If $x_j > x_{(r)}$ and $x_{n+1} < x_{(r)}$, we know $x_{(r)}^{j\gets (n+1)} = \max\{x_{(r-1)}, x_{n+1}\}$, so $x_{(r)}^{j\gets (n+1)} - x_{(r)} \geq x_{(r-1)} - x_{(r)}$.
\end{proof}
\begin{lemma}\label{lemma:order_drop_one}
For almost surely distinct random variables $x_1,...,x_n$, let $\{x_{(r)}: r\in [n]\}$ be the $r$-th smallest value in $\{x_i: i\in [n]\}$, and $\{x_{(r)}^{[n]\setminus\{j\}}: r\in [n-1]\}$ be the $r$-th smallest value in $\{x_i: i\in [n]\setminus\{j\}\}$, then for any $r \in [n-1]$ we have: $x_{(r)}^{[n]\setminus\{j\}} = x_{(r)}$ if $x_j > x_{(r)}$ and $x_{(r)}^{[n]\setminus\{j\}} = x_{(r+1)}$ if $x_j \leq x_{(r)}$.
\end{lemma}
\begin{proof}
    The conclusion is trivial.
\end{proof}
\begin{lemma}[Lemma 3 in \citet{bao2023selective}]\label{lemma:uniform_spacing}
    Let $U_1,\cdots ,U_n \stackrel{i.i.d.}{\sim}\operatorname{Uniform}([0,1])$, and $U_{(1)}\leq U_{(2)}\leq \cdots \leq U_{(n)}$ be their order statistics. For any $\delta \in (0,1)$, it holds that
    \begin{align}\label{eq:max_spacing}
        \sP\LRs{\max_{0\leq \ell \leq n-1}\LRl{U_{(\ell+1)} - U_{(\ell)}} \geq \frac{1}{1 - 2\sqrt{\frac{\log \delta}{n+1}}}\frac{2\log \delta}{n+1}} \leq 2 \delta.
    \end{align}
\end{lemma}
\begin{proof}[Proof of Proposition \ref{pro:quantile_selection}]
    Let $F_v(\cdot)$ be the c.d.f. of $\{V_i\}_{i\geq -n}$. If $\gA_j$ takes the quantile of $\{V_i\}_{i\leq j-1}$ for $j\geq 0$, then notice that
    \begin{align*}
        \Indicator{V_j \leq \gA_j\LRs{\{V_i\}_{i\leq j-1}}} = \Indicator{F_v(V_j) \leq \gA_j\LRs{\{F_v(V_i)\}_{i\leq j-1}}}.
    \end{align*}
    Without loss of generality, we assume $V_i \stackrel{i.i.d.}{\sim} \operatorname{Uniform}([0,1])$. Denote $\ermV_{(r)}^{[n+j]}$ and $\ermV_{(r)}^{[n+j]\setminus s}$ the $r$-th smallest values in $\{V_i\}_{i\leq j-1}$ and $\{V_i\}_{i\leq j-1,i\neq s}$ respectively. Then we have
    \begin{align*}
        &\LRabs{\gA_j\LRs{\{V_i\}_{i\leq j-1}} - \gA_j\LRs{\{V_i\}_{i\leq j-1,i\neq s}, V_t}}\\
        &\qquad\leq \max\LRl{\ermV_{(\lceil \beta (n+j)\rceil)}^{[n+j]} - \ermV_{(\lceil \beta (n+j)\rceil-1)}^{[n+j]}, \ermV_{(\lceil \beta (n+j)\rceil+1)}^{[n+j]} - \ermV_{(\lceil \beta (n+j)\rceil)}^{[n+j]}}\\
        &\qquad\leq \max\LRl{\ermV_{(\lceil \beta (n+j)\rceil)}^{[n+j]\setminus s} - \ermV_{(\lceil \beta (n+j)\rceil-1)}^{[n+j]\setminus s}, \ermV_{(\lceil \beta (n+j)\rceil+1)}^{[n+j]\setminus s} - \ermV_{(\lceil \beta (n+j)\rceil)}^{[n+j]\setminus s}}\\
        &\qquad=: \sigma_j,
    \end{align*}
    where the first inequality follows from Lemma \ref{lemma:order_swap_one}; and the second inequality follows from Lemma \ref{lemma:order_drop_one}. Invoking Lemma \ref{lemma:uniform_spacing}, we can guarantee that for any $\delta_j \in (0,1)$,
    \begin{align*}
        \sP\LRs{\sigma_j > \frac{1}{1 - 2\sqrt{\frac{\log(1/\delta_j)}{n+j}}} \frac{2\log\delta_j}{n+j}} \leq 2\delta_j.
    \end{align*}
    Taking $\delta_j = (n+j)^{-3}$ and applying union's bound, we have
    \begin{align*}
        \sP\LRs{\bigcup_{0\leq j \leq T-1} \LRl{\sigma_j > \frac{1}{1 - 2\sqrt{\frac{3\log(n+j)}{n+j}}} \frac{6\log(n+j)}{n+j}}}\leq \sum_{j=0}^{T-1} (n+j)^{-3} \leq (T+n)^{-2}.
    \end{align*}
    Then with probability at least $1 - (T+n)^{-2}$, it holds that
    \begin{align*}
        \epsilon(t) &= 2\sum_{j=0}^{T-1}\sigma_j + 3(\sqrt{e} + 1)\log(T+n)\nonumber\\
        &\leq \sum_{j=0}^{T-1} \frac{2}{1 - 2\sqrt{\frac{3\log(n+j)}{n+j}}} \frac{6\log(n+j)}{n+j} + 3(\sqrt{e} + 1)\log(T+n)\\
        &\Eqmark{i}{\leq} \sum_{j=0}^{T-1} \frac{24\log(n+j)}{n+j} + 3(\sqrt{e} + 1)\log(T+n)\\
        &\leq \sum_{j=0}^{T-1} \frac{24\log(T+n)}{n+j} + 3(\sqrt{e} + 1)\log(T+n)\\
        &\Eqmark{ii}{\leq} 24\log^2(T+n) + 3(\sqrt{e} + 1)\log(T+n),
    \end{align*}
    where $(i)$ holds due to the assumption $48\log n \leq n$ and $n\geq 3$ (the function $\log x/x$ is decreasing on $[3,+\infty)$); $(ii)$ follows from $\sum_{j=0}^{T-1}\frac{1}{n+j} \leq \int_{n-1}^{T+n-1} \frac{1}{x} dx \leq \log(T+n)$.
\end{proof}

\subsection{Proof of Theorem \ref{thm:FCR_cal_selection}}
The following lemma is parallel to Lemma \ref{lemma:zero_gap_swap}, which can be proved in similar arguments in Section \ref{proof:lemma:zero_gap_swap}.
\begin{lemma}\label{lemma:zero_gap_st}
    Under the conditions of Theorem \ref{thm:FCR_cal_selection}, the following relation holds:
    \begin{align*}
        \E\LRm{\frac{S_t C_{t,s}}{\rmS_t(T) + 1} \frac{\Indicator{R_{t} > Q_{\alpha}(\{R_i\}_{i\in \widehat{\gC}_t \cup \{t\}})} }{|\widehat{\gC}_t|+1} } = \E\LRm{\frac{S_t C_{t,s}}{\rmS_t^{(s\gets t)}(T) + 1} \frac{\Indicator{R_s > Q_{\alpha}(\{R_i\}_{i\in \widehat{\gC}_t \cup \{t\}})}}{|\widehat{\gC}_t|+1}  }.
    \end{align*}
\end{lemma}
\begin{proof}[Proof of Theorem \ref{thm:FCR_cal_selection}]
For $s\leq t-1$, we denote $C_{t,s} = \Indicator{V_{s} > \gA\LRs{\{V_i\}_{i \neq s}, V_t}}$. Using the definition of quantile, it holds that
    \begin{align}\label{eq:alpha_quantile_virtual}
    \frac{1}{|\widehat{\gC}_t|+1}\sum_{s\in \widehat{\gC}_t\cup \{t\}} \Indicator{R_s > Q_{\alpha}(\{R_i\}_{i\in \widehat{\gC}_t \cup \{t\}})} \leq \alpha.
\end{align}
From the construction in \eqref{eq:cond_PI}, we also have
\begin{align}\label{eq:cover_equivalence}
    \Indicator{Y_t \not\in \gI_t^{\rm{CAP}}(X_t)} &= \Indicator{R_t > Q_{\alpha}(\{R_i\}_{i\in \widehat{\gC}_t \cup \{t\}})}.
\end{align}
By arranging \eqref{eq:alpha_quantile_virtual} and \eqref{eq:cover_equivalence}, we can upper bound the miscoverage indicator as
\begin{align}\label{eq:miscover_upper_bound}
    \Indicator{Y_t \not\in \gI_t^{\rm{CAP}}(X_t)} \leq \alpha + \frac{1}{|\widehat{\gC}_t|+1}\sum_{s\in \widehat{\gC}_t} \Indicator{R_t > Q_{\alpha}(\{R_i\}_{i\in \widehat{\gC}_t \cup \{t\}})} - \Indicator{R_s > Q_{\alpha}(\{R_i\}_{i\in \widehat{\gC}_t \cup \{t\}})}.
\end{align}
For each pair $(s, t)$ with $s \in \widehat{\gC}_t$, we introduce a sequence of virtual decision indicators:
\begin{align}
    S_j^{(s \gets t)} = \Indicator{V_j \leq \gA\LRs{\{V_i\}_{i\neq s}, V_t}},\quad \text{for }0\leq j \leq T, j\neq t.
\end{align}
Correspondingly, we denote $\rmS_t(T) = \sum_{j=0, j \neq t}^{T} S_j$ and $\rmS_t^{(s\gets t)}(T) = \sum_{j=0, j \neq t}^{T} S_j^{(s \gets t)}$. 
Plugging \eqref{eq:miscover_upper_bound} into the definition of FCR gives
\begin{align}
    &\FCR(T)\nonumber\\
    &= \E\LRm{\sum_{t=0}^T \frac{S_t}{\rmS_t(T) + 1} \Indicator{Y_t \not\in \gI_t^{\rm{CAP}}(X_t)}}\nonumber\\
    &\leq \alpha + \E\LRm{\sum_{t=0}^T \frac{S_t}{\rmS_t(T) + 1} \frac{1}{|\widehat{\gC}_t|+1} \sum_{s\in \widehat{\gC}_t} \LRs{\Indicator{R_t > Q_{\alpha}(\{R_i\}_{i\in \widehat{\gC}_t \cup \{t\}})} - \Indicator{R_s > Q_{\alpha}(\{R_i\}_{i\in \widehat{\gC}_t \cup \{t\}})}}} \nonumber\\
    %&= \alpha + \E\LRm{\frac{1}{\sum_{j=0}^T S_j \vee 1} \sum_{t=0}^T \frac{1}{|\widehat{\gC}_t|+1} \sum_{s=-n}^{-1} S_t C_{t,s} \LRs{\Indicator{R_t > Q_{\alpha}(\{R_i\}_{i\in \widehat{\gC}_t \cup \{t\}})} - \Indicator{R_s > Q_{\alpha}(\{R_i\}_{i\in \widehat{\gC}_t \cup \{t\}})}}}\nonumber\\
    &= \alpha + \sum_{t=0}^T \sum_{s=-n}^{-1} \E\LRm{\frac{1}{\rmS_t(T) + 1} \frac{S_t C_{t,s}}{|\widehat{\gC}_t|+1}  \LRs{\Indicator{R_t > Q_{\alpha}(\{R_i\}_{i\in \widehat{\gC}_t \cup \{t\}})} - \Indicator{R_s > Q_{\alpha}(\{R_i\}_{i\in \widehat{\gC}_t \cup \{t\}})}}}\nonumber\\
    &= \alpha + \sum_{t=0}^T\sum_{s=-n}^{-1} \E\LRm{\LRl{\frac{1}{\rmS_t^{(s\gets t)}(T) + 1} - \frac{1}{\rmS_t(T) + 1}}\frac{S_t C_{t,s}}{|\widehat{\gC}_t|+1} \Indicator{R_s > Q_{\alpha}(\{R_i\}_{i\in \widehat{\gC}_t \cup \{t\}})}}\nonumber\\
    &\leq \alpha + \sum_{t=0}^T\sum_{s=-n}^{-1} \E\LRm{\LRabs{\frac{1}{\sum_{j=0,j\neq t} S_j^{(s\gets t)} +1} - \frac{1}{\sum_{j=0,j\neq t} S_j +1}}\cdot \frac{S_t C_{t,s}}{|\widehat{\gC}_t|+1} },
    %&\leq \alpha + \sum_{t=0}^T\sum_{s=-n}^{-1} \E\LRm{\frac{S_t}{\sum_{j=0,j\neq t} S_j +1}\frac{ C_{t,s}}{|\widehat{\gC}_t|+1}\frac{\LRabs{\sum_{j=0, j\neq t} S_j - S_j^{(s\gets t)}}}{\sum_{j=0,j\neq t} S_j^{(s\gets t)} +1} },
    \label{eq:FCR_with_gap}
\end{align}
where the last equality follows from Lemma \ref{lemma:zero_gap_st}. Let $F_V(\cdot)$ be the c.d.f. of $\{V_i\}_{i\geq -n}$. 
Denote $\widehat{q}^{(s\gets t)} = F_V\{\gA\LRs{\{V_i\}_{i\neq s}, V_t}\}$, and $q = F_V\{\gA\LRs{\{V_i\}_{i\neq s}, V_s}\}$. 
Then given $Z_t, \{Z_i\}_{1\leq i \leq n}$, we know $\sum_{j=0,j\neq t} S_j \sim \operatorname{Binomial}(T, 1-\widehat{q})$ and $\sum_{j=0,j\neq t} S_j^{(s\gets t)} \sim \operatorname{Binomial}(T, 1-\widehat{q}^{(s\gets t)})$, which further yield
\begin{align}\label{eq:rej_num_diff}
    &\E\LRm{\frac{1}{\sum_{j=0,j\neq t} S_j^{(s\gets t)} +1} - \frac{1}{\sum_{j=0,j\neq t} S_j +1} \mid Z_t, \{Z_i\}_{1\leq i \leq n} }\nonumber\\
    &\qquad = \frac{1 - (\widehat{q}^{(s\gets t)})^{T+1}}{(T+1) (1 - \widehat{q}^{(s\gets t)})} - \frac{1 - \widehat{q}^{T+1}}{(T+1) (1 - \widehat{q})}\nonumber\\
    &\qquad = \frac{1 - \widehat{q}^{T+1}}{(T+1) (1 - \widehat{q})} \LRl{\frac{1- \widehat{q}}{1- \widehat{q}^{(s\gets t)}} \frac{1 - (\widehat{q}^{(s\gets t)})^{T+1}}{1 - \widehat{q}^{T+1}} - 1}\nonumber\\
    &\qquad = \E\LRm{\frac{1}{\sum_{j=0,j\neq t} S_j +1} \mid Z_t, \{Z_i\}_{1\leq i \leq n} } \cdot \LRl{\frac{1- \widehat{q}}{1- \widehat{q}^{(s\gets t)}} \frac{1 - (\widehat{q}^{(s\gets t)})^{T+1}}{1 - \widehat{q}^{T+1}} - 1}.
\end{align}
Notice that, if $\widehat{q}^{(s\gets t)} \geq \widehat{q}$, 
%{\color{red}[the $ \widehat{q}^{(t)}$ should be $\widehat{q}^{(s\gets t)}$?]}
    \begin{align}\label{eq:ratio_one_diff}
        &\frac{1 - \widehat{q}}{1 - \widehat{q}^{(s\gets t)}} \frac{1 - (\widehat{q}^{(s\gets t)})^{T+1}}{1 - \widehat{q}^{T+1}} - 1\\
        &\qquad= \frac{1 - (\widehat{q}^{(s\gets t)})^{T+1}}{1 - \widehat{q}^{T+1}} \frac{ \widehat{q}^{(s\gets t)} - \widehat{q} }{1 - \widehat{q}^{(t)}} + \frac{(\widehat{q}^{(s\gets t)})^{T+1} - \widehat{q}^{T+1}}{1 - \widehat{q}^{T+1}}\nonumber\\
        &\qquad= \frac{1 - (\widehat{q}^{(s\gets t)})^{T+1}}{1 - \widehat{q}^{T+1}} \frac{ \widehat{q}^{(s\gets t)} - \widehat{q} }{1 - \widehat{q}^{(t)}} + \frac{\LRs{\widehat{q}^{(s\gets t)} - \widehat{q}}\sum_{k=0}^T (\widehat{q}^{(s\gets t)})^k q^{T-k} }{1 - \widehat{q}^{T+1}}\nonumber\\
        &\qquad\leq \frac{ \widehat{q}^{(s\gets t)} - \widehat{q} }{1 - \widehat{q}^{(t)}} + \frac{\widehat{q}^{(s\gets t)} - \widehat{q} }{1 - \widehat{q}}\nonumber\\
        &\qquad= \frac{2(\widehat{q}^{(s\gets t)} - \widehat{q}) }{1 - \widehat{q}^{(s\gets t)}}.
    \end{align}
Since $S_t$, $C_{t,s}$, $|\widehat{\gC}_t|$ and $\Indicator{R_s > Q_{\alpha}(\{R_i\}_{i\in \widehat{\gC}_t \cup \{t\}})}$ depend only on calibration set and $Z_t$, substituting \eqref{eq:rej_num_diff} and \eqref{eq:ratio_one_diff} into \eqref{eq:FCR_with_gap} results in the following upper bound
\begin{align}
    \FCR(T) 
    %&\leq \alpha + \E\LRm{\sum_{t=0}^T\frac{S_t}{\sum_{j=0,j\neq t} S_j +1} \sum_{s=-n}^{-1} \frac{C_{t,s} \Indicator{\widehat{q}^{(s\gets t)} \geq \widehat{q}}}{|\widehat{\gC}_t|+1} \LRl{\frac{1- \widehat{q}}{1- \widehat{q}^{(s\gets t)}} \frac{1 - (\widehat{q}^{(s\gets t)})^{T+1}}{1 - \widehat{q}^{T+1}} - 1}  }\nonumber\\
    &\leq \alpha + \E\LRm{\sum_{t=0}^T\frac{S_t}{\sum_{j=0,j\neq t} S_j +1} \sum_{s=-n}^{-1} \frac{C_{t,s} \Indicator{\widehat{q}^{(s\gets t)} \geq \widehat{q}}}{|\widehat{\gC}_t|+1}  \frac{2(\widehat{q}^{(s\gets t)} - \widehat{q}) }{1 - \widehat{q}^{(s\gets t)}} }\nonumber\\
    &\leq \alpha + \E\LRm{\sum_{t=0}^T\frac{S_t}{\sum_{j=0,j\neq t} S_j +1} \sum_{s=-n}^{-1} \frac{C_{t,s} }{|\widehat{\gC}_t|+1}  \frac{2\LRabs{\widehat{q}^{(s\gets t)} - \widehat{q}} }{1 - \widehat{q}^{(s\gets t)}} }.\nonumber
\end{align}
% In addition, invoking the second conclusion of Lemma \ref{lemma:zero_gap_st}, we can obtain the bound for $\mFCR(T)$ as
% \begin{align}
%     \mFCR(T) &= \frac{1}{\E\LRm{1 \vee \sum_{j=0}^T S_j}}\E\LRm{\sum_{t=0}^T S_t \Indicator{Y_t \not \in \gI_t^{\rm{CAP}}(X_t)}}\nonumber\\
%     &\leq \frac{1}{\E\LRm{1 \vee \sum_{j=0}^T S_j}}\E\LRm{\sum_{t=0}^T S_t \LRl{\alpha + \sum_{s = 1}^n \frac{C_{t,n+1} C_{t,s}}{|\widehat{\gC}_t|+1} \LRs{\Indicator{R_t > Q_{\alpha}(\{R_i\}_{i\in \widehat{\gC}_t \cup \{t\}})} - \Indicator{R_s > Q_{\alpha}(\{R_i\}_{i\in \widehat{\gC}_t \cup \{t\}})}}}}\nonumber\\
%     &= \alpha\cdot \frac{\E\LRm{\sum_{j=0}^T S_j}}{\E\LRm{1 \vee \sum_{j=0}^T S_j}}\nonumber\\
%     &\leq \alpha.\nonumber
% \end{align}
\end{proof}

\section{Proofs of CAP under distribution shift}

Denote the selection time by $\{\tau_1,\ldots,\tau_M\}$, where $M = \sum_{j=0}^T S_j$ and
\begin{align*}
    \tau_m = \min\LRl{0\leq t \leq T: \sum_{j=1}^t S_j = m},\quad\text{for }1\leq m\leq M
\end{align*}
Then we know 
\begin{align*}
    \alpha_{\tau_{m+1}}^i \gets \alpha_{\tau_{m}}^i + \gamma_i (\alpha - \operatorname{err}_{\tau_{m}}^i),\quad\text{for }1\leq m\leq M-1.
\end{align*}
\begin{lemma}[Lemma 4.1 of \citet{gibbs2021adaptive}, modified.]\label{lemma:alpha_ti_range}
    With probability one we have that $\alpha_{\tau_m}^i \in [-\gamma_{\tau_m}^i, 1+\gamma_{\tau_m}^i]$ for $0\leq m \leq M$.
\end{lemma}
\begin{proof}[Proof of Theorem \ref{thm:cond_DtACI}]
%{\color{red}[Our current Algorithm 2 is not random. So we may not need to use $\E_{\rm{A}}$ ]}
    The proof is adapted from the proof of Theorem 3.2 in \citet{gibbs2022conformal}, and here we provide it for completeness. We write $\E_{\rm{A}}[\cdot]$ as the expectation taken over the randomness from the algorithm. Let $\sum_{j=0}^T S_j = M$. Let $\widetilde{\alpha}_t = \sum_{i=1}^k \frac{p_t^i \alpha_t^i}{\gamma_i}$ with $p_t^i = w_{t}^i/(\sum_{j=1}^k w_t^j)$. From the update rule of Algorithm \ref{alg:cond_DtACI}, we know
    \begin{align*}
        \widetilde{\alpha}_{\tau_m} &= \sum_{i=1}^k \frac{p_{\tau_m}^i}{\gamma_i}\LRs{\alpha_{\tau_m+1}^i + \gamma_i (\err_{\tau_m}^i - \alpha)}\\
        &= \sum_{i=1}^k \frac{p_{\tau_m}^i \alpha_{\tau_{m+1}}^i}{\gamma_i} + \sum_{i=1}^k p_{\tau_m}^i (\err_{\tau_m}^i -\alpha)\\
        &=\widetilde{\alpha}_{\tau_{m+1}} + \sum_{i=1}^k \frac{(p_{\tau_m}^i-p_{\tau_{m+1}}^i) \alpha_{\tau_{m+1}}^i}{\gamma_i} + \sum_{i=1}^k p_{\tau_m}^i (\err_{\tau_m}^i - \alpha).
    \end{align*}
    Notice that $\alpha_{\tau_m} = \alpha_{\tau_m}^i$ with probability $p_{\tau_m}^i$, hence $\err_{\tau_m} = \err_{\tau_m}^i$ with probability $p_{\tau_m}^i$, which means $\E_{\ermA}[\err_{\tau_m}] = \sum_{i=1}^k p_{\tau_m}^i \err_{\tau_m}^i$. It follows that
    \begin{align}\label{eq:errt_recursion}
        \E_{\ermA}[\err_{\tau_m}] - \alpha = \widetilde{\alpha}_{\tau_m} - \widetilde{\alpha}_{\tau_{m+1}} + \sum_{i=1}^k \frac{(p_{\tau_{m+1}}^i - p_{\tau_m}^i) \alpha_{\tau_{m+1}}^i}{\gamma_i}.
    \end{align}
    Now, denote $W_{\tau_m} = \sum_{i=1}^k w_{\tau_m}^i$ and $\widetilde{p}_{\tau_{m+1}}^i = \frac{p_{\tau_m}^i \exp\LRs{-\eta_{\tau_m} \ell(\beta_{\tau_m}, \alpha_{\tau_m}^i)}}{\sum_{j=1}^k p_{\tau_m}^j \exp\LRs{-\eta_{\tau_m} \ell(\beta_{\tau_m}, \alpha_{\tau_m}^j)}}$. From the definition of $p_{\tau_{m+1}}^i$, we know
    \begin{align}\label{eq:p_and_p_tilde}
        p_{\tau_{m+1}}^i &= \frac{w_{\tau_{m+1}}^i/ W_{\tau_m}}{\sum_{j=1}^k w_{\tau_{m+1}}^j/ W_{\tau_m}}\nonumber\\
        &=  \frac{(1-\phi_{\tau_m}) \bar{w}_{\tau_m}^i/W_{\tau_m} + \phi_{\tau_m}\sum_{j=1}^k (\bar{w}_{\tau_m}^j/W_{\tau_m})/k }{\sum_{j=1}^k [(1-\phi_{\tau_m}) \bar{w}_{\tau_m}^j/W_{\tau_m} + \phi_{\tau_m}/k\sum_{l=1}^k \bar{w}_{\tau_m}^l/W_{\tau_m}]}\nonumber\\
        &= \frac{(1-\phi_{\tau_m}) \bar{w}_{\tau_m}^i/W_{\tau_m} + \phi_{\tau_m}\sum_{j=1}^k (\bar{w}_{\tau_m}^j/W_{\tau_m})/k }{(1-\phi_{\tau_m})\sum_{j=1}^k \bar{w}_{\tau_m}^j/W_{\tau_m} + \phi_{\tau_m}\sum_{l=1}^k \bar{w}_{\tau_m}^l/W_{\tau_m}}\nonumber\\
        &= \frac{(1-\phi_{\tau_m}) p_{\tau_m}^i \exp\LRs{-\eta_{\tau_m} \ell(\beta_{\tau_m}, \alpha_{\tau_m}^i)} + \phi_{\tau_m}\sum_{j=1}^k p_{\tau_m}^j \exp\LRs{-\eta_{\tau_m} \ell(\beta_{\tau_m}, \alpha_{\tau_m}^j)}/k }{\sum_{i=1}^k p_{\tau_m}^i \exp\LRs{-\eta_{\tau_m} \ell(\beta_{\tau_m}, \alpha_{\tau_m}^i)}}\nonumber\\
        &= (1 - \phi_{\tau_m})\widetilde{p}_{\tau_{m+1}}^i + \frac{\phi_{\tau_m}}{k}.
    \end{align}
    Further, we also have
    \begin{align}\label{eq:tilde_p_diff}
        \widetilde{p}_{\tau_{m+1}}^i - p_{\tau_m}^i &= \frac{p_{\tau_m}^i \exp\LRs{-\eta_{\tau_m} \ell(\beta_{\tau_m}, \alpha_{\tau_m}^i)}}{\sum_{j=1}^k p_{\tau_m}^j \exp\LRs{-\eta_{\tau_m} \ell(\beta_{\tau_m}, \alpha_{\tau_m}^j)}} - p_{\tau_m}^i\nonumber\\
        &=p_{\tau_m}^i \cdot \frac{\sum_{j=1}^k p_{\tau_m}^j \LRl{\exp\LRs{-\eta_{\tau_m} \ell(\beta_{\tau_m}, \alpha_{\tau_m}^i)}- \exp\LRs{-\eta_{\tau_m} \ell(\beta_{\tau_m}, \alpha_{\tau_m}^j)}}}{\sum_{j=1}^k p_{\tau_m}^j \exp\LRs{-\eta_{\tau_m} \ell(\beta_{\tau_m}, \alpha_{\tau_m}^j)}}\nonumber\\
        &=p_{\tau_m}^i \cdot \frac{\sum_{j=1}^k p_{\tau_m}^j \exp\LRs{-\eta_{\tau_m} \ell(\beta_{\tau_m}, \alpha_{\tau_m}^j)} \LRl{\exp\LRs{\eta_{\tau_m} \LRm{\ell(\beta_{\tau_m}, \alpha_{\tau_m}^j) - \ell(\beta_{\tau_m}, \alpha_{\tau_m}^i)}}- 1}}{\sum_{j=1}^k p_{\tau_m}^j \exp\LRs{-\eta_{\tau_m} \ell(\beta_{\tau_m}, \alpha_{\tau_m}^j)}}\nonumber\\
        &=p_{\tau_m}^i\cdot \sum_{j=1}^k \widetilde{p}_{\tau_m}^j \LRl{\exp\LRs{\eta_{\tau_m} \LRm{\ell(\beta_{\tau_m}, \alpha_{\tau_m}^j) - \ell(\beta_{\tau_m}, \alpha_{\tau_m}^i)}}- 1}.
    \end{align}
    By Lemma \ref{lemma:alpha_ti_range} we know $\alpha_{\tau_m}^i \in [-\gamma_{\tau_m}^i, 1+\gamma_{\tau_m}^i]$, which implies $\LRabs{\ell(\beta_{\tau_m}, \alpha_{\tau_m}^j) - \ell(\beta_{\tau_m}, \alpha_{\tau_m}^i)} \leq \max\{\alpha, 1-\alpha\}\LRabs{\alpha_{\tau_m}^j - \alpha_{\tau_m}^i} \leq 1+2\gamma_{\max}$. By the intermediate value theorem, we can have
    \begin{align*}
        \LRabs{\exp\LRs{\eta_{\tau_m} \LRm{\ell(\beta_{\tau_m}, \alpha_{\tau_m}^j) - \ell(\beta_{\tau_m}, \alpha_{\tau_m}^i)}}- 1} \leq \eta_{\tau_m}(1 + 2\gamma_{\max})\exp\LRl{\eta_{\tau_m} (1+2\gamma_{\max})}.
    \end{align*}
    Plugging it into \eqref{eq:tilde_p_diff} yields
    \begin{align}
        \LRabs{\widetilde{p}_{\tau_{m+1}}^i - p_{\tau_m}^i} \leq p_{\tau_m}^i \eta_{\tau_m}(1 + 2\gamma_{\max})\exp\LRl{\eta_{\tau_m} (1+2\gamma_{\max})}.\nonumber
    \end{align}
    Together with \eqref{eq:p_and_p_tilde}, we have
    \begin{align}
        \LRabs{\frac{(p_{\tau_{m+1}}^i - p_{\tau_m}^i) \alpha_{\tau_{m+1}}^i}{\gamma_i}} &\leq (1 - \phi_{\tau_m}) \LRabs{\frac{(\widetilde{p}_{\tau_{m+1}}^i - p_{\tau_m}^i) \alpha_{\tau_{m+1}}^i}{\gamma_i}} + \phi_{\tau_m} \LRabs{\frac{(k^{-1} - p_{\tau_m}^i) \alpha_{\tau_{m+1}}^i}{\gamma_i}}\nonumber\\
        &\leq \frac{\eta_{\tau_m} (1 + 2\gamma_{\max})^2}{\gamma_{\min}}\exp\LRl{\eta_{\tau_m} (1+2\gamma_{\max})} + 2\phi_{\tau_m} \frac{1+\gamma_{\max}}{\gamma_{\min}}\nonumber,
    \end{align}
    where we used Lemma \ref{lemma:alpha_ti_range}.
    Telescoping the recursion \eqref{eq:errt_recursion} from $m=1$ to $m=M$, we can get
    \begin{align}
        \LRabs{\sum_{m=1}^{M} \LRs{\E_{\ermA}[\err_{\tau_m}] - \alpha}} &\leq \LRabs{\widetilde{\alpha}_{\tau_1} - \widetilde{\alpha}_{\tau_{M+1}}} + \frac{ (1 + 2\gamma_{\max})^2}{\gamma_{\min}}\sum_{m=1}^{M} \eta_{\tau_m}\exp\LRl{\eta_{\tau_m} (1+2\gamma_{\max})}\nonumber\\
        &\qquad+ 2 \frac{1+\gamma_{\max}}{\gamma_{\min}}\sum_{m=1}^M \phi_{\tau_m}\nonumber\\
        &\leq \frac{1+2\gamma_{\max}}{\gamma_{\min}} + \frac{ (1 + 2\gamma_{\max})^2}{\gamma_{\min}}\sum_{m=1}^{M} \eta_{\tau_m}\exp\LRl{\eta_{\tau_m} (1+2\gamma_{\max})}\nonumber\\
        &\qquad + 2 \frac{1+\gamma_{\max}}{\gamma_{\min}}\sum_{m=1}^M \phi_{\tau_m}.\nonumber
    \end{align}
    According to the definition of $\tau_m$, we can rewrite the above relation as
    \begin{align}
        \LRabs{\sum_{t=0}^T S_t \LRs{\E_{\ermA}[\err_{t}] - \alpha}} & = \LRabs{\sum_{m=1}^{M} \LRs{\E_{\ermA}[\err_{\tau_m}] - \alpha}}\nonumber\\
        &\leq \frac{1+2\gamma_{\max}}{\gamma_{\min}} + \frac{ (1 + 2\gamma_{\max})^2}{\gamma_{\min}}\sum_{t=0}^{T}S_t \eta_{t} \exp\LRl{\eta_{t} (1+2\gamma_{\max})}\nonumber\\
        &\qquad + 2 \frac{1+\gamma_{\max}}{\gamma_{\min}}\sum_{t=0}^{T}S_t \phi_{t}.\nonumber
    \end{align}
    Since the randomness of Algorithm \ref{alg:cond_DtACI} is independent of the decisions $\{S_i\}_{i=0}^T$ and the data $\{Z_i\}_{i=-n}^T$, we have
    \begin{align}
        \E\LRm{\frac{\sum_{t=0}^T S_t\cdot \err_t}{\sum_{j=0}^T S_j}} - \alpha = \E\LRm{\frac{\sum_{t=0}^T S_t\cdot \LRs{\E_{\ermA}[\err_{t}] - \alpha}}{\sum_{j=0}^T S_j} }.\nonumber
    \end{align}
    The conclusion follows from the definition of $\varrho_t$ immediately.
\end{proof}

\section{Additional simulation details}
\setcounter{figure}{0}
\def\thefigure{F.\arabic{figure}}

\subsection{Details of e-LOND-CI}\label{appen:e-lord-ci}
The e-LOND-CI is similar to  LORD-CI, except for using e-values and LOND procedure instead. At each time $t$, the prediction interval is constructed as $\{y: e_t(X_t,y)<\alpha_t^{-1}\}$, where $e_t(X_t,y)$ is the e-value at time $t$ associated with $X_t$ and $y$ and $\alpha_t$ is the target level at time $t$ computed by $\alpha_t=\alpha\gamma^{\rm LOND}_t(\mathbf{S}_{t-1}+1)$, where $\gamma^{\rm LOND}_t$ is discount sequence. We choose $\gamma^{\rm LOND}_t=1/\{t(t-1)\}$ as \citet{xu2023online} suggested.

\begin{figure}[tb]
    \centering
    \includegraphics[width=\textwidth]{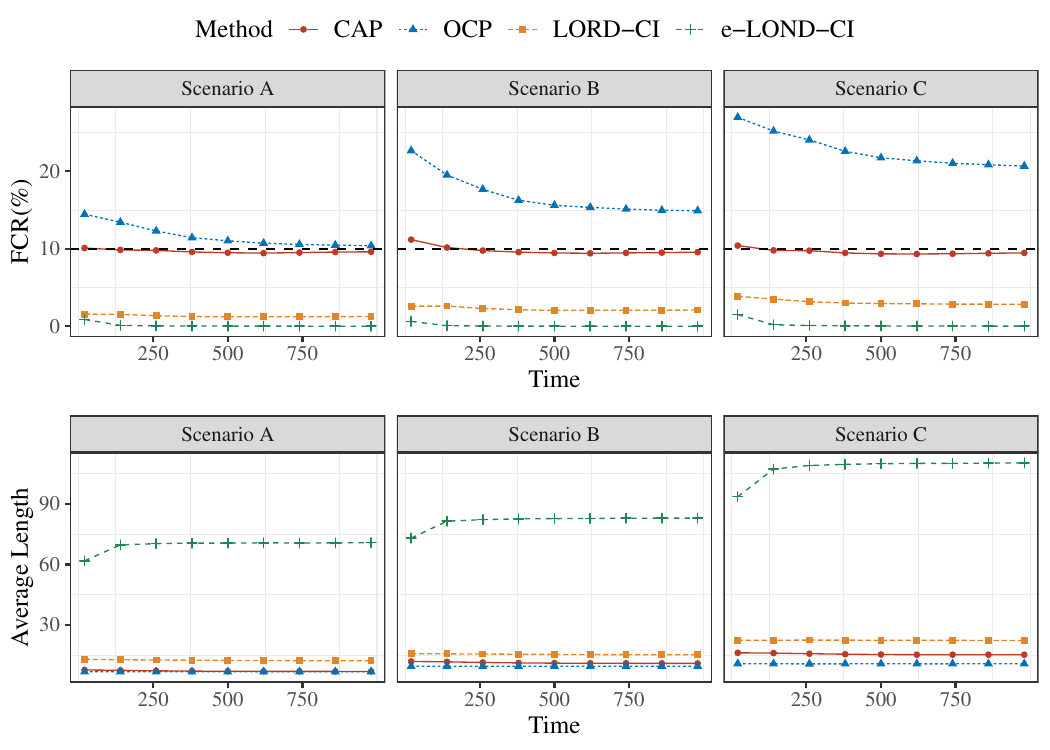}
    \caption{\small\it Real-time FCR plot and average length  plot from time $20$ to $2,000$ for e-LOND-CI. The selection rule is \textbf{Dec-driven} and the incremental holdout set with window size $200$ is considered. The black dashed line represents the target FCR level $10\%$.}
    \label{fig:eLONDCI}
\end{figure}

The e-value for constructing prediction intervals is transformed by $p$-values. By the duality of confidence interval and hypothesis testing, we can invert the task of constructing prediction intervals as testing. Let $H_{0t}: Y_t=y$, then the $p$-values are defined as
$$p_t(X_t,y)=\frac{\sum_{i\in\gC_t}\Indicator{|y-\widehat{\mu}(X_t)|\leq R_i}+1}{|\gC_t|+1}.$$
Following \citet{xu2023online}, we can directly convert this $p$-value into
$$e_t(X_t,y)=\frac{\Indicator{p_t(X_t,y)\leq \alpha_t}}{\alpha_t}.$$
By a same discussion as Proposition 2 in \citet{xu2023online} , we can verify that $\E[e_t(X_t,Y_t)\Indicator{Y_t=y}]\leq 1$, hence $e_t(X_t,Y_t)$ is a valid e-value.

We provide additional simulations for e-LOND-CI. Figure \ref{fig:eLONDCI} illustrates the FCR and average length under different scenarios using decision-driven selection for e-LOND-CI. As it is shown, the prediction intervals produced by e-LOND-CI are considerably wide, limiting the e-LOND-CI to provide non-trivial uncertainty quantification.

\subsection{Details of online multiple testing procedure using conformal $p$-values}\label{appen:conformal_p}

In our procedure, the conformal p-values \citep{jin2022selection} are constructed using an additional labeled data set $\gD_{\rm Add}=\{X_i,Y_i\}_{i=-(n+m)}^{-(n+1)}$, instead of the current holdout set. By doing this, these conformal p-values are independent given $\gD_{\rm Add}$, making the online multiple testing procedure decision-driven. The specific construction of the conformal p-values is outlined as follows:

Recall that the selection problem can be viewed as the following multiple hypothesis tests: for time $t$ and some constant $c_0 \in \sR$,
\begin{align*}
    H_{0,t}: Y_i \leq c_0\quad \text{v.s.}\quad H_{1,t}: Y_t > c_0.
\end{align*}
Defined the null data set in $\gD_{\rm Add}$ as $\gD^0_{\rm Add}=\{(X_i,Y_i)\in \gD_{\rm Add}:Y_i\leq c_0\}$. 
For each test data point, the (marginal) conformal $p$-value $p^{\rm marg}_t$ based on same-class calibration \citep{bates2023aos} can be calculated by
\begin{align*}\label{eq:conformal_p_values}
    p^{\rm marg}_t = \frac{1 + |\{(X_i,Y_i)\in \gD^0_{\rm Add}: g(X_i) \leq g(X_t)\}|}{|\gD^0_{\rm Add}|+1}, 
\end{align*}
where $g(x) =c_0- \widehat{\mu}(x)$ is the nonconformity score function for constructing $p$-values. If each test corresponds to different constant $c_t$ for determining null and non-null, we can use the conformal p-value in \citet{jin2022selection} which uses the whole additional data $\gD_{\rm Add}$ and specific nonconformity score functions for construction.

To control the online $\FDR$ at the level $\beta \in (0,1)$ throughout the procedure, we deploy the SAFFRON \citep{ramdas2018saffron} procedure. The main idea of SAFFRON is to make a more precise estimation of current FDP by incorporating the null proportion information. Given $\beta\in(0,1)$, the user starts to pick a constant $\lambda\in (0,1)$ used for estimating the null proportion, an initial wealth $W_0\leq \beta$ and a positive non-increasing sequence $\{\gamma_j\}_{j=1}^\infty$ of summing to one. The SAFFRON begins by allocation the rejection threshold $\beta_1=\min \{(1-\lambda)\gamma_1W_0,\lambda\}$ and for $t\leq 2$ it sets:
$$\beta_t=\min\Big\{\lambda, (1-\lambda)\Big(W_0\gamma_{t-C_{0+}}+(\alpha-W_0)\gamma_{t-\tau_1-C_{1+}}+\sum_{j\geq 2}\beta \gamma_{t-\tau_j-C_{j+}}\Big)\Big\},$$
where $\tau_j$ is the time of the $j$-th rejection (define $\tau_0=0$), and $C_{j+}=\sum_{i=\tau_j+1}^{t-1}\Indicator{p_i\leq\lambda}$. Thus for each time $t$, we reject the hypothesis if $p_t\leq\beta_t$. In our experiment, we set defaulted parameters, where $W_0=\beta/2$, $\lambda=0.5$ and $\gamma_j\propto 1/j^{1.6}$.

At last, we discuss the potential issue of online multiple testing procedure using conformal p-values. The conformal p-value conditional on $\gD_{\rm Add}$ is no longer super-uniform, which may hinder the validity of online FDR control. But it does not affect the performance of our CAP procedure, as we focus on interval construction and FCR control. For practitioners requiring rigorous online FDR control, we provide to use calibration-conditional $p$-values \citet{bates2023aos} to guarantee this.

The calibration-conditional $p$-value $p^{\rm ccv}$ proposed by \citet{bates2023aos} is valid conditional on the additional null labeled set $\gD^0_{\rm Add}$ for at least probability $1-\delta$, i.e.
$$\sP\left\{\sP(p^{\rm ccv}_t\leq x \mid \gD^0_{\rm Add})\leq x \text{ for all }x\in(0,1] \right\}\geq 1-\delta.$$
Let $\gD^0_{\rm Add} = |m_0|$, \citet{bates2023aos} used an adjustment function $h(x)=b_{\lceil (m_0+1)x\rceil}$ to map the marginal conformal $p$-value to the calibration-conditional p-value, where $b=\{b_i\}_{i=1}^{m_0}$ ($b_1\leq b_2\leq\cdots\leq b_{m_0}$) satisfying
$\sP(U_{(1)}\leq b_1,\cdots,U_{(m_0)}\leq b_{n_0})\geq 1-\delta$, 
and $U_{(i)}$ is the $i$-th largest from $m_0$ i.i.d. uniform random variables. Then the calibration-conditional $p$-value can be computed by 
$$p_t^{\rm ccv}=h(p_t^{\rm marg}),\quad \forall t\geq 0.$$
The determination of $b_i$ can be through Simes inequality or Monte Carlo approach, see \cite{bates2023aos} for detailed discussion. 

Conditional on $\gD^0_{\rm Add}$, $\{p^{\rm ccv}_{t}\}_{t\geq 0}$ are all independent and super-uniform (with a probability $\delta$), hence online multiple testing procedures such as AI \citep{foster2008alpha}, GAI \citep{aharoni2014generalized} with $p$-values $\{p^{\rm ccv}_{t}\}_{t\geq 0}$ can control the FDR below $\beta$ with probability $1-\delta$, where $\beta$ is the nominal level.

\subsection{Experiments on fixed calibration set}

\begin{figure}[htbp!]
    \centering
    \includegraphics[width=\textwidth]{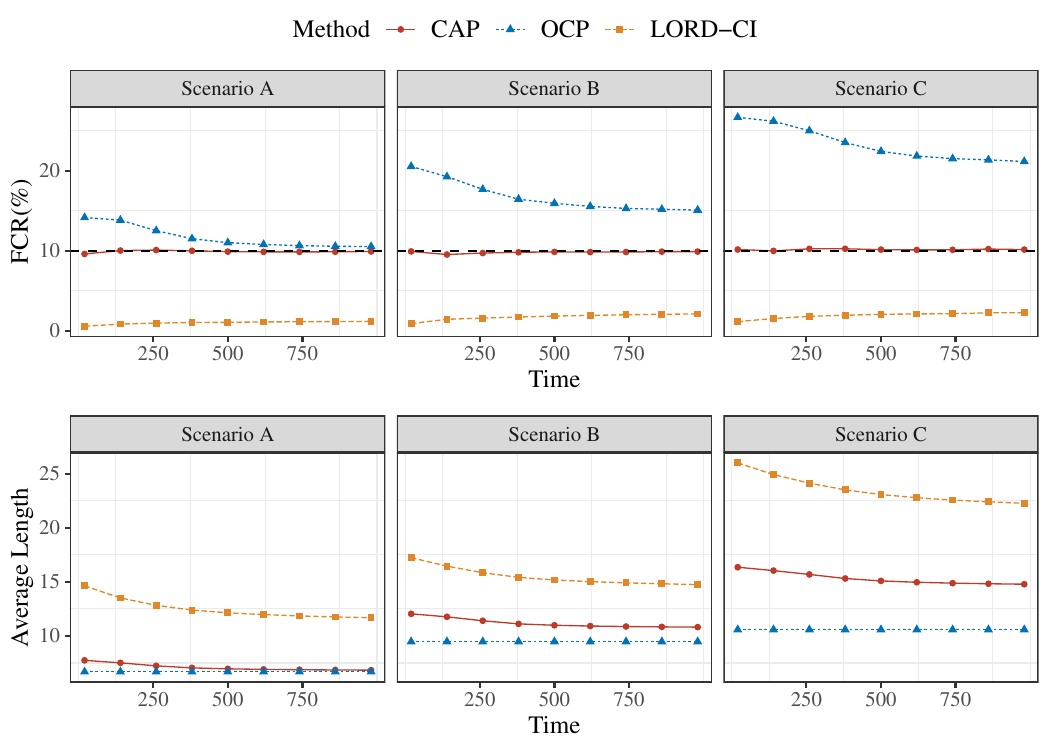}
    \caption{\small\it Real-time FCR plot and average length plot from time $20$ to $1,000$ for fixed calibration set after $500$ replications. The black dashed line represents the target FCR level $10\%$.}
    \label{fig:fixed}
\end{figure}

We verify the validity of our algorithms with respect to a fixed calibration set. The size of the fixed calibration set is set as $50$, and the procedure stops at time $1,000$. We design a decision-driven selection strategy. At each time $t$, the selection indicator is $S_t=\Indicator{V_t>\tau(\sum_{j=0}^{t-1}S_j)}$, where $V_t=\widehat{\mu}(X_t)$ and $\tau(s)=\tau_0-\min\{s/50,2\}$.
The parameter $\tau_0$ is pre-fixed for each scenario. Three different initial thresholds for different scenarios due to the change of the scale of the data. The thresholds $\tau_0$ are set as $1$, $4$ and $3$ for Scenarios A, B and C respectively. This selection rule is more aggressive when the number of selected samples is small.

We choose the target FCR level as $\alpha=10\%$. The real-time results are demonstrated in Figure \ref{fig:fixed} based on $500$ repetitions. Across all the settings, it is
evident that the CAP can deliver quite accurate FCR control and outputs narrower
PIs.

\subsection{Impacts of initial holdout set size}

\begin{figure}[tb]
    \centering
    \includegraphics[width=\textwidth]{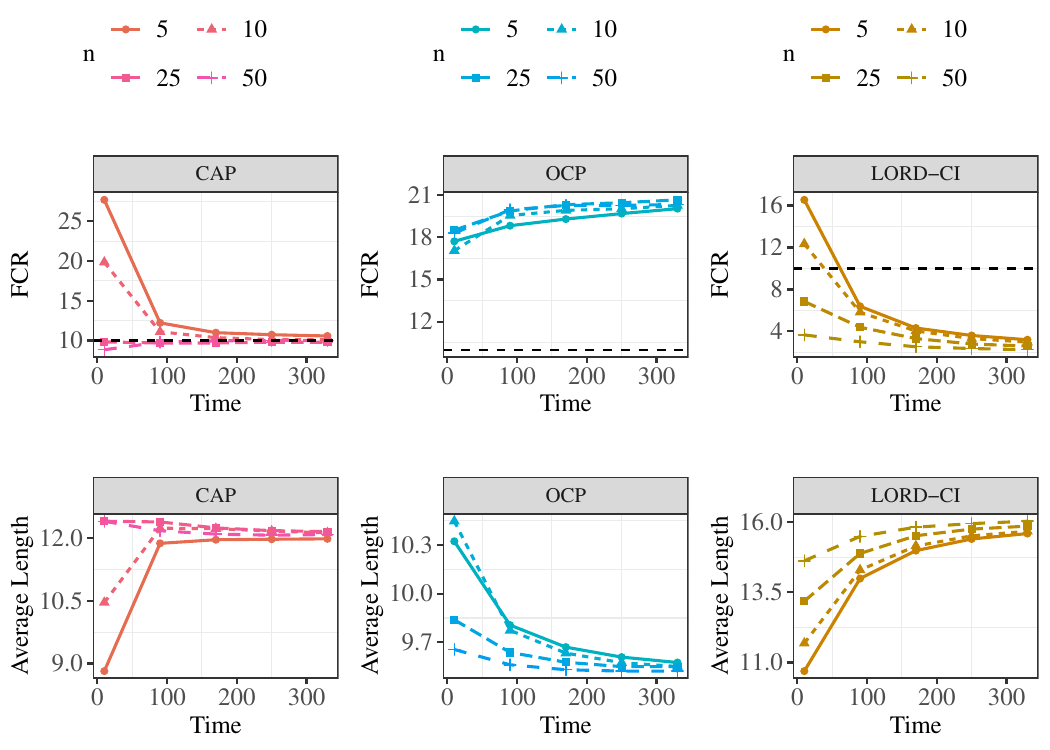}
    \caption{\small\it Real-time FCR and average length from time $20$ to $400$ using different sizes of the initial holdout set for CAP, OCP and LORD-CI. The basic setting is Scenario B and the quantile selection rule is used. The black dashed line denotes the target FCR level $10\%$.}
    \label{fig:n_change}
\end{figure}

Next we assess the impact of the initial holdout set size $n$. For simplicity, we focus on Scenario B and employ the quantile selection rule. We vary the initial size $n$ within the set $\{5,10,25,50\}$, and summarize the results among $500$ repetitions in Figure \ref{fig:n_change}. When the initial size is small, the CAP tends to exhibit overconfidence at the start of the stage. However, as time progresses, the FCR level approaches the target of $10\%$. Conversely, with a moderate initial size such as $25$, the CAP achieves tight FCR control throughout the procedure, thereby confirming our theoretical guarantee. A similar phenomenon is also observed with OCP and LORD-CI, wherein the FCR at the initial stage significantly diverges from the FCR at the end stage when a small value of $n$ is utilized. To ensure a stabilized FCR control throughout the entire procedure, we recommend employing a moderate size of for the initial holdout set.

\subsection{Comparisons of CAP and EXPRESS}\label{appen:compare_express}

Recall that our CAP-ada picks the calibration set by
\begin{align*}
    \hat{\gC}_t^{\rm CAP}= 
    \LRl{-n\leq s \leq t-1: \Pi_t(X_s) \prod_{i\in \gN_t^{\rm on}}\mathbbm{1}\{\Pi_i(X_s) = \Pi_i(X_t)\} = 1}.
\end{align*}
where $\gN_t^{\rm on} = \{0\leq j \leq t-1: \Pi_t(X_j) = 1\}$. While the EXPRESS proposed by \citet{sale2025online} outputs a calibration set indexed by
\begin{align}\nonumber
    \hat{\gC}_t^{\rm EXPRESS} = \LRl{-n\leq s \leq t-1: \Pi_t(X_s) \prod_{i=0}^{t-1}\mathbbm{1}\{\Pi_i(X_s) = \Pi_i(X_t)\} = 1},
\end{align}
which is more conservative compared to ours because $\gN_t^{\rm on} \subseteq \{0,\ldots,t-1\}$.

We conduct several simulations to verify the empirical performance of our proposed CAP and EXPRESS. If the picked calibration set is empty, we will report an interval with infinite length, which contributes to a correct selection when computing FCP. Therefore, we also compare the size of picked calibration points, the frequency of infinite length intervals, and the median length of interval instead of the mean length. In our simulations, we do not employ randomization to achieve exact coverage, which differs slightly from the procedure in \citet{sale2025online}. Specifically, we also compare the variants of CAP and EXPRESS by using a window scheme. It means we check the picking rule for online data within a windowed range, which can reduce the frequency that the picked set is empty. Denote our approach with window size $k$ as K-CAP
\begin{align*}
    \hat{\gC}_t^{{\rm K-CAP}}= 
    \LRl{-n\leq s \leq t-1: \Pi_t(X_s) \prod_{i\in \gN_t^{\rm on}\cap\{t-w,\cdots,t-1\}}\mathbbm{1}\{\Pi_i(X_s) = \Pi_i(X_t)\} = 1}
\end{align*}
and EXPRESS with window size $k$ as K-EXPRESS
\begin{align}\nonumber
    \hat{\gC}_t^{{\rm K-EXPRESS}} = \LRl{-n\leq s \leq t-1: \Pi_t(X_s) \prod_{i=t-k}^{t-1}\mathbbm{1}\{\Pi_i(X_s) = \Pi_i(X_t)\} = 1}.
\end{align}

\paragraph{Comparison Case 1}:
The first setting is from our Scenario A. Here $Y=\mu(X)+\epsilon$, $X\sim {\rm Unif}[-2,2]^{10}$ and $\mu(X)=X^\top\beta$ where $\beta=(\mathbf{1}_5^{\top},-\mathbf{1}_5^{\top})^\top$ and $\mathbf{1}_5$ is a $5$-dimensional vector with all elements $1$. The noise is heterogeneous and follows the conditional distribution $\epsilon\mid X\sim N(0,\{1+|\mu(X)|\}^2)$. We employ ordinary least squares (OLS) to obtain $\widehat{\mu}(\cdot)$. And the decision rule is $S_t=\mathds{1}\{\widehat{\mu}(X_t)>\tau(\sum_{i=0}^{t-1} S_i)\}$ where
$\tau(s)=2-\min\{s/20,2\}$. The initial holdout data size is $n=50$. Fixing the target FCR level at $40\%$ and the window size $K=20$ for K-CAP and K-EXPRESS, the results are depicted in Figure \ref{fig:comp_EXPRESS_oursetting}. It is clearly that our approach can produce a significantly smaller prediction interval, as long as a lower frequency of infinite interval. A similar phenomenon also happens to K-CAP and K-EXPRESS. In Table \ref{tab:comp_EXPRESS_oursetting}, we show the detailed values these methods take at $t=100$ and $t=200$.

\begin{figure}[htbp!]
    \centering
    \includegraphics[width=0.9\textwidth]{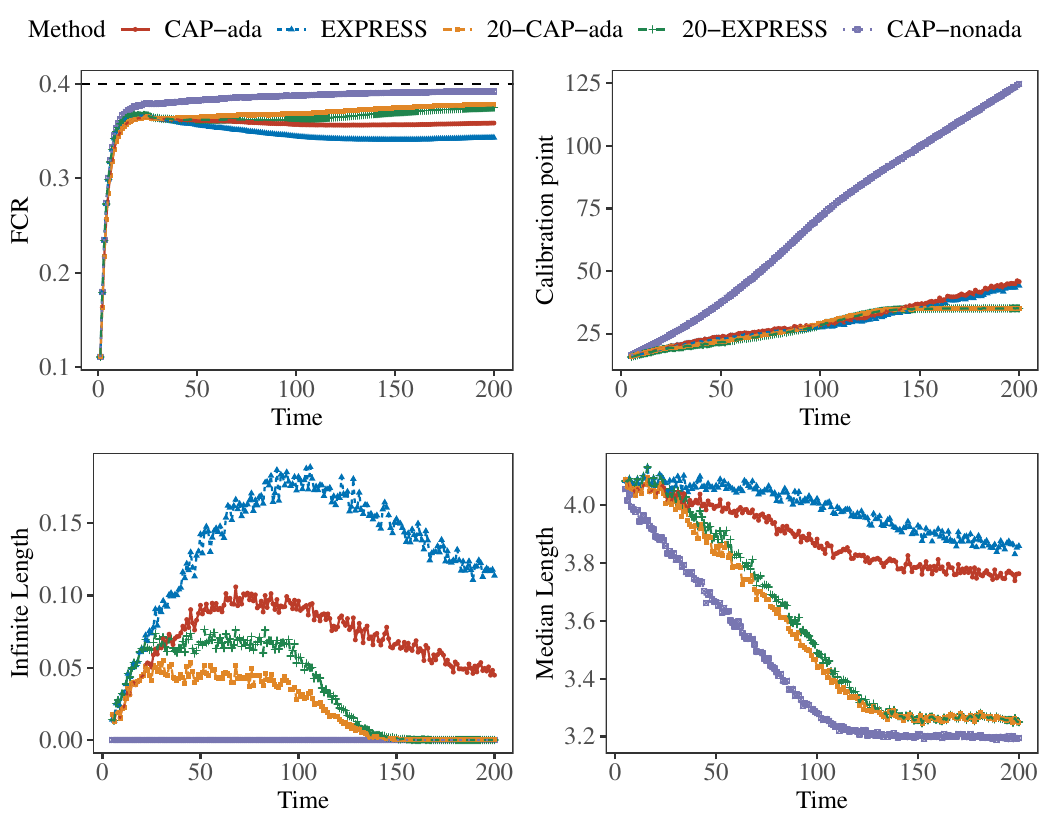}
    \caption{\small\it Comparison for our method with EXPRESS by real-time plots of FCR, calibration point number, frequency of infinite interval and median interval length from time $5$ to $200$ after $10,000$ replications under Comparison Case 1. The black dashed line represents the target FCR level $40\%$.}
    \label{fig:comp_EXPRESS_oursetting}
\end{figure}

\begin{table}[htbp]
\centering
\caption{Comparison of FCR, calibration point number (CP), frequency of infinite length interval (IL) and median interval length (ML) at $t=100$ and $t=200$ for different methods under Comparison Case 1. The target FCR level is $40\%$.}
\label{tab:comp_EXPRESS_oursetting}
\begin{tabular}{lcccccccc}
\toprule
\multirow{2}{*}{Method} & \multicolumn{4}{c}{$t=100$} & \multicolumn{4}{c}{$t=200$} \\
\cmidrule(lr){2-5} \cmidrule(lr){6-9}
& FCR & CP & IL & ML & FCR & CP & IL & ML \\
\midrule
CAP-ada     & 0.36 & 28.69 & 0.10 & 3.88  & 0.36 & 45.76  & 0.04 & 3.76 \\
EXPRESS     & 0.34 & 27.73 & 0.18 & 4.02  & 0.34 & 44.12  & 0.11 & 3.86 \\
20-CAP-ada  & 0.37 & 29.14 & 0.03 & 3.44  & 0.38 & 35.07  & 0.00 & 3.25 \\
20-EXPRESS  & 0.36 & 28.49 & 0.06 & 3.49  & 0.38 & 35.07  & 0.00 & 3.25 \\
CAP-nonada  & 0.39 & 71.83 & 0.00 & 3.28  & 0.39 & 124.77 & 0.00 & 3.19 \\
\bottomrule
\end{tabular}
\end{table}

\paragraph{Comparison Case 2}

The next setting is from \citet{sale2025online}. Let $X\sim{\rm Unif}[0,2]$ and $Y=X+\epsilon$, where $\epsilon\sim \mathcal{N}(0,X/2)$. The prediction model is defined as $\hat{\mu}(X)=X$. The selection rule is $S_t=\mathds{1}\{X_t<1+\sum_{i=0}^{t-1}S_i/200\}$. The results are summarized in Figure \ref{fig:comp_EXPRESS_SaleB} and Table \ref{tab:comp_EXPRESS_SaleB} with window size $K=10$. It shows that our method is at least as good as EXPRESS.

\begin{figure}[htbp!]
    \centering
    \includegraphics[width=0.9\textwidth]{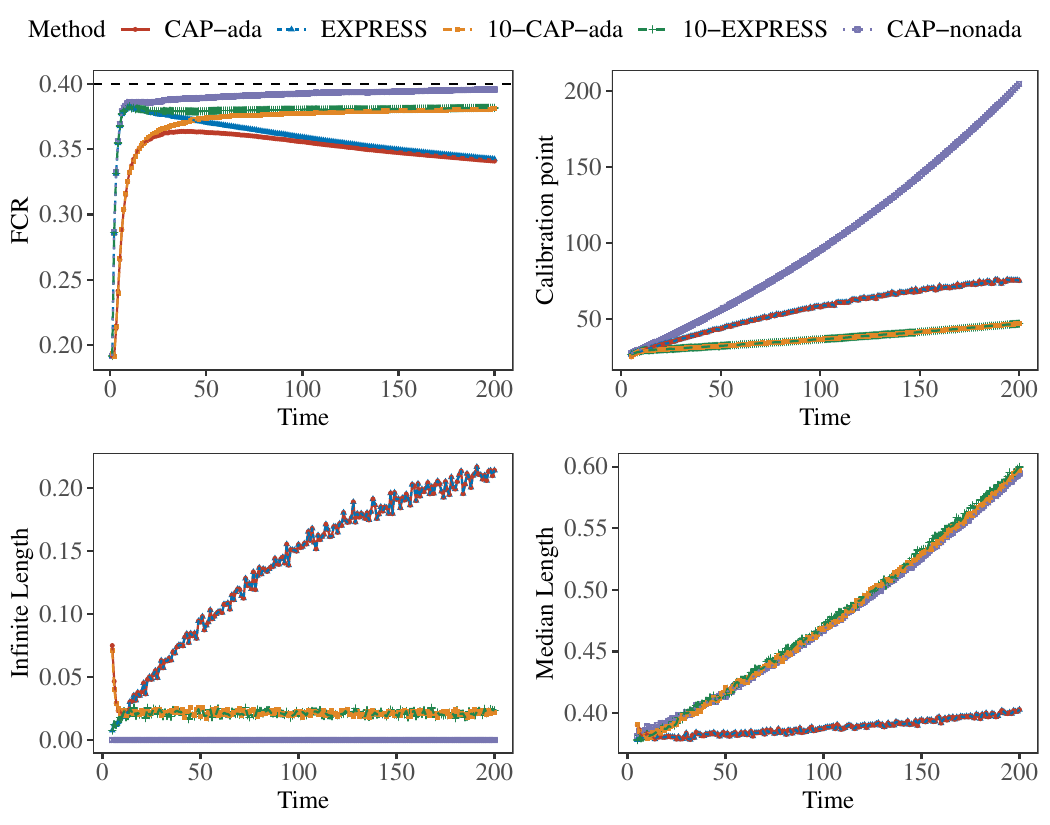}
    \caption{\small\it Comparison for our method with EXPRESS by real-time plots of FCR , calibration point number, frequency of infinite interval and median interval length from time $5$ to $200$ after $10,000$ replications under Comparison Case 2. The black dashed line represents the target FCR level $40\%$.}
    \label{fig:comp_EXPRESS_SaleB}
\end{figure}

\begin{table}[htbp]
\centering
\caption{Comparison of FCR , calibration point number (CP), frequency of infinite length interval (IL) and median interval length (ML) at $t=100$ and $t=200$ for different methods under Comparison Case 2. The target FCR level is $40\%$.}
\label{tab:comp_EXPRESS_SaleB}
\begin{tabular}{lcccccccc}
\toprule
\multirow{2}{*}{Method} & \multicolumn{4}{c}{$t=100$} & \multicolumn{4}{c}{$t=200$} \\
\cmidrule(lr){2-5} \cmidrule(lr){6-9}
& FCR & CP & IL & ML & FCR & CP & IL & ML \\
\midrule
CAP-ada     & 0.36 & 57.74 & 0.16 & 0.39 & 0.34 & 75.14 & 0.21 & 0.40 \\
EXPRESS     & 0.36 & 57.74 & 0.16 & 0.39 & 0.34 & 75.14 & 0.21 & 0.40 \\
10-CAP-ada  & 0.38 & 36.53 & 0.02 & 0.47 & 0.38 & 46.98 & 0.02 & 0.60 \\
10-EXPRESS  & 0.38 & 36.63 & 0.02 & 0.47 & 0.38 & 46.85 & 0.02 & 0.60 \\
CAP-nonada  & 0.39 & 95.40 & 0.00 & 0.47 & 0.40 & 204.53 & 0.00 & 0.59 \\
\bottomrule
\end{tabular}
\end{table}

\paragraph{Comparison Case 3} The final setting is also from \citet{sale2025online} to evaluate the performance of selection-conditional coverage. Here the data generating scenario is the same as Comparison Case 2. The decision rule is \[
S_t=\left\{
\begin{array}{ll}
    \mathds{1}\{X_t<1+\sum_{i=0}^{t-1}S_i/20\} & \text{if } t < 20 \\
    \mathds{1}\{\sum_{i=0}^{t}S_i>16\} & \text{if } t= 20.
\end{array}
\right.
\]
We stop our online procedure at $t=20$ and access the selection-conditional miscoverage by replicating $1\times 10^6$ times. 
The results are summarized in Table \ref{tab:comp_EXPRESS_SaleA} with window size $K=5$ and selection-conditional miscoverage $\alpha=40\%$. Our procedure yields identical results as EXPRESS in this setting.

\begin{table}[htbp]
\centering
 \caption{Comparison of miscoverage,  calibration point number (CP), frequency of infinite length interval (IL) and median interval length (ML) at $t=20$ for different methods under Comparison Case 3.}\label{tab:comp_EXPRESS_SaleA}
 %\resizebox{\textwidth}{!}
 {
\begin{tabular}{lcccc}
\toprule
Method  & Miscoverage & CP & IL & ML \\
\midrule
CAP-ada     & 0.308 & 9.29 & 0.234  & 0.567 \\
EXPRESS     & 0.308 & 9.29 & 0.234  & 0.567 \\
5-CAP-ada    & 0.346 & 10.5 & 0.0980 & 0.716 \\
5-EXPRESS   & 0.346 & 10.5 & 0.0980 & 0.716 \\
CAP-nonada   & 0.437 & 30   & 0      & 0.645 \\
\bottomrule
\end{tabular}}
\end{table}

\subsection{Additional simulation results for nonincreasing decision-driven selection rule}
We provide additional simulation results for nonincreasing decision-driven selection rule. The selection rule is $S_t=\mathds{1}\{\widehat{\mu}(X_t)>\tau(\sum_{i=0}^{t-1} S_i)\}$ where
$\tau(s)=\tau_0+\min\{s/50,2\}$ with $\tau_0$ is fixed as $1$, $4$ and $3$ for Scenario A, B and C respectively. The CAP is implemented with nonadaptive pick rule. The results are summarized in Figure \ref{fig:non_incre}. It is shown that the nonadaptive CAP controls FCR level precisely, which verifies our theory.

\begin{figure}[htbp!]
    \centering
    \includegraphics[width=0.9\textwidth]{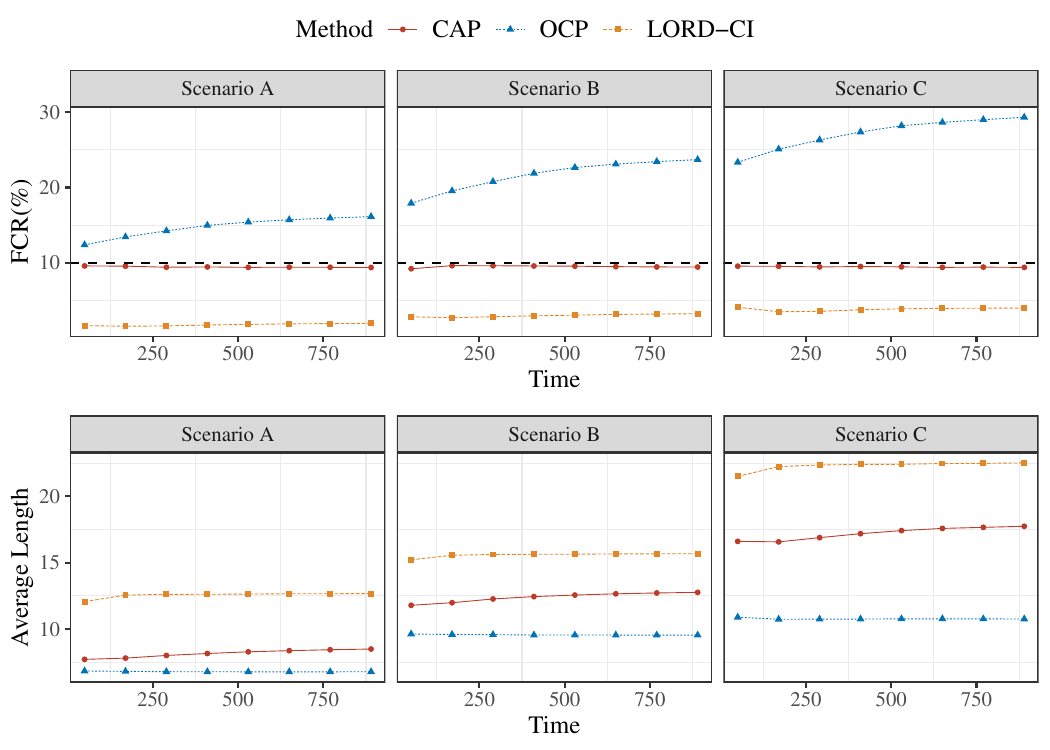}
    \caption{\small\it Comparison for three methods by real-time FCR plot and average length plot from time $50$ to $1,000$ for full calibration set after $500$ replications under nonincreasing decision-driven selection rule. The black dashed line represents the target FCR level $10\%$.}
    \label{fig:non_incre}
\end{figure}

\subsection{Comparisons of adaptive and nonadaptive pick rules for decision-driven selection}\label{appen:compare_inter}
Under \textbf{Dec-driven} rule,
we make empirical comparisons of calibration set picked by adaptive rule $\widehat{\gC}_t^{\rm{ada}} =  \{s\in\gH_t: \Pi_t(X_s) \prod_{i\in \gN_t^{\rm on}}\Indicator{\Pi_i(X_s) = \Pi_i(X_t)}\}$ and calibration set picked by nonadaptive rule $\widehat{\gC}_t^{\rm{nonada}} = \{s\in \gH_t: \Pi_t(X_s) = 1\}$. The setting is the same as Scenario A except that we change the target FCR level at $\alpha=40\%$ and stop the procedure at $t=200$.

\begin{figure}[htbp!]
    \centering
    \includegraphics[width=0.9\textwidth]{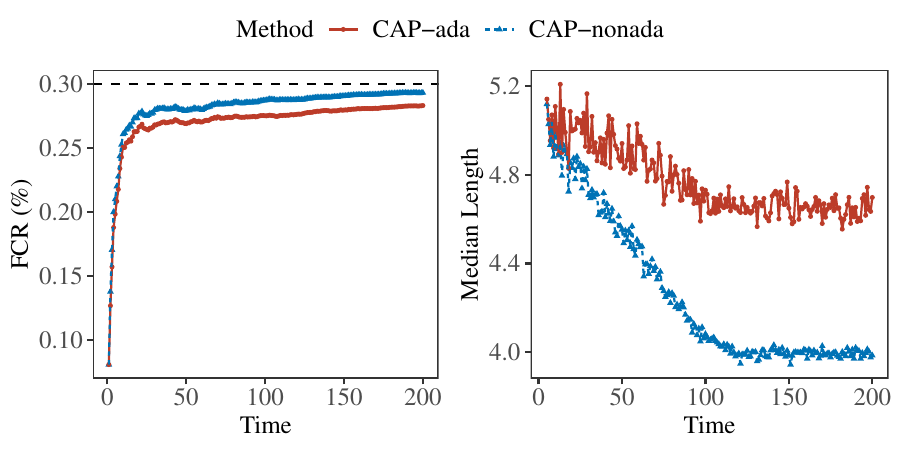}
    \caption{\small\it Comparison for CAP-ada and  CAP-nonada by real-time FCR plot and average length plot from time $1$ to $200$  after $1,000$ replications. The black dashed line represents the target FCR level $40\%$.}
    \label{fig:comp_inter}
\end{figure}

The results are demonstrated in Figure \ref{fig:comp_inter}. As CAP-ada usually picks none of the calibration data, leading to intervals with infinite length, we report the median length of prediction intervals instead of the average length among $1,000$ replications. Both methods can control the FCR. But CAP-ada ($\widehat{\gC}_t^{\rm{ada}}$) has a much wider interval compared to CAP-nonada ($\widehat{\gC}_t^{\rm{nonada}}$). And we find that at time $t=200$, CAP-ada picks no calibration data at a proportion of $5.2\%$, while CAP-nonada always provides sufficient calibration points.

\subsection{Comparisons of adaptive and nonadaptive pick rules for symmetric selection}\label{appen:compare_swap}

\begin{figure}[htbp!]
    \centering
    \includegraphics[width=0.8\textwidth]{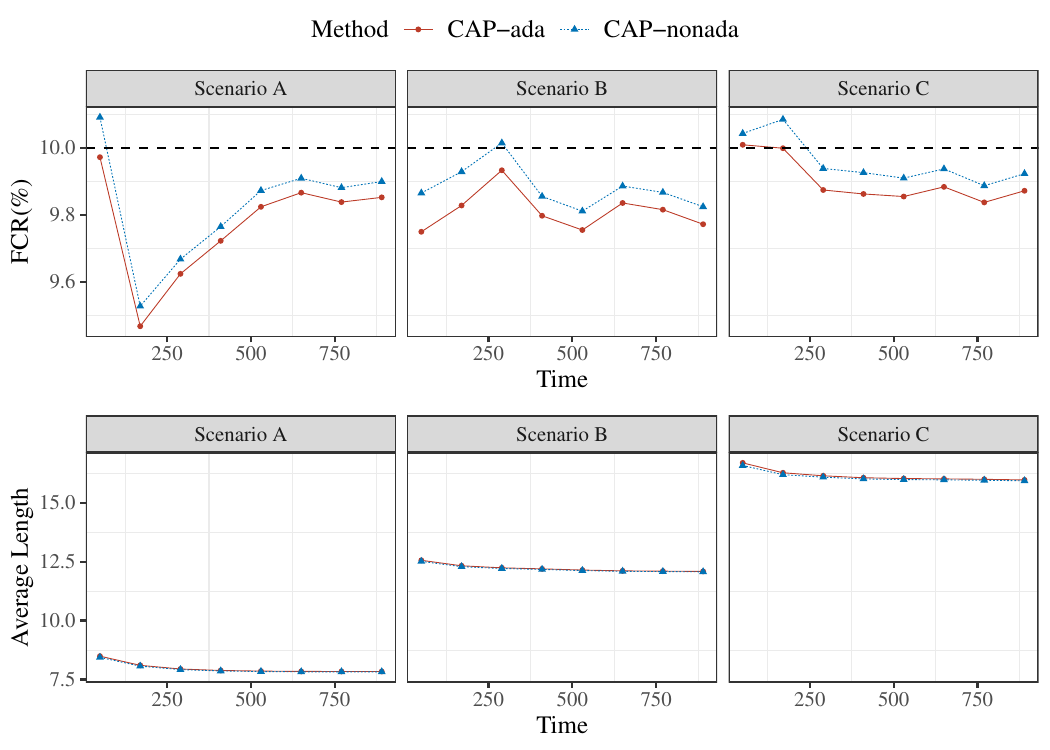}
    \caption{\small\it Comparison for  CAP-ada and CAP-nonada by real-time FCR plot and average length plot from time $50$ to $1,000$ for full calibration set after $500$ replications. The black dashed line represents the target FCR level $10\%$.}
    \label{fig:comp_swap}
\end{figure}

We study the difference of $\widehat{\gC}_t^{\rm{ada}}$ and $\widehat{\gC}_t^{\rm{nonada}}$ for quantile selection rule. Figure \ref{fig:comp_swap} displays the results for both methods under three scenarios using $70\%$-quantile selection rule. The CAP-ada (using $\widehat{\gC}_t^{\rm{ada}}$) and  CAP-nonada (using $\widehat{\gC}_t^{\rm{nonada}}$) perform almost identically.

\subsection{Discussions of DtACI}\label{subsec:discuss_dtaci}
We compare our method with Algorithm \ref{alg:cond_DtACI} using vanilla conformal prediction, which is denoted as DtACI-sel. It is aware of the selection effect, renewing the parameter based on the performance of selected prediction intervals. But DtACI-sel constructs the prediction interval using the whole observed labeled data instead of data in the picked calibration set.
Under the same experimental setup as in our paper, we conducted empirical investigations.

We first examined the i.i.d. data setting, where the distribution shift arises solely due to selection. The results, presented in Figure \ref{fig:DtACI_iid}, clearly show that DtACI-sel requires a longer time to adapt to the shifted distribution caused by selection, whereas our method maintains precise control of the online FCR at all times.

\begin{figure}[ht!]
    \centering
    \includegraphics[width=\textwidth]{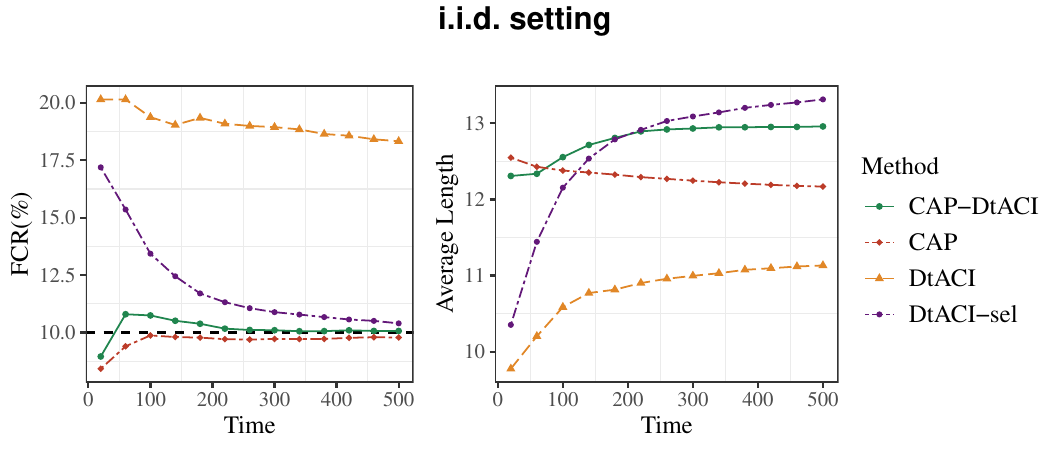}
    \caption{\small\it Comparison for CAP-DtACI, CAP, DtACI and DtACI-sel (Algorithm 2 but using vanilla conformal prediction) by real-time FCR and average length from time $20$ to $500$ for quantile selection rule under i.i.d. setting. The black dashed line represents the target FCR level $10\%$.}
    \label{fig:DtACI_iid}
\end{figure}

This phenomenon can also be observed in settings where the data gradually shift over time. Figure \ref{fig:DtACI_ss} illustrates the performance of DtACI-sel during the initial stage under a scenario of slow distribution shift. Our method outperforms DtACI-sel, as it does not need to first adapt to the selection-induced shift, allowing for quicker adjustment to the distributional changes in the data itself.

\begin{figure}[ht!]
    \centering
    \includegraphics[width=\textwidth]{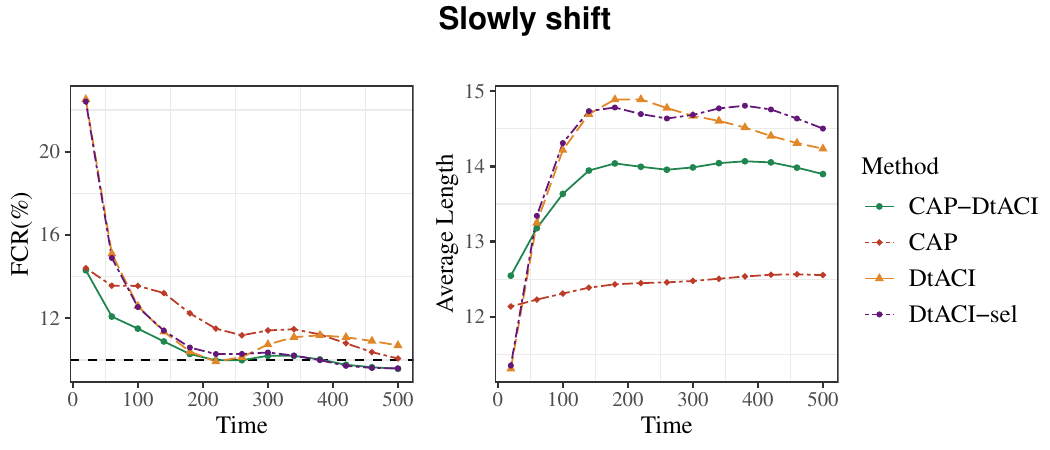}
    \caption{\small\it Comparison for CAP-DtACI, CAP, DtACI and DtACI-sel (Algorithm 2 but using vanilla conformal prediction) by real-time FCR and average length from time $20$ to $500$ for quantile selection rule under slowly shift setting. The black dashed line represents the target FCR level $10\%$.}
    \label{fig:DtACI_ss}
\end{figure}

In conclusion, while both CAP-DtACI and DtACI-sel can guarantee long-term FCR control, CAP-DtACI is more efficient in practice for addressing the selective problem.

\subsection{Illustrative plot for drug discovery}
\begin{figure}[htb!]
    \centering
    \includegraphics[width=\textwidth]{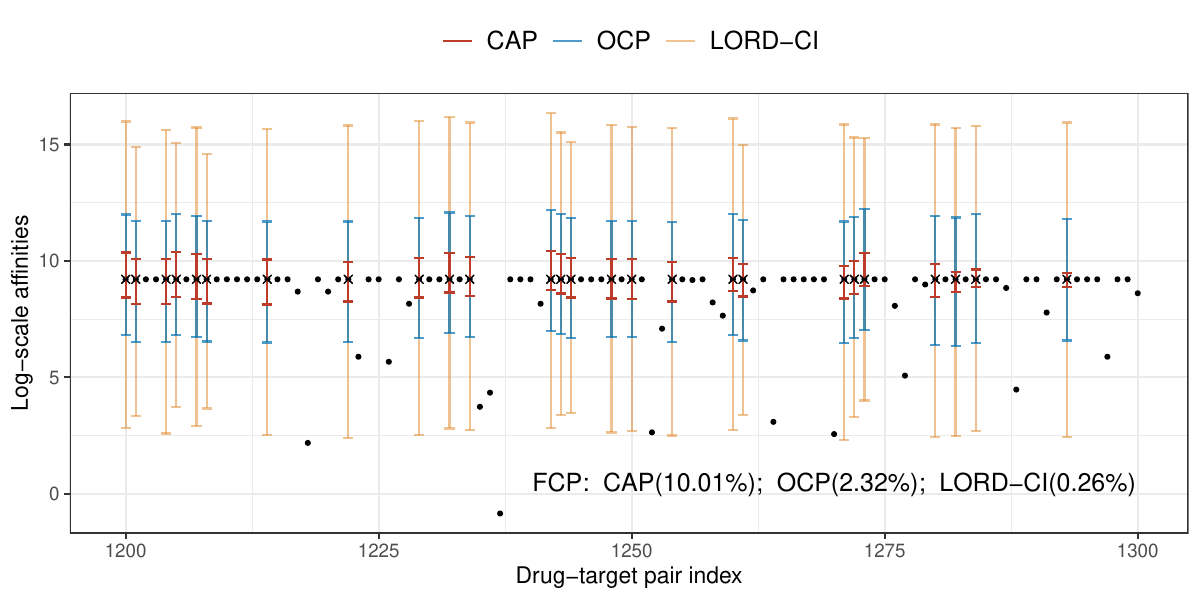}
    \caption{\small\it Plot for the real-time log-scale affinities and PIs for selected points from index 1,200 to 1,300. The selected points are marked by the cross. The PIs are constructed by three methods with a target FCR level $10\%$. Red interval: CAP (FCP at index 1,300 is $10.01\%$); Blue interval: ordinary online conformal prediction which provides marginal interval (FCP is $2.32\%$); Orange interval: LORD-CI with defaulted parameters (FCP is $0.26\%$). }
    \label{fig:Drug_plot}
\end{figure}

%To illustrate the effectiveness of our method, 
Here, we provide an illustrative plot in drug discovery to show the effectiveness of our method. Figure \ref{fig:Drug_plot} visualizes the real-time PIs with target FCR level $10\%$ constructed by different methods. The simulation setups are the same as Section \ref{sec:real_data}, except that we use quantile selection rule only. Our proposed method CAP (red ones) constructs the shortest intervals with FCP at 10.01\%. %which is defined as $\FCP(T):={\sum_{t=0}^T S_t\cdot \Indicator{Y_t \not\in \gI_t(X_t)}}/{(1\vee \sum_{t=0}^T S_t)}$. %The blue intervals are constructed by ordinary online conformal prediction (OCP),  
In contrast, both the OCP (blue ones) and the LORD-CI (orange ones) produce excessively wide intervals and yield conservative FCP levels, 2.32\% and 0.26\%, respectively. Therefore, the CAP emerges as a valid approach to accurately quantifying uncertainty while simultaneously achieving effective interval sizes.

\subsection{Additional real-data application to airfoil self-noise}

Airflow-induced noise prediction and reduction is one of the priorities for both the energy and aviation industries \citep{brooks1989airfoil}. We consider applying our method to the airfoil data set from the UCI Machine Learning Repository \citep{Dua:2019}, which involves $1,503$ observations of a response $Y$ (scaled sound pressure level of NASA airfoils), and a five-dimensional
feature including log frequency, angle of attack, chord length, free-stream velocity,
and suction side log displacement thickness. The data is obtained via a series of aerodynamic and acoustic tests, and the distributions of the data are in different patterns at different times. This dataset can be regarded as having distribution shifting over time, and we aim to implement the CAP-DtACI with the same parameters in Section \ref{subsec:simu_shift} to solve this problem.

\begin{figure}[htb]
    \centering
    \includegraphics[width=\textwidth]{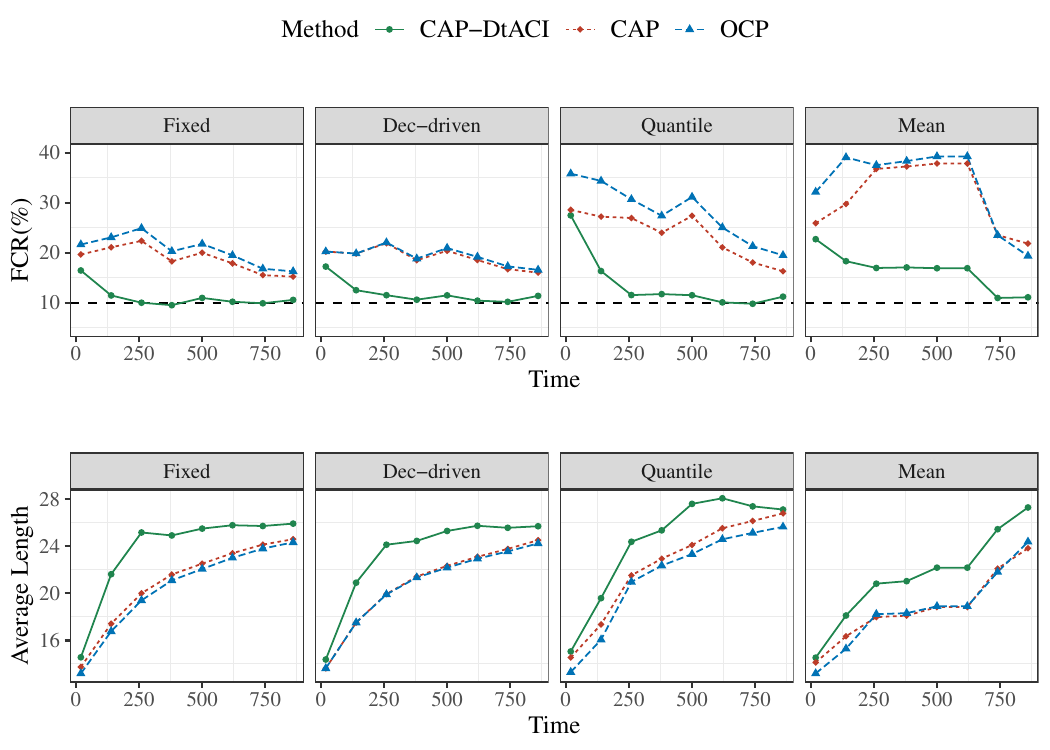}
    \caption{\small\it Real-time FCR and average length from time $20$ to $900$ by $20$ replications for four selection rules in airfoil noise task. The black dashed line denotes the target FCR level $10\%$. }
    \label{fig:Airfoil-FCR}
\end{figure}

We reserve the first $480$ samples as a training set to train an SVM model with defaulted parameters, and then we use the following $23$ samples as the initial holdout set. Since the data is in time order, we take an integrated period of size 900 from the remaining samples as the online data set. We treat each choice of the periods (starting at different times) as a repetition to compute the FCR and average length. %And we use a window size of $500$ in this task. 
Four selection rules are considered here: fixed selection rule with $S_t=\Indicator{\widehat{\mu}(X_t)>115}$; decision-driven selection rule with $S_t=\Indicator{\widehat{\mu}(X_t)>110+\min\{\sum_{j=0}^{t-1} S_j/30,10\}}$; quantile selection rule with $S_t=\Indicator{\widehat{\mu}(X_t)>\mathcal{A}(\{\widehat{\mu}(X_i)\}_{i=t-500}^{t-1})}$ where $\mathcal{A}(\{\widehat{\mu}(X_i)\}_{i=t-500}^{t-1})$ is the $35\%$-quantile of $\{\widehat{\mu}(X_i)\}_{i=t-500}^{t-1} $; mean selection rule with $S_t=\Indicator{\widehat{\mu}(X_t)>\sum_{i=t-500}^{t-1} \widehat{\mu}(X_i)/500}$. %Here we do not consider the online multiple testing rule due to the limited sample size, leading to few selections. 
We adopt a windowed scheme with window size $500$, and set the target FCR level $\alpha=10\%$. 

Figure \ref{fig:Airfoil-FCR} summarizes the FCR and average lengths of CAP-DtACI, CAP and OCP among 20 replications.
%Aggregating 20 different choices of the online data set, the results are summarized in Figure \ref{fig:Airfoil-FCR}. 
As illustrated, the CAP-DtACI performs well in delivering FCR close to the target level as time grows across almost every setting. In contrast, CAP and OCP cannot obtain the desired FCR control due to a lack of consideration of distribution shifts.

%\bibliography{ref-JMLR-final.bib}